\documentclass[dvipsnames]{article}
\usepackage[utf8]{inputenc}
\usepackage{xcolor}
\usepackage[margin=1.0in]{geometry}
\usepackage{amsfonts,amssymb,amsthm,amsmath}
\usepackage{csquotes}
\usepackage{url}
\usepackage{subfig}
\usepackage[toc,page]{appendix}
\usepackage{mathtools}
\usepackage[hidelinks]{hyperref}
\usepackage[round,sort]{natbib}
\usepackage{enumitem}
\let\cite\citep

\DeclareMathOperator{\blkdiag}{blockdiag}
\newtheorem{theorem}{Theorem}
\newtheorem{proposition}{Proposition}
\newtheorem{lemma}{Lemma}
\newtheorem{assumption}{Assumption}
\newtheorem{corollary}{Corollary}
\newtheorem{definition}{Definition}
\newtheorem{claim}{Claim}
\newtheorem{example}{Example}
\newtheorem{remark}{Remark}
\usepackage{apptools}
\AtAppendix{\counterwithin{lemma}{section}}
\AtAppendix{\counterwithin{proposition}{section}}
\AtAppendix{\counterwithin{theorem}{section}}
\AtAppendix{\counterwithin{corollary}{section}}
\AtAppendix{\counterwithin{definition}{section}}
\AtAppendix{\counterwithin{claim}{section}}
\DeclareMathOperator{\spec}{spec}
\newcommand{\mc}{\mathcal}
\newcommand{\mb}{\mathbb}
\newcommand{\vep}{\varepsilon}

\newcommand{\g}{g}

\newcommand{\la}{\langle}
\newcommand{\ra}{\rangle}
\newcommand{\rar}{\rightarrow}

\renewcommand{\Re}{\mathrm{Re}}
\renewcommand{\Im}{\mathrm{Im}}
\newcommand{\m}{n}
\newcommand{\reac}{r}

\newcommand{\bmat}[1]{\begin{bmatrix}#1\end{bmatrix}}

\newcommand{\gda}{{\tt GDA}}

\newcommand{\kk}{k}

\newcommand{\schur}{{S}}

\newcommand{\schurtt}{{\tt S}}
\newcommand{\W}{\mc{W}}
\renewcommand{\vec}{\mathrm{vec}}
\newcommand{\vech}{\mathrm{vech}}

\newcommand{\B}{A_{12}}

\newcommand{\comp}{{\tt c}}
\newcommand{\Hstable}{\mc{S}^\comp(\mb{C}_-^\circ)}
\newcommand{\Hstablecomp}{\mc{S}^\comp(\mb{C}_-^\circ)}
\newcommand{\inter}{\mathrm{int}}
\newcommand{\Dis}{\mathrm{D}}
\newcommand{\Gen}{\mathrm{G}}
\DeclareMathOperator{\supp}{supp}
\newcommand{\mD}{\mc{D}}
\newcommand{\lambdam}{\lambda_{\tt m}}
\usepackage{tikz}
\usepackage{xkcdcols}
\usetikzlibrary{shapes.geometric}
\usetikzlibrary{patterns}
\tikzstyle{decision} = [diamond, draw, text badly centered, inner sep=3pt, aspect=2,rotate=45]
\newif\ifdnesection
\dnesectionfalse
\newif\iflyap
\lyapfalse
\newif\ifinstability
\instabilityfalse
\newif\ifoldproof
\oldprooffalse
\newif\ifexistencecommentary
\existencecommentaryfalse
\DeclareMathOperator{\tr}{tr}
\newtheorem{innercustomthm}{}
\newenvironment{customthm}[1]
  {\renewcommand\theinnercustomthm{#1}\innercustomthm}
  {\endinnercustomthm}

\title{Gradient Descent-Ascent Provably Converges to Strict Local Minmax Equilibria with a Finite Timescale Separation}
\date{}
\begin{document}

\author{
  \textbf{Tanner Fiez} \\
  University of Washington \\
  \textup{fiezt@uw.edu}
  \and
  \textbf{Lillian J. Ratliff} \\
  University of Washington \\
  \textup{ratliffl@uw.edu}
}
\maketitle

\begin{abstract}
This paper concerns the local stability and convergence rate of gradient descent-ascent in two-player non-convex, non-concave zero-sum games. We study the role that a finite timescale separation parameter $\tau$ has on the learning dynamics where the learning rate of player 1 is denoted by $\gamma_1$ and the learning rate of player 2 is defined to be $\gamma_2=\tau\gamma_1$. Existing work analyzing the role of timescale separation in gradient descent-ascent has primarily focused on the edge cases of players sharing a learning rate ($\tau =1$) and the maximizing player approximately converging between each update of the minimizing player ($\tau \rightarrow \infty$). For the parameter choice of $\tau=1$, it is known that the learning dynamics are not guaranteed to converge to a game-theoretically meaningful equilibria in general as shown by \citet{mazumdar2020gradient} and \citet{daskalakis:2018aa}. In contrast,~\citet{jin2019local} showed that the stable critical points of gradient descent-ascent coincide with the set of strict local minmax equilibria as $\tau\rightarrow\infty$. In this work, we bridge the gap between past work by showing there exists a finite timescale separation parameter $\tau^{\ast}$ such that $x^{\ast}$ is a stable critical point of gradient descent-ascent for all $\tau \in (\tau^{\ast}, \infty)$ if and only if it is a strict local minmax equilibrium. Moreover, we provide an explicit construction for computing $\tau^{\ast}$ along with corresponding convergence rates and results under deterministic and stochastic gradient feedback. The convergence results we present are complemented by a non-convergence result: given a critical point $x^{\ast}$ that is not a strict local minmax equilibrium, then there exists a finite timescale separation $\tau_0$ such that $x^{\ast}$ is unstable for all $\tau\in (\tau_0, \infty)$. Finally, we extend the stability and convergence results regarding gradient descent-ascent to gradient penalty regularization methods for generative adversarial networks~\cite{mescheder2018training} and empirically demonstrate on the CIFAR-10 and CelebA datasets the significant impact timescale separation has on training performance. 
\end{abstract}

\addtocontents{toc}{\protect\setcounter{tocdepth}{-10}}

\section{Introduction}
\label{sec:intro}
In this paper we study learning in zero-sum games of the form
\[\min_{x_1 \in X_1}\max_{x_2 \in X_2}f(x_1, x_2)\]
where the objective function of the game $f$ is assumed to be sufficiently smooth and potentially non-convex and non-concave in the strategy spaces $X_1$ and $X_2$ respectively with each $X_i$ a precompact subset of $\mb{R}^{\m_i}$. This general problem formulation has long been fundamental in game theory~\cite{basar:1998aa} and recently it has become central to machine learning with applications in generative adversarial networks~\cite{goodfellow2014generative}, robust supervised learning~\cite{madry2017towards, sinha2017certifiable}, reinforcement and multi-agent reinforcement learning~\cite{rajeswaran2020game, zhang2019multi}, imitation learning~\cite{ho2016generative}, constrained optimization~\cite{cherukuri2017saddle}, and hyperparameter optimization~\cite{mackay2019self, lorraine2020optimizing} among several others. 
From a game-theoretic viewpoint, the work on learning in games can, in some sense, be viewed as explaining how equilibria arise through an iterative competition for optimality~\cite{fudenberg1998theory}. However, in machine learning, the primary purpose of learning dynamics is to compute equilibria efficiently for the sake of providing meaningful solutions to problems of interest. 

As a result of this perspective, there has been significant interest in the study of gradient descent-ascent owing to the fact that the learning rule is computationally efficient and a natural analogue to gradient descent from function optimization. Formally, the learning dynamics are given by each player myopically updating a strategy with an individual gradient as follows:
\begin{align*}
x_1^{+} &= x_1 - \gamma_1 D_1f(x_1, x_2)\\
x_2^{+} &= x_2 + \gamma_2 D_2f(x_1, x_2).
\end{align*}
The analysis of gradient descent-ascent is complicated by the intricate optimization landscape in non-convex, non-concave zero-sum games. To begin, there is the fundamental question of what type of solution concept is desired. Given the class of games under consideration, local solution concepts have been proposed and are often taken to be the goal of a learning algorithm. The primary notions of equilibrium that have been adopted are the local Nash and local minmax/Stackelberg concepts with a focus on the set of strict local equilibrium that can be characterized by gradient-based sufficient conditions. Following several past works, from here on we refer to strict local Nash equilibrium and strict local minmax/Stackelberg equilibrium as differential Nash equilibrium and differential Stackelberg equilibrium, respectively.

Regardless of the equilibrium notion under consideration, a number of past works highlight failures of standard gradient descent-ascent in non-convex, non-concave zero-sum games.
Indeed, it has been shown gradient descent-ascent with a shared learning rate ($\gamma_1=\gamma_2$) is prone to reaching critical points that are neither differential Nash equilibrium nor differential Stackelberg equilibrium~\cite{daskalakis:2018aa, mazumdar2020gradient,jin2019local}. While an important negative result, it does not rule out the prospect that gradient descent-ascent may be able to guarantee equilibrium convergence as it fails to account for a key structural parameter of the learning dynamics, namely the ratio of learning rates between the players.

Motivated by the observation that the order of play between players is fundamental to the definition of the game, the role of \textit{timescale separation} in gradient descent-ascent has been explored theoretically in recent years~\cite{heusel2017gans, chasnov2019uai,jin2019local}. On the empirical side of past work, it has been widely demonstrated and prescribed that timescale separation in gradient descent-ascent between the generator and discriminator, either by heterogeneous learning rates or unrolled updates, is crucial to improving the solution quality when training generative adversarial networks~\cite{goodfellow2014generative, arjovsky2017wasserstein, heusel2017gans}. 
Denoting $\gamma_1$ as the learning rate of the player 1, the learning rate of player 2 can be redefined as $\gamma_2=\tau \gamma_1$ where $\tau = \gamma_2/\gamma_1>0$ is the ratio of learning rates or \emph{timescale separation} parameter. The work of~\citet{jin2019local} took a meaningful step toward understanding the effect of timescale separation in gradient descent-ascent by showing that as $\tau\rightarrow \infty$ the stable critical points of the learning dynamics coincide with the set of differential Stackelberg equilibrium. In simple terms, the aforementioned result implies that all `bad critical points' (that is, critical points lacking game-theoretic meaning) become unstable as the timescale separation approaches infinity and that all `good critical points' (that is, game-theoretically meaningful equilibria) remain or become stable as the timescale separation approaches infinity.
While a promising theoretical development on the local stability of the underlying dynamics, it does not lead to a practical, implementable learning rule or necessarily provide an explanation for the satisfying performance in applications of gradient descent-ascent with a \emph{finite} timescale separation. It remains an open question to fully understand gradient descent-ascent as a function of the timescale separation and to determine whether the desirable behavior with an infinite timescale separation is achievable for a range of finite learning rate ratios.

This paper continues the theoretical study of gradient descent-ascent with timescale separation in non-convex, non-concave zero-sum games. We focus our attention on answering the remaining open questions regarding the behavior of the learning dynamics with finite learning rate ratios and provide a number of conclusive results. Notably, we develop necessary and sufficient conditions for a critical point to be stable for a range of finite learning rate ratios. The results imply that differential Stackelberg equilibria are stable for a range of finite learning rate ratios and that non-equilibria critical points are unstable for a range of finite learning rate ratios. Together, this means gradient descent-ascent only converges to differential Stackelberg equilibrium for a range of finite learning rate ratios. To our knowledge, this is the first provable guarantee of its kind for an implementable first-order method. Moreover, the technical results in this work rely on tools that have not appeared in the machine learning and optimization communities analyzing games and expose a number interesting directions of future research.
Explicitly, the notion of a guard map, which is arguably even an obscure tool in modern control and dynamical systems theory, is `rediscovered' in this work as a technique for analyzing the stability of game dynamics.

\subsection{Contributions}
To motivate our primary theoretical results, we present a self-contained description of what is known about the local stability of gradient descent-ascent around critical points in Section~\ref{sec:prelim_obs}. The existing results primarily concern gradient descent-ascent without timescale separation and with a ratio of learning rates approaching infinity (see Figure~\ref{fig:inclusions} for a graphical depiction of known results in each regime). In contrast,  this paper is focused on characterizing the stability and convergence of gradient descent-ascent across a range of finite learning rate ratios. 
To hint at what is achievable in this realm, we present simple examples for which gradient descent-ascent converges to non-equilibrium critical points and games with differential Stackelberg equilibrium that are unstable with respect to gradient descent-ascent without timescale separation (see Examples~\ref{example:nonstablestack} and~\ref{example:stablenoneq}, Section~\ref{sec:mainresults}). While the existence of such examples is known~\cite{daskalakis:2018aa, mazumdar2020gradient, jin2019local}, we demonstrate in them that a finite timescale separation is sufficient to remedy the undesirable stability properties of gradient descent-ascent without timescale separation. 

Toward characterizing this phenomenon in its full generality, we provide intermediate results which are known, but we prove using technical tools not yet broadly seen and exploited by this community. To begin, it is known that the set of differential Nash equilibrium are stable with respect to gradient descent-ascent~\cite{mazumdar2019cdc,daskalakis:2018aa}, and that they remain stable for any timescale separation parameter $\tau \in (0, \infty)$~\cite{jin2019local}. We provide a proof for this result (Proposition~\ref{lem:zsspecbdd}) using the concept of quadratic numerical range~\cite{tretter2008spectral}. Furthermore, \citet{jin2019local} recently showed that as the timescale separation $\tau\rightarrow \infty$, the stable critical points of gradient descent-ascent coincide with the set of differential Stackelberg equilibrium. We reveal that this result has long existed in the literature on singularly perturbed systems~\cite[Chapter~2 and the citations within]{kokotovic1986singular} and provide a proof (see Proposition~\ref{prop:simgrad_inf}) using analysis methods from the aforementioned line of work that are novel to the literature on learning in games from the machine learning and optimization communities in recent years.

A relevant line of study on singularly perturbed systems is that of characterizing the range of perturbation parameters for which a system is stable~\cite{kokotovic1986singular,saydy1990guardian,saydy1996newstability}. Debatably introduced by \citet{saydy1990guardian}, guardian or guard maps act as a certificate that the roots of a polynomial
lie in a particular guarded domain for a range of parameter values. Historically, guard maps serve as a tool for studying the stability of parameterized families of dynamical systems. We bring this tool to learning in games and construct a map that guards a class of Hurwitz stable matrices parameterized by the timescale separation parameter $\tau$ in order to analyze the range of learning rate ratios for which a critical point is stable with respect to gradient descent-ascent. This technique leads to the following result. 
\begin{customthm}{Informal Statement of Theorem~\ref{thm:iffstack}}
Consider a sufficiently regular critical point $x^\ast$ of gradient descent-ascent. There exists a $\tau^\ast\in(0,\infty)$ such that  $x^{\ast}$ is stable for all $\tau\in (\tau^\ast,\infty)$ if and only if $x^\ast$ is a differential Stackelberg equilibrium. 
\end{customthm}
Theorem~\ref{thm:iffstack} confirms that there does indeed exist a range of finite learning ratios such that a differential Stackelberg equilibrium is stable with respect to gradient descent-ascent. Moreover, such a range of learning rate ratios only exists if a critical point is a differential Stackelberg equilibrium. As we show in Corollary~\ref{cor:asymptoticiffstack}, the former implication of Theorem~\ref{thm:iffstack} nearly immediately implies there exists a $\tau^\ast\in(0,\infty)$ such that gradient descent-ascent converges locally asymptotically for all $\tau\in(\tau^\ast,\infty)$ if and only if $x^\ast$ is a differential Stackelberg equilibrium given a suitably chosen learning rate and deterministic gradient feedback. We give an explicit asymptotic rate of convergence in Theorem~\ref{thm:convergencerateDSE} and characterize the iteration complexity in Corollary~\ref{cor:finitetimebound}. Moreover, we extend the convergence guarantees to stochastic gradient feedback in Theorem~\ref{thm:stochastic_convergence}.

The latter implication of Theorem~\ref{thm:iffstack} says that there exists a finite learning rate ratio such that a non-equilibrium critical point of gradient descent-ascent is unstable. Building off of this, we complement the stability result of Theorem~\ref{thm:iffstack} with the following analagous instability result.
\begin{customthm}{Informal Statement of Theorem~\ref{prop:instability}}
Consider any  stable critical point $x^\ast$ of gradient descent-ascent which is not a differential Stackelberg equilibrium. There exists a finite learning rate ratio $\tau_0\in (0,\infty)$ such that $x^{\ast}$ is unstable for all $\tau \in (\tau_0,\infty)$.
\end{customthm}
Theorem~\ref{prop:instability} establishes that there exists a range of finite learning ratios non-equilibrium critical points are unstable with respect to gradient descent-ascent. This implies that for a suitably chosen finite timescale separation, gradient descent-ascent avoids critical points lacking game-theoretic meaning. Together, Theorem~\ref{thm:iffstack} and Theorem~\ref{prop:instability} answer affirmatively that gradient descent-ascent with timescale separation can guarantee equilibrium convergence, which answers a standing open question. Moreover, we provide explicit constructions for computing $\tau^{\ast}$ and $\tau_0$ given a critical point. In fact our construction of $\tau^\ast$ in Theorem~\ref{thm:iffstack} is tight, and this is confirmed by our  numerical experiments.

We finish the theoretical analysis of gradient descent-ascent in this paper by connecting to the literature on generative adversarial networks. We show under common assumptions on generative adversarial networks~\cite{nagarajan2017gradient, mescheder2018training} that the introduction of gradient penalty based regularization to the discriminator does not change the set of critical points for the dynamics and, further, there exists a finite learning rate ratio $\tau^\ast$ such that for any learning rate ratio $\tau\in (\tau^\ast,\infty)$ and any non-negative, finite regularization parameter $\mu$, the continuous time limiting regularized learning dynamics remain stable, and hence, there is a range of learning rates $\gamma_1$ for which the discrete time update locally converges asymptotically.
\begin{customthm}{Informal Statement of Theorem~\ref{thm:ganconvergence}}
Consider  training a generative adversarial network with a gradient penalty (for any fixed regularization parameter $\mu\in(0,\infty)$) via a zero-sum game with generator network $\Gen_\theta$, discriminator network $\Dis_\omega$, and loss $f(\theta,\omega)$ such that relaxed realizable assumptions are satisfied for a critical point $(\theta^\ast,\omega^\ast)$. 
Then, $(\theta^\ast,\omega^\ast)$ is a differential Stackelberg equilibrium, and  for any $\tau\in(0,\infty)$,
gradient descent-ascent converges locally asymptotically.  Moreover, an asymptotic rate of convergence is given in Corollary~\ref{cor:rate}.
\end{customthm}

The theoretical results we provide are complemented by extensive experiments. In simulation, we explore a number of interesting behaviors of gradient descent-ascent with timescale separation analyzed theoretically including differential Stackelberg equilibria shifting from being unstable to stable and non-equilibrium critical points moving from being stable to unstable. Furthermore, we examine how the vector field and the spectrum of the game Jacobian evolve as a function of the timescale separation and explore the relationship with the rate of convergence. We experiment with gradient descent-ascent on the Dirac-GAN proposed by \citet{mescheder2018training} and illustrate the interplay between timescale separation, regularization, and rate of convergence. Building on this, we train generative adversarial networks on the CIFAR-10 and CelebA datasets with regularization and demonstrate that timescale separation can benefit performance and stability. In the experiments we observe that regularization and timescale separation are intimately connected and there is an inherent tradeoff between them. This indicates that insights made on simple generative adversarial network formulations may carry over to the complex problems where players are parameterized by neural networks.

Collectively, the primary contribution of this paper is the near-complete characterization of the behavior of gradient descent-ascent with finite timescale separation. Moreover, by introducing a novel set of analysis tools to this literature, our work opens a number of future research questions. As an aside, we believe these technical tools open up novel avenues for not only proving results about learning dynamics in games, but also for synthesizing algorithms.

\subsection{Organization}
The organization of this paper is as follows. 
Preliminaries on game theoretic notions of equilibria, gradient-based learning algorithms, and dynamical systems theory are reviewed in Section~\ref{sec:prelim}. 

Convergence analysis proceeds in two phases. In Section~\ref{sec:mainresults}, we study the stability properties of the continuous time limiting dynamical system given a timescale separation between the minimizing and maximizing players. Specifically, we show the first result on necessary and sufficient conditions for convergence of the continuous time limiting system corresponding to gradient descent-ascent with time scale separation to game theoretically meaningful equilibria (i.e., local minmax equilibria in zero-sum games). Following this, in Section~\ref{sec:convergencerates}, we provide convergence guarantees for the original discrete time dynamical system of interest (namely, gradient descent ascent). Using the results in the proceeding section, we show that gradient descent-ascent converges to a critical point if and only if it is a differential Stackelberg equilibrium (i.e., a \emph{sufficiently regular} local minmax). In addition, we  characterize the  iteration complexity of gradient descent-ascent dynamics and provide finite-time bounds on local convergence to approximate local Stackelberg equilibria. 

We apply the main results in the preceding sections to generative adversarial networks in Section~\ref{sec:gans}, and in Section~\ref{sec:experiments} we present several illustrative examples including generative adversarial networks where we show that tuning the learning rate ratio along with regularization and the exponential moving average hyperparameter significantly improves the Fr\'{e}chet Inception Distance (FID) metric for generative adversarial networks. 

 Given its length, prior to concluding in Section~\ref{sec:related}, we review related work drawing connections to solution concepts,  gradient descent-ascent learning dynamics, applications to adversarial learning where the success of heuristics provide strong motivation for the theoretical work in this paper, and historical connections to dynamical systems theory. Throughout the sections proceeding Section~\ref{sec:related}, we draw connections to related works and results in an effort to place our results in the context of the literature.   We conclude in Section~\ref{sec:discussion} with a discussion on the significance of the results and open questions. The appendix includes the majority of the detailed proofs as well as additional experiments and commentary.

\section{Preliminaries}
\label{sec:prelim}

In this section, we review game theoretic and dynamical systems preliminaries. Additionally, we  formulate the class of learning rules analyzed in this paper.

\subsection{Game Theoretic Preliminaries}
A two--player zero-sum continuous game is defined by a collection of costs $(f_1,f_2)$ where $f_1\equiv f$ and $f_2\equiv -f$ with  $f\in C^r(X,\mb{R})$ for some $r\geq 2$ and where $X=X_1\times X_2$ with each $X_i$ a precompact subset of $\mb{R}^{\m_i}$ for $i=1,2$. Let $n=n_1+n_2$ be the dimension of the joint strategy space $X=X_1\times X_2$.
Player $i\in \mc{I}$ seeks to minimize their cost function $f_i(x_i,x_{-i})$ with respect to their choice variable $x_i$ where $x_{-i}$ is the vector of all other actions $x_j$ with $j\neq i$.

There are two natural equilibrium concepts for such games depending on the order of play---i.e., the Nash equilibrium concept in the case of simultaneous play and the Stackelberg equilibrium concept in the case of hierarchical play. 
Each notion of equilibria can be characterized as the
intersection points of the reaction curves of the players~\cite{basar:1998aa}.

\begin{definition}[Local Nash Equilibrium]
    The joint strategy $x\in X$ is a local Nash
    equilibrium on $\prod_{i\in \mc{I}}U_i\subset X$, where $U_i\subseteq X_i$, if 
    $f(x_1,x_2)\leq
    f(x_1',x_2)$,  for all $x_1'\in U_1\subset X_1$ and $f(x_1,x_2)\geq f(x_1,x_2')$ for all $x_2'\in U_2\subset X_2$. Furthermore, if the inequalities are strict, we say $x$ is a strict local Nash equilibrium.
\label{def:nash_intersection}
\end{definition}

\begin{definition}[Local Stackelberg Equilibrium]
    Consider $U_i\subset X_i$ for $i=1,2$  where, without loss of generality, player 1 is the leader (minimizing player) and player 2 is the  follower (maximizing player). The strategy $x_1^\ast\in U_1$ is a local Stackelberg solution for the leader  if, $\forall x_1\in U_1$,
      \[\textstyle
        \sup_{x_2\in \reac_{U_2}(x_1^\ast)}  f(x_1^\ast, x_2)\leq \sup_{x_2\in
      \reac_{U_2}(x_1)}f(x_1,x_2),
\]
      where $\reac_{U_2}(x_1)=\{y\in U_2|f(x_1,y)\geq f(x_1,x_{2}),  \forall   x_2\in U_2\}$ is the reaction curve.
      Moreover, for any $x_2^\ast\in \reac_{U_2}(x_1^\ast)$, the joint strategy profile $(x_1^\ast, x_2^\ast)\in U_1\times U_2$ is a local Stackelberg equilibrium on $U_1\times U_2$.
   \label{def:lse}
\end{definition}

\ifexistencecommentary
While characterizing existence of equilibria is outside the scope of this work, we remark that
Nash equilibria exist for convex costs on compact and convex strategy spaces and Stackelberg equilibria exist on compact strategy spaces~\cite[Thm.~4.3, Thm.~4.8, \& Sec.~4.9]{basar:1998aa}. This means the class of games on which Stackelberg equilibria exist is broader than on which Nash equilibria exist. 
Existence of local equilibria is guaranteed if the neighborhoods and cost functions restricted to those neighborhoods satisfy the assumptions of the cited results. 
\fi
Predicated on existence,\footnote{Characterizing existence of equilibria is outside the scope of this work. However, we remark that
Nash equilibria exist for convex costs on compact and convex strategy spaces and Stackelberg equilibria exist on compact strategy spaces~\cite[Thm.~4.3, Thm.~4.8, \& Sec.~4.9]{basar:1998aa}.} equilibria can be characterized in terms of sufficient conditions on player costs. 
Indeed, in continuous games, first and second order conditions on player cost functions leads to a differential characterization (i.e., necessary and sufficient conditions) of local Nash equilibria reminiscent of optimality conditions in nonlinear programming~\cite{ratliff:2016aa}.\footnote{The differential characterization of local Nash equilibria in continuous games was first reported in~\cite{ratliff2013allerton}. Genericity and structural stability we studied in general-sum settings in \cite{ratliff2014acc} and in zero-sum settings in \cite{mazumdar2019cdc}.}

We denote $D_if_i$ as the
derivative of $f_i$ with respect to $x_i$, $D_{ij}f_i$ as the partial derivative of $D_if_i$ with respect to $x_j$, $D_i^2f_i$ as the partial derivative of $D_if_i$ with respect to $x_i$, and $D(\cdot)$ as the total
derivative.\footnote{Example: given $f(x,h(x))$, $Df=D_1f+D_2f\circ D h$.}

\begin{proposition}[Necessary and Sufficient Conditions for Local Nash {\cite[Thm.~1 \& 2]{ratliff:2016aa}}] If $x$ is a local Nash equilibrium
      of the zero-sum game $((f,-f)$, then 
$D_1f(x)=0$, $-D_2f(x)=0$, $D_{1}^2f(x)\geq 0$ and $D_2^2f(x)\leq 0$.  On the other hand, if $D_1f(x)=0$, $D_2f(x)=0$, and $D_1^2f(x)>0$ and
$D_{2}^2f(x)<0$, then $x\in X$ is a local Nash equilibrium.
  \end{proposition}
 
The following definition, characterized by sufficient conditions for a local Nash equilibrium as defined in Definition~\ref{def:nash_intersection},  was first introduced in \cite{ratliff2013allerton}.
  \begin{definition}[Differential Nash Equilibrium~\cite{ratliff2013allerton}]
    The joint strategy $x\in X$ is a differential Nash equilibrium if
    $D_1f(x)=0$, $-D_2f(x)=0$, $D_1^2f(x)>0$ and $D_2^2f(x)<0$.
        \label{def:nash}
\end{definition}

Analogous sufficient conditions can be stated to characterize a local Stackelberg equilibrium as defined in Definition~\ref{def:lse}.\footnote{The differential characterization of Stackelberg equilibria was first introduced in \cite{fiez2020stackarxiv}. The genericity and structural structural stability were shown in \cite{fiez:2020icml}.}
Suppose that $f\in C^{r+1}(X,\mb{R})$ for some $r\geq 1$. Given a point $x^\ast$ at which $\det(D_2^2f(x^\ast))\neq 0$, the implicit function theorem~\cite[Thm. 2.5.7]{abraham2012manifolds} implies 
there is a neighborhood $U_1$ of $x_1^\ast$ and a unique $C^r$ map $h:U_1\rar \mb{R}^{n_2}$ such that for all $x_1\in U_1$, 
 $D_2f(x_1,h(x_1))=0$.
The map  $h:x_1\mapsto x_2$ is referred to as the \emph{implicit map}.  Further, $Dh\equiv-(D_2^2f)^{-1}\circ D_{21}f$. Note that  $\det(D_2^2f(x))\neq0$ is a generic condition  (cf.~\cite[Lem.~C.1]{fiez:2020icml}). Let $Df(x_1,h(x_1))$ be the total derivative of $f$ and analogously, let $D^2f$ be the second-order total derivative.
\begin{definition}[Differential Stackelberg Equilibrium~\cite{fiez:2020icml}]
The joint strategy $x = (x_1,x_2)\in X$ is a differential Stackelberg equilibrium if $D_1f(x_1,x_2)=0$, $-D_2f(x_1,x_2)=0$, $D^2f(x_1,x_2)>0$, and $D_2^2f(x_1,x_2)<0$.
 \label{def:stackelberg}
\end{definition}
Observe that in a general sum setting the first order conditions for player $1$ are equivalent  the total derivative of $f$ being zero at the candidate critical point where $x_2$ is implicitly defined as a function of $x_1$ via the implicit mapping theorem applied to $D_2f(x_1,x_2)=0$. Since in this paper and in Definition~\ref{def:stackelberg}, the class of games is zero sum, $D_1f(x_1,x_2)=0$ and $D_2f(x_1,x_2)=0$ (along with the condition that $\det(D_2^2f(x_1,x_2))\neq 0$ which is implied by the second order conditions) are sufficient to imply that  the total derivative $Df(x_1,x_1)$ is zero. 

The  Jacobian of the first order necessary and sufficient condition---i.e., conditions that define potential candidate differential Nash and/or Stackelberg equilibria---is a useful mathematical object for understanding convergence properties of gradient based learning rules as we will see in subsequent sections. 
Consider the vector of individual gradients $g(x)=(D_1f(x),-D_2f(x))$ which define first order conditions for a differential Nash equilibrium. Let $J(x)$ denote the Jacobian of $g(x)$ which is defined by
\begin{equation}
    J(x)=\bmat{D_1^2f_1(x) & D_{12}f_1(x)\\ D_{21}f_2(x) & D_2^2f_2(x)}.
    \label{eq:gamejac-nash}
\end{equation}
We recall from \citet{fiez:2020icml} an alternative (to Definition~\ref{def:stackelberg}, but equivalent set of sufficient conditions  for a differential Stackelberg in terms of $J(x)$.
Let ${\tt S}_1(\cdot)$ denote the Schur complement of $(\cdot)$ with respect to the $n_2\times n$ block-row matrix in $(\cdot)$.
\begin{proposition}[Proposition 1 ~\citet{fiez:2020icml}]\label{prop:gensum}
Consider a zero-sum game $(f,-f)$ defined by $f\in C^r(X,\mb{R})$ with $r\geq 2$ and player 1 (without loss of generality) taken to be the leader (minimizing player). 
Let  $x^\ast$ satisfy $-D_2f(x^\ast)=0$ and $-D_2^2f(x^\ast)>0$. Then $Df(x^{\ast})=0$ and ${\tt S}_1(J(x^{\ast}))>0$ if and only if $x^{\ast}$ is a differential Stackelberg equilibria. 
\end{proposition}

\begin{remark}[A comment on the genericity of differential Stackelberg/Nash equilibria.] Due to \citet[Theorem 1]{fiez:2020icml}, differential Stackelberg are generic amongst local Stackelberg equilibria and, similarly, due to \citet[Theorem 2]{mazumdar2019cdc}, differential Nash equilibria are generic amongst local Nash equilibria. This means that the property of being a differential Stackelberg (respectively, differential Nash) equilibrium in a zero-sum game is generic in the class of zero-sum games defined by $C^2(X,\mb{R})$ functions---that is, \emph{for almost all} (in some formall sense) zero-sum games, all the local Stackelberg/Nash are differential Stackelberg/Nash.
\end{remark}

\subsection{Gradient-based learning algorithms}
As noted above, in this paper we focus on settings in which agents or players in this game are seeking equilibria of the game via a learning algorithm. We study arguably the most natural learning rule in zero-sum continuous games: gradient descent-ascent (\gda). This gradient-based learning rule is a simultaneous gradient play algorithm in that agents update their actions at each iteration simultaneously.

Gradient descent-ascent is defined as follows.
At iteration $\kk$, each agent $i\in
\mc{I}$ updates their choice variable $x_{i,\kk} \in X_i$ by the process
\begin{equation}
        x_{i,\kk+1} =  x_{i,\kk} - \gamma_{i} \g_i(x_{i,\kk},x_{-i,\kk})
    \label{eq:gengrad}
\end{equation}
where 
$\gamma_{i}$
is agent $i$'s learning rate, and $g_i(x)$ is agent $i$'s gradient-based update mechanism.
For simultaneous gradient play, 
\begin{equation}
     \g(x)=(D_1f_1(x), D_2f_2(x))
     \label{eq:gameformnash}
 \end{equation} is the vector of individual gradients and in a zero-sum setting, {\gda} is defined using $\g(x)=(D_1f(x),-D_2f(x))$ where the first player is the minimizing player and the second player is the maximizing player.

One of the key contributions of this paper is that we provide convergence rates for settings in which there is a timescale separation between the learning processes of the minimizing and maximizing players---i.e., settings in which the agents learning rates $\gamma_{i}$ are not homogeneous.
 Define $\Gamma=\mathrm{blkdiag}(\gamma_1I_{\m_1},  \gamma_{2}I_{\m_2})$ where $I_{\m_i}$ denotes the $\m_i\times \m_i$ identity matrix. Let $\tau=\gamma_2/\gamma_1$ be the \emph{learning rate ratio} and define $\Lambda_{\tau}=\blkdiag(I_{\m_1}, \tau I_{\m_2})$.
 The $\tau$-{\gda} dynamics are given by
 \begin{equation}
     x_{k+1}=x_k-\gamma_1\Lambda_{\tau} \g(x_k).
     \label{eq:tausimgrad}
 \end{equation}

\paragraph{Tools for Convergence Analysis.} We analyze the \emph{iteration complexity} or \emph{local asymptotic rate of convergence} of learning rules of the form \eqref{eq:gengrad} in the neighborhood of an equilibrium. 
Given two real valued functions $F(k)$ and $G(k)$, we write $F(k)=O(G(k))$ if there exists a positive constant $c>0$ such that $|F(k)|\leq c|G(k)|$. For example, consider iterates generated by \eqref{eq:gengrad} with initial condition $x_0$ and critical point $x^\ast$. Suppose that we show $\|x_{k+1}-x^\ast\|\leq M^k\|x_0-x^\ast\|$. Then, we write $F(k)=O(M^k)$ where $c=\|x_0-x^\ast\|$.  

\subsection{Dynamical Systems Primer}

   In this paper, we study learning rules employed by agents seeking game-theoretically meaningful equilibria in continuous games. Dynamical systems tools for both continuous and discrete time play a crucial role in this analysis. 
   
\paragraph{Stability.} Before we proceed, 
we recall and remark on some facts from dynamical systems theory concerning stability of equilibria in the continuous-time dynamics
\begin{equation}
   \dot{x} = -g(x)
  \label{eq:ode-uniform}   
\end{equation}
relevant to convergence analysis for 
the discrete-time learning dynamics in~\eqref{eq:gengrad}.
Observe that equilibria are shared between~\eqref{eq:gengrad} and~\eqref{eq:ode-uniform}. Our focus is on 
the subset of equilibria that satisfy Definition~\ref{def:stackelberg}, and the subset thereof defined in Definition~\ref{def:nash}.
 Recall the following equivalent characterizations of stability for an equilibrium of~\eqref{eq:ode-uniform} in terms of the Jacobian matrix $J(x) = Dg(x)$.
\begin{theorem}[{\cite[Thm.~4.15]{khalil2002}}]
Consider a critical point $x^\ast$ of $g(x)$. The following are equivalent:
(a) $x^\ast$ is a locally exponentially stable equilibrium of $\dot{x}=-g(x)$;
(b) $\spec(-J(x^\ast))\subset \mb{C}_-^\circ$;
(c) there exists a symmetric positive-definite matrix $P = P^\top > 0$ such that $P\,J(x^*)  + J(x^*)^\top P > 0$.
\label{thm:exponentialstability}
\end{theorem}

It was shown in~\cite[Prop.~2]{ratliff:2016aa} that if the spectrum of $-J(x)$ at a differential Nash equilibrium $x$ is in the open left-half complex plane---i.e., $\spec(-J(x))\subset \mb{C}_-^\circ$---then $x$ is a
locally exponentially stable equilibrium of \eqref{eq:ode-uniform}.
Indeed, if all agents learn at the same rate so $\Gamma = \gamma I_{\m}$  in~\eqref{eq:gengrad}, then a straightforward application of the Spectral Mapping Theorem~\cite[Thm.~4.7]{callier2012linear} ensures that an exponentially stable equilibrium $x^*$ for~\eqref{eq:ode-uniform} is locally exponentially stable in~\eqref{eq:gengrad} so long as $\gamma > 0$ is chosen sufficiently small \cite{chasnov2019uai}.  
However, this observation 
does not directly tell us how to select $\gamma$ or the resulting iteration complexity in an asymptotic or finite-time sense; furthermore, this line of reasoning does not apply when agents learn at different rates ($\Gamma \neq \gamma I$ in~\eqref{eq:gengrad}). 

\paragraph{Limiting dynamical systems.} The continuous time dynamical system takes the form $\dot{x}=-\Lambda_{\tau} \g(x)$ due to the timescale separation $\tau$.
Such a system is known as a singularly perturbed system or a multi-timescale system in the dynamical systems theory literature~\cite{kokotovic1986singular}, particularly where $\tau^{-1}$ is small. Singularly perturbed systems are classically expressed as
\begin{equation}
\begin{array}{lcl}\dot{x} &=& -D_1f_1(x,z)\\
\epsilon\dot{z} & =& -D_2f_2(x,z)\end{array}
\label{eq:system-simgrad}
\end{equation}
where $\epsilon=\tau^{-1}$ is most often a physically meaningful quantity inherent to some dynamical system that describes the evolution of some physical phenomena; e.g., in circuits it may be a constant related to device material properties, and in communication networks, it is often the speed at which data flows through a physical medium such as cable.

In the classical asymptotic analysis of a system of the form \eqref{eq:system-simgrad}---which we write more generally as $\dot{x}=F(t,x,\epsilon)$ for the purpose of the following observations---the goal is to obtain an approximate solution, say $\tilde{x}(t,\epsilon)$, such that the approximation error $x(t,\epsilon)-\tilde{x}(t,\epsilon)$ is small in some norm for small $|\epsilon|$ and, further, the approximate solution is expressed in terms of a reduced order system. Such results have significance in terms of revealing underlying structural properties of the original system $\dot{x}=F(t,x,\epsilon)$ and its corresponding state $x(t,\epsilon)$ for small $|\epsilon|$. One of the contributions of this work is that we take a similar analysis approach in order to reveal underlying structural properties of the optimization landscape of zero-sum games/minimax optimization problems. 
Indeed, asymptotic methods can reveal multiple timescale structures that are inherent in many machine learning problems, as we observe in this paper for zero-sum games.
One key point of separation in applying dynamical systems theory to the study of algorithms versus physical system dynamics---in particular, learning in games---this parameter no longer necessarily is a physical quantity but is most often a hyper-parameter subject to design.
In this paper, treating the inverse of the learning rate ratio as a timescale separation parameter, we combine the asymptotic analysis tools from singular perturbation theory with tools from  algebra to obtain convergence guarantees.

\paragraph{Leveraging Linearization to Infer Qualitative Properties.}
The Hartman-Grobman theorem asserts that it is possible to continuously deform all trajectories of a nonlinear system onto trajectories of the linearization at a fixed point of the nonlinear system. Informally, the theorem states that if the linearization  of the  nonlinear dynamical system $\dot{x}=F(x)$  around a fixed point $\bar{x}$---i.e., $F(\bar{x})=0$---has no zero or purely imaginary eigenvalues, then there exists a neighborhood $U$ of $\bar{x}$ and a homeomorphism $h:U\rar \mb{R}^{\m}$---i.e., $h, h^{-1}\in C(U,\mb{R}^{\m})$---taking trajectories of $\dot{x}=F(x)$ and mapping them onto those of $\dot{z}=DF(\bar{x})z$. In particular, $h(\bar{x})=0$. 

Given a dynamical system $\dot{x}=F(x)$, the state or solution of the system at time $t$ starting from $x$ at time $t_0$ is called the flow and is denoted $\phi^{t}(x)$.
\begin{theorem}[Hartman-Grobman~{\cite[Thm.~7.3]{sastry1999nonlinear}; \cite[Thm.~9.9]{teschl2000ordinary}}]
\label{thm:hg}
Consider the $\m$-dimensional dynamical system $\dot{x}=F(x)$ with equilibrium point $\bar{x}$. If $DF(\bar{x})$ has no zero or purely imaginary eigenvalues, there is a homeomorphism $h$ defined on a neighborhood $U$ of $\bar{x}$ taking orbits of the flow $\phi^t$ to those of the linear flow $e^{tDF(\bar{x})}$ of $\dot{x}=F(x)$---that is, the flows are topologically conjugate. The homeomorphism preserves the sense of the orbits and is chosen to preserve parameterization by time.
\end{theorem}
The above theorem says that the qualitative properties of the nonlinear system $\dot{x}=F(x)$ in the vicinity (which is determined by the neighborhood $U$) of an isolated equilibrium $\bar{x}$ are determined by its linearization if the linearization has no eigenvalues on the imaginary axes in the complex plane. We also remark that Hartman-Grobman can also be applied to discrete time maps (cf.~\citet[Thm.~2.18]{sastry1999nonlinear}) with the same qualitative outcome.

\paragraph{Internally Chain Transitivity.} In proving results for stochastic gradient descent-ascent, we leverage what is known as the \emph{ordinary differential equation} method in which the flow of the limiting continuous time system starting at sample points from the stochastic updates of the players actions is compared to \emph{asymptotic psuedo-trajectories}---i.e., linear interpolations between sample points. To understand stability in the stochastic case, we need the notion of internally chain transitive sets. For more detail, the reader is referred to \cite[Chap.~2--3]{alongi2007recurrence}. 

A closed set $U\subset \mb{R}^m$ is an invariant set for a differential equation $\dot{x}=F(x)$ if any trajectory $x(t)$ with $x(0)\in U$ satisfies $x(t)\in U$ for all $t\in \mb{R}$. 
Let $\phi^t$ be a flow on a metric space $(X,d)$. 
Given $\vep>0$, $T>0$ and $x,y\in X$, an $(\vep,T)$-chain from $x$ to $y$ with respect to $\phi^t$ and $d$ is a pair of finite sequences $x=x_0,x_1, \ldots, x_{k-1},x_k=y$ in $X$ and $t_0,\ldots, t_{k-1}$ in $[T,\infty)$, denoted together by $(x_0,x_1, \ldots, x_{k-1},x_k;t_0,\ldots, t_{k-1})$, such that $d(\phi^{t_i}(x_i),x_{i+1})<\vep$ for $i=0,1,2,\ldots, k-1$.
A set $U\subseteq X$ is \emph{(internally) chain transitive} with respect to $\phi^t$ if $U$ is a non-empty closed invariant set with respect to $\phi^t$ such that for each $x,y\in U$, $\epsilon>0$ and $T>0$ there exists an $(\vep, T)$-chain from $x$ to $y$.
 A compact invariant set $U$ is invariantly connected if
it cannot be decomposed into two disjoint closed nonempty invariant sets. It is easy to see that every internally chain transitive set is invariantly connected.

\section{Stability of Continuous Time GDA with Timescale Separation}
\label{sec:mainresults}
\label{subsec:stabilityDSE}
To characterize the convergence of $\tau$-{\gda}, we begin by studying its continuous time limiting system
\begin{equation}
\dot{x}=-\Lambda_{\tau} g(x),
\label{eq:gdact}
\end{equation}
where we recall that $\Lambda_{\tau}=\mathrm{blockdiag}(I_{\m_1},\tau I_{\m_2})$. 
Throughout this section, the class of zero-sum games we consider are sufficiently smooth, meaning that
$f\in C^r(X,\mb{R})$ for some $r\geq 2$. The Jacobian of the system from~\eqref{eq:gdact} in zero-sum games of the form $(f_1, f_2) = (f, -f)$ is given as
\begin{equation}
J_{\tau}(x)= \bmat{D_1^2f(x) & D_{12}f(x)\\ \ -\tau D_{12}^{\top}f(x) & -\tau D_2^2f(x)}.
\label{eq:taujac}
\end{equation}
By analyzing the stability of the continuous time system as a function of the timescale separation $\tau$ using the Jacobian from~\eqref{eq:taujac} in this section, we can then draw conclusions about the stability and convergence of the discrete time system $\tau$-{\gda} in Section~\ref{sec:convergencerates}. 

The organization of this section is as follows. To begin, we present a collection of preliminary observations in Section~\ref{sec:prelim_obs} regarding the stability of continuous time gradient descent-ascent with timescale separation to motivate the results in the subsequent subsections by establishing known results and introducing alternative analysis methods that the technical results in this paper build on. Then, in Sections~\ref{sec:stability_main} and~\ref{sec:instability} respectively, we present necessary and sufficient conditions for stability of the continuous time system around critical points in terms of the learning rate ratio along with sufficient conditions to guarantee the instability of the continuous time system around non-equilibrium critical points in terms of the timescale separation.

 \begin{figure*}[t!]
  \centering
  \subfloat[][$\tau=1$]{\includegraphics[width=.2\textwidth]{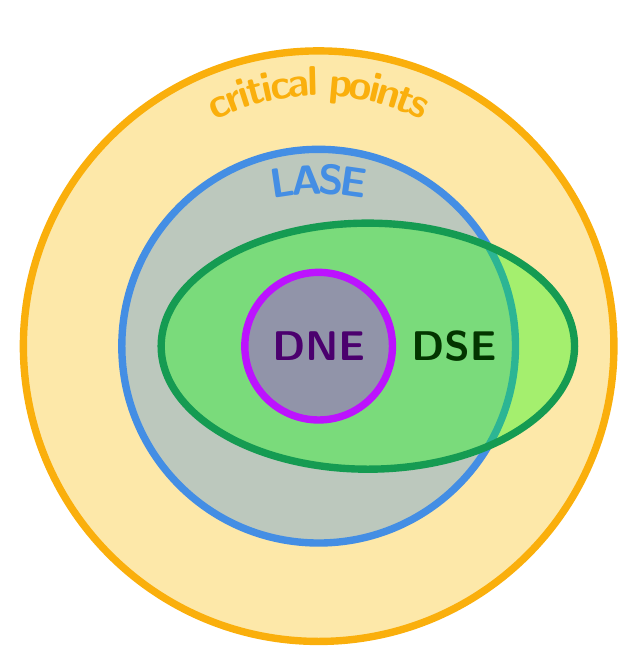}\label{fig:inclusions_a}} \hspace{9mm}
  \subfloat[][$\tau\rightarrow \infty$]{\includegraphics[width=.2\textwidth]{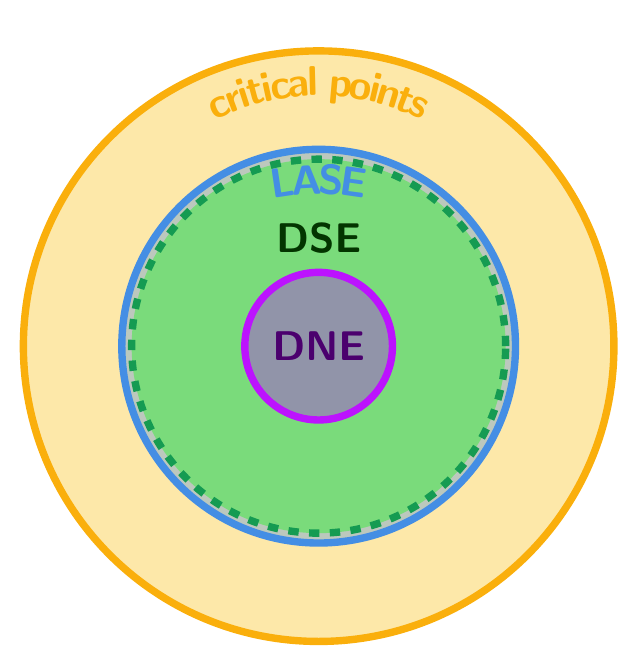}\label{fig:inclusions_b}}
  \caption{Graphical representation of the known stability results on $\tau$-{\gda} in relationship to local equilibrium concepts with $\tau=1$ and $\tau \rightarrow \infty$. The acronyms in the figure are differential Nash equilibria (DNE), differential Stackelberg equilibria (DSE), and locally asymptotically stable equilibria (LASE). Note that the terminology of locally asymptotically stable equilibria refers to the set of stable critical points with respect to the system $\dot{x}=-\Lambda_{\tau} g(x)$ for the given $\tau$.
  \citet{fiez:2020icml} reported the subset relationship between differential Nash equilibria and differential Stackelberg equilibria and \citet{jin2019local} gave a similar characterization in terms of local Nash and local minmax. For the regime of $\tau=1$, \citet{daskalakis:2018aa, mazumdar2020gradient} presented the relationship between the set of differential Nash equilibrium and the set of locally asymptotically stable equilibrium, and~\citet{jin2019local} provided the relationship between the set of differential Stackelberg equilibrium and the set of locally asymptotically stable equilibrium. Finally, \citet{jin2019local} reported the characterization of the locally asymptotically stable equilibrium as $\tau\rightarrow \infty$. The missing pieces in the literature are results as a function of finite $\tau$, which we answer in this work definitively.}
  \label{fig:inclusions}
\end{figure*} 
\subsection{Preliminary Observations}
\label{sec:prelim_obs}
In Figure~\ref{fig:inclusions} we present a graphical representation of known results on the stability of gradient descent-ascent with timescale separation in continuous time, where we remark that such results nearly directly imply equivalent conclusions regarding the discrete time system $\tau$-{\gda} with a suitable choice of learning rate $\gamma_1$. The primary focus of past work has been on the edge cases of $\tau=1$ and $\tau\rightarrow \infty$. For $\tau=1$, the set of differential Nash equilibrium are stable, but differential Stackelberg equilibrium may be stable or unstable, and non-equilibrium critical points can be stable. As $\tau\rightarrow\infty$, the set of differential Nash equilibrium remain stable, each differential Stackelberg equilibrium is guaranteed to become stable, and each non-equilibrium critical point must be unstable. We fill the gap between the known results by providing results as a function of finite $\tau$. With an eye toward this goal, we now provide examples and preliminary results that illustrate the type of guarantees that may be achievable for a range of finite learning rate ratios.

To start off, we consider the set of differential Nash equilibrium. It is nearly immediate from the structure of the Jacobian that each differential Nash equilibrium is stable for $\tau=1$~\cite{mazumdar2020gradient, daskalakis:2018aa}. Moreover, \citet{jin2019local} showed that regardless of the value of $\tau\in(0,\infty)$, the set of differential Nash equilibrium remain stable. In other words, the desirable stability characteristics of differential Nash equilibrium are retained for any choice of timescale separation.
We state this result as a proposition for later reference and since our proof technique relies on the concept of quadratic numerical range~\cite{tretter2008spectral}, which has not appeared previously in this context. The proof of Proposition~\ref{lem:zsspecbdd} is provided in Appendix~\ref{sec:stabilityofDNE}.
\begin{proposition}
 Consider a zero-sum game $(f_1,f_2)=(f,-f)$ defined by $f\in C^r(X,\mb{R})$ for some $r\geq 2$.  Suppose that
 $x^\ast$ is a differential Nash equilibrium.
Then, $\spec(-J_\tau(x^\ast))\subset \mb{C}_-^\circ$ for all $\tau\in(0,\infty)$.
\label{lem:zsspecbdd}
\end{proposition}

\citet{fiez:2020icml} show that the set of differential Nash equilibrium is a subset of the set of differential Stackelberg equilibrium. In other words, any differential Nash equilibrium is a differential Stackelberg equilibrium, but a differential Stackelberg equilibrium need not be a differential Nash equilibrium. Moreover,~\citet{jin2019local} show that the result of Proposition~\ref{lem:zsspecbdd} fails to extend from differential Nash equilibria to the broader class of differential Stackelberg equilibrium.  Indeed, 
not all differential Stackelberg equilibrium are stable with respect to the continuous time limiting dynamics of gradient descent-ascent without timescale separation. However, as the following example demonstrates, differential Stackelberg equilibrium that are unstable without timescale separation can become stable for a range of finite timescale learning rate ratios. 
\begin{example}Within the class of zero-sum games, there exists differential Stackelberg equilibrium that are unstable with respect to $\dot{x}=-\g(x)$ and stable with respect to $\dot{x}=-\Lambda_{\tau} g(x)$ for all $\tau \in (\tau^{\ast}, \infty)$ where $\tau^{\ast}$ is finite. Indeed, consider the quadratic zero-sum game 
defined by the cost  \[f(x_1,x_2)=\frac{1}{2}\bmat{x_{1}\\ x_2}^\top\bmat{
-v & 0 &-v & 0\\ 
0 & \tfrac{1}{2}v & 0 & \tfrac{1}{2}v\\
-v & 0 & -\tfrac{1}{2}v & 0\\
0 & \tfrac{1}{2}v & 0 & -v} \bmat{x_1\\ x_2} \]
where $x_1, x_2\in \mb{R}^2$ and $v>0$.
The unique critical point of the game given by $x^{\ast}=(0,0)$ is a differential Stackelberg equilibrium since $g(x^{\ast})=0$, $\schurtt_1(J(x^{\ast}))=\text{\normalfont{diag}}(v, v/4)>0$ and $-D_2^2f(x^{\ast})=\text{\normalfont{diag}}(v/2, v)>0$. The spectrum of the Jacobian of $\Lambda_{\tau} g(x)$ is given by 
\[\spec(J_{\tau}(x^{\ast}))=\Big\{\frac{v(2\tau+1\pm \sqrt{4\tau^2 -8\tau+1})}{4}, \frac{v(\tau-2\pm \sqrt{\tau^2-12\tau +4})}{4}\Big\}.\]
Observe that for $\tau=1$, $\spec(J_{\tau}(x^{\ast}))=\{\tfrac{1}{4}(3\pm i\sqrt{3})v, \tfrac{1}{4}(-1\pm i\sqrt{7})v\} \not\subset \mb{C}_{+}^{\circ}$ for any $v>0$ so that the differential Stackelberg equilibrium $x^{\ast}$ is never stable for the choice of $\tau$. However, for any $v>0$, $\spec(J_{\tau}(x^{\ast}))\subset \mb{C}_{+}^{\circ}$ for all $\tau \in (2, \infty)$, meaning that the differential Stackelberg equilibrium $x^{\ast}$ is indeed stable with respect to the dynamics $\dot{x}=-\Lambda_{\tau} g(x)$ for a range of finite learning rate ratios.
\label{example:nonstablestack}
\end{example}

We explore Example~\ref{example:nonstablestack} further via simulations in Section~\ref{sec:quadgame_exp}. The key takeaway from Example~\ref{example:nonstablestack} is that it is clearly not always necessary for the timescale separation $\tau$ to approach infinity in order to guarantee the stability of a differential Stackelberg equilibrium and instead there exists a sufficient finite learning rate ratio. Put simply, the undesirable property of differential Stackelberg equilibria not being stable with respect to gradient descent-ascent without timescale separation can potentially be remedied with only a finite timescale separation.

It is well-documented that some stable critical points of the continuous time gradient descent-ascent limiting dynamics without timescale separation can lack game-theoretic meaning, as they may be neither a differential Nash equilibria nor differential Stackelberg equilibria~\cite{mazumdar2020gradient, daskalakis:2018aa, jin2019local}. The following example demonstrates that such undesirable critical points that are stable without timescale separation can become unstable for a range of finite learning ratios.
\begin{example}
     Within the class of zero-sum games, there exists non-equilibrium critical points that are stable with respect to $\dot{x}=-\g(x)$ and unstable with respect to $\dot{x}=-\Lambda_{\tau} g(x)$ for all $\tau \in (\tau_{0}, \infty)$ where $\tau_{0}$ is finite. 
     Indeed, consider a zero sum game defined by the cost \begin{equation}f(x_1,x_2)=\frac{1}{2}\bmat{x_1\\ x_2}^{\top}\bmat{\tfrac{1}{2}v & 0 & \tfrac{1}{2}v &0\\ 0&-\tfrac{1}{4}v & 0&\tfrac{1}{2}v\\ \tfrac{1}{2}v & 0 & \tfrac{1}{4}v & 0\\ 0 & \tfrac{1}{2}v & 0 & -\tfrac{1}{2}v}\bmat{x_1\\ x_2}\label{eq:quadunstable} \end{equation}
     where $x_1, x_2\in \mb{R}^2$ and $v>0$.
The unique critical point of the game given by $x^{\ast} = (0,0)$ is neither a differential Nash equilibrium nor a differential Stackelberg equilibrium since $D_1^2f(x^{\ast}) = \text{\normalfont{diag}}(v/2, -v/4)\ngtr 0$ and $-D_2^2f(x^{\ast}) = \text{\normalfont{diag}}(-v/4, v/2)\ngtr 0$. The spectrum of the Jacobian of $\Lambda_{\tau} g(x)$ is given by
\[\spec(J_{\tau}(x^{\ast}))=\Big\{\frac{v(2\tau-1\pm \sqrt{4\tau^2 -12\tau+1})}{8}, \frac{v(2-\tau\pm \sqrt{\tau^2-12\tau +4})}{8}\Big\}.\]
Given $\tau=1$, $\spec(J_{\tau}(x^{\ast}))=\{\tfrac{1}{8}(1\pm i\sqrt{7})v,\tfrac{1}{8}(1\pm i\sqrt{7})v\} \subset \mb{C}_{+}^{\circ}$ for any $v>0$ so that the non-equilibrium critical point $x^{\ast}$ is in fact stable for the choice of timescale separation $\tau$. However, for any $v>0$, $\spec(J_{\tau}(x^{\ast}))\not\subset \mb{C}_{+}^{\circ}$ for all $\tau \in (2, \infty)$, meaning that the non-equilibrium critical point $x^{\ast}$ is unstable with respect to the dynamics $\dot{x}=-\Lambda_{\tau} g(x)$ for a range of finite learning rate ratios.

The game construction from~\eqref{eq:quadunstable} is quadratic and as a result has a unique critical point. Games can be constructed in which critical points lacking game-theoretic meaning that are stable without timescale separation become unstable for all $\tau>\tau_0$ even in the presence of multiple equilibria.
Indeed, consider a zero-sum game defined by the cost 
\begin{equation}
\begin{split}
f(x_1,x_2)&=\tfrac{5}{4}\left(x_{11}^2+2x_{11}x_{21}+\tfrac{1}{2}x_{21}^2-\tfrac{1}{2}x_{12}^2+2x_{12}x_{22}-x_{22}^2\right)(x_{11}-1)^2 \\
& \quad \textstyle +x_{11}^2\big(\sum_{i=1}^2(x_{1i}-1)^2-(x_{2i}-1)^2\big).
\end{split}
\label{eq:nonquadunstable}
\end{equation}
This game has critical points at $(0,0,0,0)$, $(1,1,1,1)$, and $(-4.73, 0.28, -92.47, 0.53)$. The critical points $(1,1,1,1)$ and $(-4.73, 0.28, -92.47, 0.53)$ are differential Nash equilibria and are consequently stable for any choice of $\tau >0$. The critical point $x^{\ast}=(0,0,0,0)$ is neither a differential Nash equilibrium nor a differential Stackelberg equilibrium. Moreover, the Jacobian of $\Lambda_{\tau} g(x^{\ast})$ for the game defined by~\eqref{eq:quadunstable} with $v=5$ is identical to that for the game defined by~\eqref{eq:nonquadunstable}. As a result, we know that $x^{\ast}$ is stable without timescale separation, but $\spec(J_{\tau}(x^{\ast}))\not\subset \mb{C}_{+}^{\circ}$ for all $\tau \in (2, \infty)$ so that the non-equilibrium critical point $x^{\ast}$ is again unstable with respect to the dynamics $\dot{x}=-\Lambda_{\tau} g(x)$ for a range of finite learning rate ratios.
\label{example:stablenoneq}
\end{example}
We investigate the game defined in~\eqref{eq:nonquadunstable} from Example~\ref{example:stablenoneq} with simulations in Section~\ref{sec:poly_nn}. In an analogous manner to Example~\ref{example:nonstablestack}, Example~\ref{example:stablenoneq} demonstrates that it is not always necessary for the timescale separation $\tau$ to approach infinity in order to guarantee non-equilibrium critical points become unstable as there can exist a sufficient finite learning rate ratio. This is to say that the unwanted property of non-equilibrium critical points being stable without timescale separation can also potentially be remedied with only a finite timescale separation.

The examples of this section have provided evidence that there exists a range of finite learning rate ratios for which differential Stackelberg equilibrium are stable and a range of learning rate ratios for which non-equilibrium critical points are unstable. Yet, no result has appeared in the literature on gradient descent-ascent with timescale separation confirming this behavior in general. We focus on doing precisely that in the subsection that follows. Before doing so, we remark on the closest existing result. As mentioned previously~\citet{jin2019local} show that as $\tau \rightarrow \infty$, the set of stable critical points with respect to the dynamics $\dot{x}=-\Lambda_{\tau} g(x)$ coincide with the set of differential Stackelberg equilibrium. However, an equivalent result in the context of general singularly perturbed systems has been known  in the literature (cf.~ \citealt[Chap.~2]{kokotovic1986singular}). We give a proof based on this type of analysis because it reveals a new set of analysis tools to the study of game-theoretic formulations of machine learning and optimization problems; a proof sketch is given below while the full proof is given in Appendix~\ref{app_sec:simgrad_inf}.
\begin{proposition}
\label{prop:simgrad_inf}
Consider a zero-sum game $(f_1,f_2)=(f,-f)$ defined by $f\in C^r(X,\mb{R})$ for some $r\geq 2$. Suppose that $x^{\ast}$ is such that $\g(x^{\ast})=0$ and $\det(D_2^2f_2(x^{\ast}))\neq 0$. Then, as $\tau \rightarrow \infty$,  $\spec(J_\tau(x^\ast))\subset \mb{C}_+^\circ$ if and only if $x^\ast$ is a differential Stackelberg equilibrium. 
\end{proposition}
\begin{proof}[Proof Sketch.]
The basic idea in showing this result is that there is a (local) transformation of coordinates from the linearized dynamics of $\dot{x}=-\Lambda_{\tau} g(x)$, which we write as
\[\dot{x}=\bmat{A_{11} & A_{12}\\ -\tau A_{12}^\top & \tau A_{22}}x,\]
in a neighborhood of a critical point to an upper triangular system that depends parametrically on $\tau$ and hence, the asymptotic behavior is readily obtainable from the block diagonal components of the system in the new coordinates. Indeed, consider the change of variables $z=x_2+L(\tau^{-1})x_1$ for the second player so that
\begin{equation}\bmat{\dot{x}_1\\ \dot{z}}=\bmat{A_{11}-A_{12}L(\tau^{-1}) & A_{12}\\ R(L,\tau) & A_{22}+\tau^{-1}L(\tau^{-1})A_{12}}\bmat{x_1\\ z}\label{eq:transformed2}
\end{equation}
where 
\[R(L,\tau)=-A_{12}^\top-A_{22}L(\tau^{-1})+\tau^{-1}L(\tau^{-1})A_{11}-\tau^{-1}L(\tau^{-1})A_{12}L(\tau^{-1})=0\]
A transformation of coordinates $L(\tau)$ such that $R(L,\tau)=0$ always exists (cf.~Lemma~\ref{lemma:exist_lemma}, Appendix~\ref{app_sec:simgrad_inf}). Hence, 
the characteristic equation of \eqref{eq:transformed2} can be expressed as
\[\chi(s,\tau)=\tau^{\m}\chi_s(s,\tau)\chi_f(p,\tau)=0\]
where $\chi_s(s,\tau)=\det(sI-(A_{11}-A_{12}L(\tau^{-1})))$ and $\chi_f(p,\tau)=\det(pI-(A_{22}+\tau^{-1}A_{12}L(\tau^{-1})))$ with $p=s\tau^{-1}$. As $\tau\to\infty$, $L(\tau^{-1})\to L(0)=-A_{22}^{-1}A_{12}^\top$.
Consequently, $n$ of the eigenvalues of $\dot{x}=-\Lambda_{\tau} g(x)$, denoted by $\{\lambda_1, \dots, \lambda_{\m_1}\}$, are the roots of the slow characteristic equation $\chi_{s}(s, \tau) = 0$ and the rest of the eigenvalues $\{\lambda_{\m_1+1}, \dots, \lambda_{\m_1+\m_2}\}$ are denoted by $\lambda_{i} = \nu_{j}/\varepsilon$ for $i=\m_1+j$ and $j\in \{1, \dots, \m_2\}$ where $\{\nu_{1}, \dots, \nu_{\m_2}\}$ are the roots of the fast characteristic equation $\chi_f(p, \tau) = 0$. The roots of $\chi_s(s,\tau)$ are precisely those of the (first) Schur complement of $-J_\tau(x^\ast)$ while the roots of $\chi_f(p,\tau)$  are precisely those of $D_2^2f(x^\ast)$. 
\end{proof}
This simple transformation of coordinates to an upper triangular dynamical system shown in \eqref{eq:transformed2} leads immediately to the asymptotic result in Proposition~\ref{prop:simgrad_inf}. It also shows that if the eigenavlues of $\schurtt_1(J_\tau(x^\ast))$ are distinct\footnote{Distinct eigenvalues is a generic property in the space of $n\times n$ real matrices.} and similarly, so are those of $D_2^2f(x^\ast)$ (although,  $\schurtt_1(J_\tau(x^\ast))$ and $D_2^2f(x^\ast)$ are allowed to have eigenvalues in common), then the asymptotic results from Proposition~\ref{prop:simgrad_inf} imply the following approximations for the elements of $\spec(J_\tau(x^\ast))$:
\begin{align*}\lambda_i&=\lambda_i(\schurtt_1(J_\tau(x^\ast))+O(\tau^{-1}), \ i=1,\ldots, \m_1,\\
\lambda_{j+\m_1}&=\tau(\lambda_j(-D_2^2f(x^\ast))+O(\tau^{-1})), \ j=1,\ldots, \m_2.
\end{align*}
This follows simply by observing that when the eigenvalues are distinct, the derivatives $ds/d\tau$ and $dp/d\tau$ are well-defined by the implicit mapping theorem and the total derivative of $\chi_s(s,\tau)$ and $\chi_f(p,\tau)$, respectively.

\subsection{Necessary and Sufficient Conditions for Stability}
\label{sec:stability_main}
The proof of Proposition~\ref{prop:simgrad_inf} provides some intuition for the next result, which is one of our main contributions. 
Indeed, as shown in \citet[Chap.~2]{kokotovic1986singular},  as $\tau\rar \infty$ the first $\m_1$ eigenvalues of $\dot{x}=-\Lambda_{\tau} g(x)$ tend to fixed positions in the complex plane defined by the eigenvalues of $-\schur_1=-(D_1^2f(x^\ast)-D_{12}f(x^\ast)(D_2^2f(x^\ast))^{-1}D_{12}^\top f(x^\ast))$, while the remaining $\m_2$ eigenvalues tend to infinity, with the linear rate $\tau$, along as asymptotes defined by the eigenvalues of $D_2^2f(x^\ast)$. The asymptotic splitting of the spectrum provides some intuition for the following result.

\begin{theorem}
Consider a zero-sum game $(f_1,f_2)=(f,-f)$ defined by $f\in C^r(X,\mb{R})$ for some $r\geq 2$. Suppose that $x^{\ast}$ is such that $\g(x^{\ast})=0$ and $\schurtt_1(J(x^\ast))$ and $D_2^2f_2(x^{\ast})$ are non-singular. There exists a $\tau^\ast\in(0,\infty)$ such that  $\spec(-J_\tau(x^\ast))\subset \mb{C}_-^\circ$ for all $\tau\in (\tau^\ast,\infty)$ if and only if $x^\ast$ is a differential Stackelberg equilibrium. 
\label{thm:iffstack}
\end{theorem}
Before getting into the proof sketch, we provide some intuition for the construction of $\tau^\ast$ and along the way revive an old analysis tool from dynamical systems theory which turns out to be quite powerful in analyzing stability properties of parameterized systems.

\paragraph{Construction of $\tau^\ast$.} There is still the question of how to construct such a $\tau^\ast$ and do so in a way that is as tight as possible. Recall Theorem~\ref{thm:exponentialstability} which states that a matrix is exponentially stable if and only if there exists a symmetric positive definite $P=P^\top>0$ such that $PJ_\tau(x^\ast)+J^\top_\tau(x^\ast)P>0$. The operator $\mc{L}(P)=J_\tau^\top(x^\ast)P+PJ_\tau(x^\ast)$ is known as the Lyapunov operator. Given a positive definite $Q=Q^\top>0$, $-J_\tau(x^\ast)$ is stable if and only if there exists a unique solution $P=P^\top$ to 
\begin{equation}((J_\tau^\top(x^\ast)\otimes I)+(I\otimes J_\tau^\top(x^\ast)))\vec(P)=(J_\tau^\top(x^\ast)\oplus J_\tau^\top(x^\ast))\vec(P)=\vec(Q)
\label{eq:vectorizelyap}
\end{equation}
where $\otimes$ and $\oplus$ denote the Kronecker product and Kronecker sum, respectively.\footnote{See \citet{magnus1988linear, lancaster1985theory} for more detail on the definition and properties of these mathematical operators, and Appendix~\ref{app_sec:iffstack} for more detail directly related to their use in this paper.} 
The existence of a unique solution $P$ occurs if and only if $J_\tau^\top$ and $-J_\tau^\top$ have no eigenvalues in common. Hence, using the fact that eigenvalues vary continuously, if we imagine varying $\tau$ and examining the eigenvalues of the map $(J_\tau^\top(x^\ast)\oplus J_\tau^\top(x^\ast))$, this will tell us the range of $\tau$ for which $\spec(-J_\tau(x^\ast))$ remains in $\mb{C}_-^\circ$. 

 This method of varying parameters and determining when the roots of a polynomial (or correspondingly, the eigenvalues of a map) cross the boundary of a domain uses what is known as a \emph{guardian} or \emph{guard map} (cf.~\citet{saydy1990guardian}). In particular, the guard map  provides a certificate that the roots of a polynomial lie in a particular guarded domain for a range of parameter values. 
Formally, let $\mc{X}$ be the set of all $n\times n$ real matrices or the set of all polynomials of degree $n$ with real coefficients.  Consider $\mc{S}$ an open subset of $\mc{X}$ with closure $\bar{\mc{S}}$ and boundary $\partial \mc{S}$.
The map $\nu: \mc{X}\rar \mb{C}$ is said to be a guardian map for $\mc{S}$ if for all $x\in \bar{\mc{S}}$, 
\[\nu(x)=0 \ \Longleftrightarrow\ x\in \partial \mc{S}.\]
Consider an open subset $\Omega$ of the complex plane that is symmetric with respect to the real axis (e.g., the open left-half complex plane $\mb{C}_-^\circ$). Then, elements of $\mc{S}(\Omega)=\{A\in \mb{R}^{n\times n}:\ \spec(A)\subset \Omega\}$ are said to be stable relative to $\Omega$.
Given a  pathwise connected subset $U$ of $\mb{R}$, a domain $\mc{S}(\Omega)$ and a guard map $\nu$, it is known that the family $\{A(\tau):\ \tau \in U\}$ is stable relative to $\Omega$ if and only if $(i)$ it is nominally stable---i.e., $A(\tau_0)\in \mc{S}(\Omega)$ for some $\tau_0\in U$---and $(ii)$ $\nu(A(\tau))\neq 0$ for all $\tau\in U$ \cite[Prop.~1]{saydy1990guardian}.
\begin{figure}
    \centering
 \includegraphics[width=0.6\textwidth]{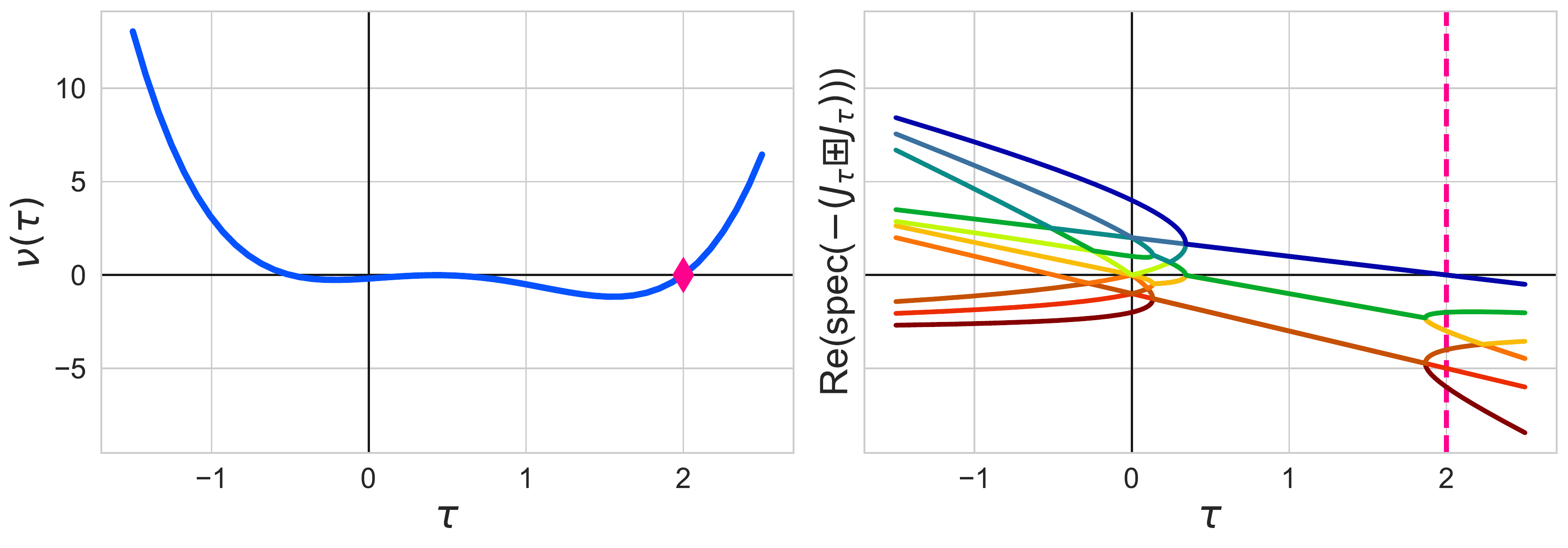}\hfill \includegraphics[width=0.333\textwidth]{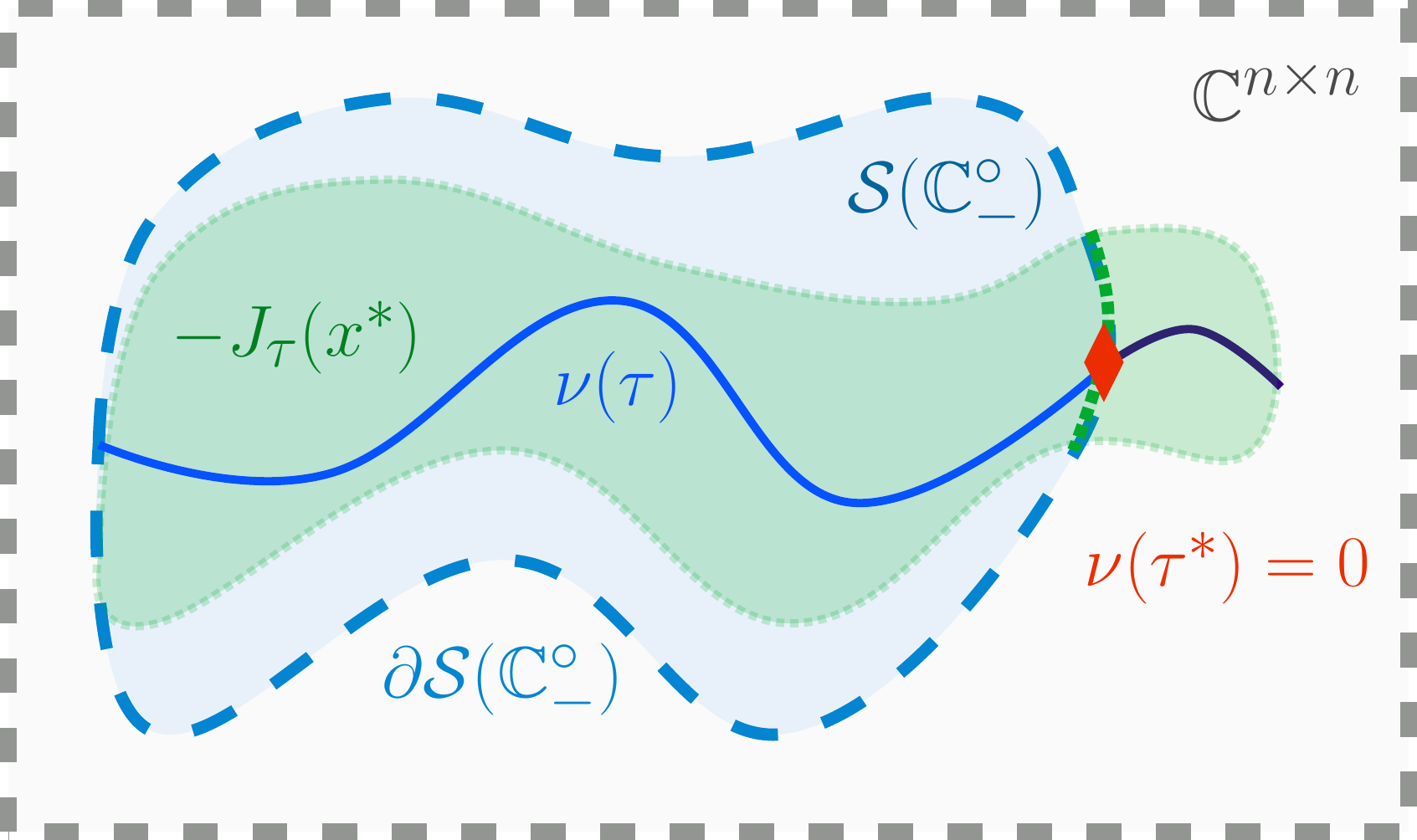}
   
    \caption{Guard map $\nu(\tau)$ and real parts of the eigenvalues of the vectorized Lyapunov operator $-(J_\tau(x^\ast)\boxplus J_\tau(x^\ast))$ using the reduction via the duplication matrix for the quadratic example given in Example~\ref{example:nonstablestack}. The largest real positive root of $\nu(\tau)$ is $\tau^\ast=2$ and in the right plot, we see that all the real parts of the eigenvalues of the Lyapunov operator are negative indicating stability. The right most graphic is a cartoon visualization of the guard map method: the outer grey region represents $\mb{C}^{n\times n}$, the blue region represents the Hurwitz stable $n\times n$ matrices $\mc{S}(\mb{C}_-^\circ)$, the green region represents the parameterized class of matrices $\{J_\tau(x^\ast)\}_\tau$, and the curve cutting through the regions is the guard map $\nu(\tau)$. The goal is to fine the subset of $\{J_\tau(x^\ast)\}_\tau$ that lie within $\mc{S}(\mb{C}_-^\circ)$, which can be done by reducing the problem to finding the roots of $\nu(\tau)$. }
    \label{fig:guardmaplyap}
\end{figure}
 To build intuition for the guard map, consider the scalar game on $\mb{R}^2$ so that the Jacobian at a critical point has the structure
\begin{equation}J_\tau(x)=\bmat{a & b\\
-\tau b &\tau d}.
\label{eq:jtauscalar}
\end{equation}
 It is known that a critical point $x$ is stable if $\det(J_\tau(x))>0$ and $\tr(J_\tau(x))>0$.
Thus, it is fairly easy to see that $\nu:A\mapsto \det(A)\tr(A)$ is a guard map for the $2\times 2$ Hurwitz stable matrices $\mc{S}(\mb{C}_-^\circ)$. Now, the trace operator can be generalized using a \emph{bialternate product}, which is denoted $A\odot B$ for matrices $A$ and $B$ and defined by $A\odot B=\tfrac{1}{2}(A\otimes B+B\otimes A)$ so that $2(A\odot I)=A\oplus A$ (cf.~\citet[Sec. 4.4.4]{govaerts2000numerical}). For $2\times 2$ matrices such as $J_\tau(x)$ in \eqref{eq:jtauscalar}, $2(A\odot I))=a+\tau d=\tr(A)$. Hence, $\nu:A\mapsto \det(A)\det(2(A\odot I))$ generalizes the map $\nu:A\mapsto \det(A)\tr(A)$ to an $n\times n$ matrix. Replacing $A$ with the parameterized family of matrices $-J_\tau(x^\ast)$, we have the guard map $\nu(\tau)=\det(-J_\tau(x^\ast))\det(2(-J_\tau(x^\ast)\odot I))$. It is fairly easy to see that this polynomial in $\tau$ also guards the open left-half complex plane $\mb{C}_-^\circ$; details are given in the full proof in Appendix~\ref{app_sec:iffstack}.  In fact, using the Schur complement formula,
\[\det(-J_\tau(x^\ast))=\tau^{n_2}\det( D_2^2f(x^\ast))\det(-\schur_1)\]
so that if $x^\ast$ is a non-degenerate (a condition implied by the hyperbolicity of $x^\ast$ for $\schur_1$ and $D_2^2f(x^\ast)$), $\det(-J_\tau(x^\ast))$ does not change the properties of the guard map. In particular,  the values of $\tau\in (0,\infty)$ where $\nu(\tau)=0$ does not depend on $\det(-J_\tau(x^\ast))$.  Hence, we can use the reduced guard map \[\nu(\tau)=\det(2(-J_\tau(x^\ast)\odot I))=\det(-J_\tau(x^\ast)\oplus (-J_\tau(x^\ast)).\]
Reflecting back to \eqref{eq:vectorizelyap}, we see that this guard map in $\tau$ is closely related to the vectorization of the Lyapunov operator and of course, this is not a coincidence.
For any symmetric positive definite $Q=Q^\top>0$, there will be a symmetric positive definite solution $P=P^\top>0$ of the Lyapunov equation 
\[[-(J_\tau^\top(x^\ast)\oplus J_\tau^\top(x^\ast))]\vec(P)=\vec(-Q)\]
  if and only if the operator $-(J_\tau(x^\ast)\oplus J_\tau(x^\ast))$ is non-singular. In turn, this is equivalent to $\det(-(J_\tau(x^\ast)\oplus J_\tau(x^\ast)))\neq 0$. Hence, to find the range of $\tau$ for which, given any $Q=Q^\top>0$, the solution switches from a positive definite $P=P^\top>0$ to a negative definite $P=P^\top<0$ we need to find the value of $\tau$ such that $\nu(\tau)=\det(-(J_\tau(x^\ast)\oplus J_\tau(x^\ast)))=0$---i.e., where it hits the boundary $\partial \mc{S}(\mb{C}_-^\circ)$. 
\begin{proof}[Proof Sketch for Theorem~\ref{thm:iffstack}.]
The `necessary' direction follows directly from the above observation, while the    `sufficiency' direction follows by construction.

We leverage the guard map as described above to construct $\tau^\ast$. Define $\boxplus$ as an operator that generates an $\tfrac{1}{2}\m(\m+1)\times \tfrac{1}{2}\m(\m+1)$ matrix from a matrix $A\in \mb{R}^{\m\times \m}$ such that
\[A\boxplus A=H_\m^{+} (A\oplus A)H_\m\]
where $H_\m^{+}=(H_\m^\top H_\m)^{-1}H_\m^\top$ is the (left) pseudo-inverse of $H_\m$, a full column rank duplication matrix (cf.~Appendix~\ref{app_sec:convergencerate}) which maps 
a $\tfrac{\m}{2}(\m+1)$ vector to a $\m^2$ vector generated by applying $\vec(\cdot)$ to a symmetric matrix  and it is designed to respect the vectorization map $\vec(\cdot)$.\footnote{The intuition can be gained simply by examining the map $\nu(\tau)=\det(-(J_\tau(x^\ast)\otimes J_\tau(x^\ast)))$, however, this does not produce a tight estimate of $\tau^\ast$ and requires more computation due to the redundancies from symmetries in the subcomponents of $-J_\tau(x^\ast)$.} 
It is fairly straightforward to see that the Kronecker sum $A\oplus A=A\otimes I+I \otimes A$ has spectrum $\{\lambda_j+\lambda_i\}$ where $\lambda_i,\lambda_j\in \spec(A)$. The operator $A\boxplus A$ is simply a more computationally efficient expression of $A\oplus A$, and as such the eigenvalues of $A\boxplus A$ are those of $A\oplus A$ removing redundancies. We use $A\boxplus A$ specifically because of its computational advantages in computing $\tau^\ast$.

We show in Lemma~\ref{lem:boxplus} (Appendix~\ref{app_sec:convergencerate}) that $\nu: A \mapsto \det(A\boxplus A)$ guards the set of $\m\times \m$ Hurwitz stable matrices $\mc{S}(\mb{C}_-^\circ)$.
We then extend this guard map to the parametric guard map $\nu(\tau)=\det(-(J_{\tau}(x^\ast)\boxplus J_\tau(x^\ast)))$. Indeed, if we consider the subset of the family of matrices parameterized by $\tau$ that lies in $\mc{S}(\mb{C}_-^\circ)$, then for any $\tau$ such that $-J_\tau(x^\ast)$ is in this subset, we have that $\nu(\tau)=0$ if and only if $-(J_\tau(x^\ast)\boxplus J_\tau(x^\ast))$ is singular if and only if $-J_\tau(x^\ast)\in \partial\mc{S}(\mb{C}_-^\circ)$. This shows that $\nu(\tau)$ guards the space of $n\times n$ Hurwitz stable matrices $\mc{S}(\mb{C}_-^\circ)$. The map $\nu(\tau)$ defines a polynomial in $\tau$ and to determine the range of $\tau$ such that $\spec(-J_\tau(x^\ast))\subset \mb{C}_-^\circ$, we need to find the value of $\tau$ such that $\nu(\tau)=0$.
Towards this end, the guard map $\nu(\tau)$ can be further decomposed by applying the Schur determinant formula to $-(J_{\tau}(x^\ast)\boxplus J_\tau(x^\ast))$. This gives rise to a polynomial of the form 
\[\nu(\tau)=\tau^{\m_2(\m_2+1)/2}\det(D_2^2f(x^\ast)\boxplus D_2^2f(x^\ast))\det(\schur_1\boxplus \schur_1)\det(\tau I_{\m_1\m_2}-B)\]
for some $\m_1\m_2\times \m_1\m_2$ matrix $B$. Hence, the problem of determining the value of $\tau$ such that $\nu(\tau)=0$ (i.e., where the polynomial meets the boundary of $\mc{S}(\mb{C}_-^\circ)$) is reduced to an eigenvalue problem in $\tau$ for the matrix $B$. This value of $\tau$ is precisely the value $\tau^\ast$ and since its derived from an eigenvalue problem it is precisely $\tau^\ast=\lambda_{\max}^+(B)$ where $\lambda_{\max}^+(\cdot)$ is the largest positive real eigenvalue of its argument if one exists and otherwise its zero (meaning that the matrix $-J_\tau(x^\ast)$ is stable for all $\tau\in (0,\infty)$). The expression for the matrix $B$ is given in the fill proof contained in Appendix~\ref{app_sec:iffstack}.
\end{proof}

\ifinstability
The following result (which was independently derived by \citet{daleckii2002stability} and \citet{ostrowski1962some}) extends Lyapunov's theorem for stability (cf.~\ref{thm:exponentialstability}).
\begin{theorem}
Consider a matrix $A\in\mb{C}^{n\times n}$. If $H$ is Hermitian and satisfies
\[AH+HA^\ast=W, \ W>0\]
then $H$ is non-singular and $H$ and $A$ have the same the number of eigenvalues with negative real parts, positive real parts and zero real parts respectively.
\label{thm:genlyap}
\end{theorem}
\begin{proposition}
Consider a hyperbolic fixed point $x^\ast$ of $g(x)$ at which  $\det(D_1^2f(x^\ast)-D_{12}f(x^\ast)(D_2^2f(x^\ast))^{-1}D_{21}f(x^\ast))\neq 0$ and $\det(D_2^2f(x^\ast))\neq 0$. 
\end{proposition}

\fi

\begin{corollary}
\label{cor:suptau}
Consider a zero-sum game $\mc{G}=(f,-f)$ with $f\in C^r(X, \mb{R})$ for some $r\geq 2$. Suppose that the assumptions of Theorem~\ref{thm:iffstack} hold and that the set of differential Stackelberg equilibria, denoted ${\tt DSE}(\mc{G})$, is finite.   Let $\tau^\ast=\max_{x^\ast\in {\tt DSE}(\mc{G})}\tau(x^\ast)$ where $\tau(x^\ast)$ is the value of $\tau$ obtained via Theorem~\ref{thm:iffstack} for each individual critical point $x^\ast\in {\tt DSE}(\mc{G})$. Then,  for all $\tau\in (\tau^\ast,\infty)$ and $x^\ast\in {\tt DSE}(\mc{G})$, $\spec(-J_\tau(x^\ast))\subset \mb{C}_-^\circ$.
\end{corollary}
In short,  selecting the maximum value of $\tau^\ast$ over the finite set of equilibria guarantees that the local linearization of  $\dot{x}=-\Lambda_{\tau} g(x)$ around any  differential Stackelberg equilibria is stable, and hence, the nonlinear system is locally stable around each of these critical points. 

\paragraph{New algebraic tools for analysis at the intersection of game theory and machine learning.} Before moving on, we remark on the utility of the algebraic tools we use in the proof for Theorem~\ref{thm:iffstack}. Indeed, the guard map concept is extremely powerful for understanding stability of parameterized families of dynamical systems, and it is not limited to single parameter families. Hence, there is potential to extend the above results to games with more than two players or additional parameters. In fact, we do exactly this in Section~\ref{sec:gans} where we present results for GANs trained with gradient-penalty type regularizers for the discriminator. Moreover, it is fairly easy to construct analogous guard maps for non-zero sum games. Many of the tools and constructions readily extend. We leave these results to a different paper so as to not create too much clutter in the present work. 
\subsection{Sufficient Conditions for Instability} 
\label{sec:instability}
Note that Theorem~\ref{thm:iffstack} also implies that for any stable spurious critical points, meaning non-Nash/non-Stackelberg equilibria, there is no finite $\tau^\ast$ such that  $\spec(-J_\tau(x^\ast))\subset \mb{C}_-^\circ$ for all $\tau\in (\tau^\ast, \infty)$. In particular, there exists at least one finite, positive value of $\tau$ such that $\spec(-J_\tau(x^\ast))\not\subset \mb{C}_-^\circ$ since the only critical point attractors are differential Stackelberg equilibria for large enough finite $\tau$.  We can extend this result to address the question of whether or not there exists a finite learning rate ratio such that for all larger learning rate ratios $-J_\tau(x^\ast)$ has at least one eigenvalue with strictly positive real part, thereby implying that $x^\ast$ is unstable.

\begin{theorem}[Instability of spurious critical points] 
Consider a zero sum game $(f_1,f_2)=(f,-f)$ defined by $f\in C^r(X,\mb{R})$ for some $r\geq 2$. 
Suppose that $x^\ast$ is any  stable critical point  of $\dot{x}=-g(x)$ which is not a differential Stackelberg equilibrium. There exists a finite learning rate ratio $\tau_0\in (0,\infty)$ such that $\spec(-J_\tau(x^\ast))\not\subset \mb{C}_-^\circ$ for all $\tau \in (\tau_0,\infty)$.
\label{prop:instability}
\end{theorem}
\begin{proof}[Proof Sketch] The full proof is provided in Appendix~\ref{app_sec:proofinstabilityprop}. The proof leverages the fact that a nonlinear system is unstable if its linearization is unstable, meaning that the linearization has at least one eigenvalue with strictly positive real part. In our setting, this can be shown by leveraging the Lyapunov equation and Lemma~\ref{lem:landtis} 
which states that if  $\schurtt_1(-J(x^\ast))$ has no eigenvalues with zero real part, then there exists matrices $P_1=P_1^\top$ and $Q_1=Q_1^\top>0$ such that $P_1\schurtt_1(-J(x^\ast))+\schurtt_1(-J(x^\ast))P_1=Q_1$ where $P_1$ and $\schurtt_1(-J(x^\ast))$ have the same \emph{inertia}---i.e., the number of eigenvalues with positive, negative and zero real parts, respectively, are the same. An analogous statement applies to $-D_2^2f(x^\ast)$. From here, we construct a  matrix $P$  that is \emph{congruent} to $\mathrm{blockdiag}(P_1,P_2)$ and a matrix $Q_\tau$ such that $-PJ_\tau(x^\ast)-J_\tau^\top(x^\ast)P=Q_\tau$. Since $P$ and $\mathrm{blockdiag}(P_1,P_2)$ are congruent, Sylvester's law of inertia implies that they have the same number of eigenvalues with positive, negative, and zero real parts, respectively, so that in turn $P$ has at least one eigenvalue with strictly negative real part.  We then construct $\tau_0$ via an eigenvalue problem such that for all $\tau>\tau_0$, $Q_\tau>0$. Applying Lemma~\ref{lem:landtis}  %
again, we get that $J_\tau(x^\ast)$ has at least one eigenvalue with strictly negative real part so that $\spec(-J_\tau(x^\ast))\not\subset \mb{C}_-^\circ$ for all $\tau>\tau_0$.
\end{proof}

Unlike $\tau^\ast$ Theorem~\ref{thm:iffstack}, $\tau_0$ in Theorem~\ref{prop:instability} is not tight in the sense that $-J_\tau(x^\ast)$ may become unstable for $\tau<\tau_0$. The reason for this is that there are potentially many matrices $P_1$ and $Q_1$ that satisfy $\schurtt_1(J(x^\ast))P_1+P_1\schurtt_1(J(x^\ast))=Q_1$ such that $\schurtt_1(J(x^\ast))$ and $P_1$ have the same inertia; an analogous statement holds for $P_2$, $Q_2$ and $-D_2^2f(x^\ast)$. The choice of these matrices impact the value of $\tau_0$.  
Hence, the question of finding the exact value of $\tau$ beyond which a spurious stable critical point for $1$-{\gda} is unstable remains open.

\section{Provable Convergence of GDA with Timescale Separation}
\label{sec:convergencerates}

In this section, derive convergence guarantees for  $\tau$-{\gda} to differential Stackelberg equilibria in both the deterministic (i.e., where agents have oracle access to their individual gradients) and the stochastic (i.e., where agents have an unbiased estimator of their individual gradient) settings.

\subsection{Convergence Rate of Deterministic GDA with Timescale Separation}
\label{sec:gda_determ}
As a corollary to Theorem~\ref{thm:iffstack}, we first show that the discrete time $\tau$-{\gda} update is locally asymptotically stable for a range of learning rates $\gamma_1$. 

We need the following lemma to prove asymptotic convergence as well as the subsequent results on convergence rates.
\begin{lemma}
Consider a zero-sum game $(f_1,f_2)=(f,-f)$ defined by $f\in C^r(X, \mb{R})$ for some  $r\geq 2$. Suppose that $x^\ast$ is a differential Stackelberg equilibrium and that given $\tau>0$,  $\spec(-J_\tau(x^\ast))\subset\mb{C}_-^\circ$. Let $\gamma=\min_{\lambda\in \spec(J_\tau(x^\ast))} 2\Re(\lambda)/|\lambda|^2$. For any $\gamma_1\in(0,\gamma)$,  $\tau$-{\gda}  converges locally asymptotically.
\label{lem:convergencerate-asymptotic}
\end{lemma}
\begin{corollary}[Asymptotic convergence of $\tau$-{\gda}]Suppose the assumptions of Theorem~\ref{thm:iffstack} hold so  that 
$x^{\ast}$ is a critical points of $\g$  and 
$\schurtt_1(J(x^\ast))$ and $D_2^2f_2(x^{\ast})$ are non-singular. There exists a $\tau^\ast\in(0,\infty)$ such that  $\tau$-{\gda} with $\gamma_1\in(0,\gamma(\tau))$ where $\gamma(\tau)=\arg\min_{\lambda\in \spec(J_\tau(x^\ast))}2\Re(\lambda)/|\lambda|^2$ %
converges locally asymptotically for all $\tau\in(\tau^\ast,\infty)$ if and only if $x^\ast$ is a differential Stackelberg equilibrium. 
\label{cor:asymptoticiffstack}
\end{corollary}

\begin{figure}

    \centering
    \begin{tikzpicture}[scale=2]
    \def\eigone{4}
    \def\axcol{black!75}
    \draw[very thick, <->, \axcol] (-1.5,0) --  (2,0);
    \node at (1.8,-0.4/2) {\small{\tt$\mathrm{Re}(z)$}};
    \draw[very thick, <->,\axcol] (0,-1.5) --  (0,1.5);
    \node at (0.3,1.4) {\small$\Im(z)$};
 \def\eigone{1.5}  
 \def\vlen{0.05}
            \draw[very thick] (\eigone-\vlen,-\eigone+2+\vlen) -- (\eigone+\vlen,-\eigone+2-\vlen);
            \node at (\eigone-\vlen,-\eigone+2+\vlen+0.1) {$\lambda_2$};
    \draw[very thick] (\eigone-\vlen,-\eigone+2-\vlen) -- (\eigone+\vlen,-\eigone+2+\vlen);
      \draw[very thick] (\eigone-\vlen,\eigone-2+\vlen) -- (\eigone+\vlen,\eigone-2-\vlen);
    \draw[very thick] (\eigone-\vlen,\eigone-2-\vlen) -- (\eigone+\vlen,\eigone-2+\vlen);
    \node at (\eigone-\vlen,\eigone-2-\vlen-0.1) {$\lambda_1$};
 
  \def\eigone{-1.25}  
  \def\eigtwo{1.5/2}
        \draw[very thick] (\eigtwo-0.1/2,\eigone+0.1/2) -- (\eigtwo+0.1/2,\eigone-0.1/2);
    \draw[very thick] (\eigtwo-0.1/2,\eigone-0.1/2) -- (\eigtwo+0.1/2,\eigone+0.1/2);
      \node at (\eigtwo-0.1/2+0.25,\eigone-0.1/2+0.1) {$\lambda_3$};

       \draw[very thick,Gray, dotted] (0,0) -- (0.75,1.25);

       \draw[ very thick, black, fill=Gray!15, fill opacity=0.2] (0,0) circle (1cm);
       
       \draw[Gray] (1,-0.1/2) -- (1,0.1/2);
       \draw[Gray] (0.75,-0.1/2) -- (0.75,0.1/2);
       \node at (1.075,-0.25/2) {\tt $1$};
       
       \draw (1.5/8,2.5/8) arc (30.9638:11:1);
       \node at (0.75/2,0.45/2) {$\theta$};
       \node[Red] at (0.9,1.275) { ${\lambdam}$};
       
       \draw[Red, dashed, very thick, draw opacity=1] (0,0) -- (1.5/2,0);
         \draw[Red, very thick, draw opacity=1] (\vlen,2.5/2) -- (-\vlen,2.5/2);
       \node[Red] at (0.75/2,-0.35/2) {\small ${\mathrm{Re}({\lambdam})}$};
       \draw[Red, dashed, very thick, draw opacity=1] (0,0) -- (0,2.5/2);
       \node[Red, rotate=90] at (-0.35/2,1.2/2) {\small $\mathrm{Im}({\lambdam})$};
       \draw[Red,  very thick, draw opacity=1] (1.5/2,-\vlen) -- (1.5/2,\vlen);
                  
\draw[ thick,dashed, Blue] (0,0) circle (0.858cm);
\def\lamonereal{0.435}
\def\lamoneimg{0.76}
    \draw[ultra thick, Blue,<-, draw opacity=1] (\lamonereal,\lamoneimg) -- (0.75,1.25);
        \draw[ultra thick, Blue,<-, draw opacity=1] (\lamonereal,-\lamoneimg) -- (0.75,-1.25);
   \draw[very thick, Red] (\eigtwo-0.1/2,-\eigone+0.1/2) -- (\eigtwo+0.1/2,-\eigone-0.1/2);
          \draw[very thick,Red] (\eigtwo-\vlen,-\eigone-\vlen) -- (\eigtwo+\vlen,-\eigone+\vlen);  
    
    \def\lamtworeal{0.075}
\def\lamtwoimg{0.328}
     \draw[very thick, dashed, black!75!white,->, draw opacity=0.75] (1.5,0.5) -- (\lamtworeal,\lamtwoimg);
        \draw[very thick, dashed, black!75!white,->, draw opacity=0.75] (1.5,-0.5) -- (\lamtworeal,-\lamtwoimg);
    \end{tikzpicture}
    \caption{The inner maximization problem in \eqref{eq:gammaopt} is over a finite set $\spec(J_\tau(x^\ast))=\{\lambda_1, \ldots, \lambda_\m\}$ where $J_\tau(x^\ast)\in \mb{R}^{\m\times \m}$. As $\gamma\rar \infty$,
     $|1-\gamma\lambda_i|\rar 0$. The last $\lambda_i$ such that $1-\gamma\lambda_i$ hits the boundary of the unit circle in the complex plane---i.e., $|1-\gamma\lambda_i|=1$---gives us the optimal value of $\gamma=2\Re({\lambdam})/|{\lambdam}|^2=2\cos(\theta)/|{\lambdam}|$ and the element of $\spec(J_\tau(x^\ast))$ that achieves it (see blue arrows).} %
    \label{fig:lemmacartoon}
\end{figure}
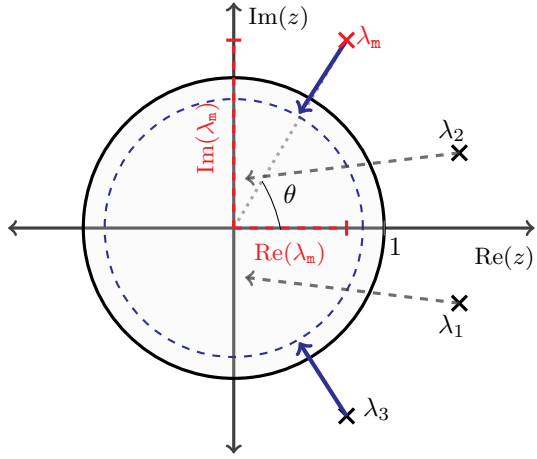

In addition to showing asymptotic convergence, we also provide an asymptotic convergence rate.
To prove  the main theorems on convergence rates for both differential Stackelberg and differential Nash equilibria, we use a common argument which is summarized in the lemma below.  
\begin{lemma}
Consider a zero-sum game $(f_1,f_2)=(f,-f)$ defined by $f\in C^r(X, \mb{R})$ for some  $r\geq 2$. Suppose that $x^\ast$ is a differential Stackelberg equilibrium and that given $\tau$,  $\spec(-J_\tau(x^\ast))\subset\mb{C}_-^\circ$. Let $\gamma=\min_{\lambda\in \spec(J_\tau(x^\ast))} 2\Re(\lambda)/|\lambda|^2$,
and ${\lambdam}=\arg\min_{\lambda\in \spec(J_\tau(x^\ast))}2\Re(\lambda)/|\lambda|^2$. For any $\alpha\in (0,\gamma)$,  $\tau$-{\gda} with learning rate $\gamma_1=\gamma-\alpha$ converges locally asymptotically at a rate of $O((1-\tfrac{\alpha}{4\beta})^{k/2})$ where $\beta=(2\Re({\lambdam})-\alpha|{\lambdam}|^2)^{-1}$.
\label{lem:convergencerate}
\end{lemma}
The full proof of the above lemma is provided in Appendix~\ref{app_sec:convergencerate}, and a short proof sketch with the main ideas is summarized below.
\begin{proof}[Proof Sketch.]
Consider a zero-sum game $(f,-f)$ as articulated in the lemma statement where $x^\ast$ is either a differential Stackelberg or differential Nash equilibrium and fix any $\tau$ such that $\spec(J_\tau(x^\ast))$. %

For the discrete time dynamical system $x_{k+1}=x_k-\gamma_1\Lambda_{\tau} \g(x_k)$, it is well known that if $\gamma_1$ is chosen such that $\rho(I-\gamma_1J_\tau(x^\ast))<1$, then $x_{k}$ locally asymptotically converges to $x^\ast$ (cf.~Proposition~\ref{prop:ort}, Appendix~\ref{app_sec:helperlemmas}). With this in mind, we formulate an optimization problem to find the upper bound $\gamma$ on the learning rate $\gamma_1$ such that for all $\gamma_1\in (0,\gamma)$, the spectral radius of the local linearization of the discrete time map is a contraction which is precisely $\rho(I-\gamma_1J_\tau(x^\ast))<1$. The optimization problem is given by \begin{equation}\gamma=\min_{\gamma>0}\left\{\gamma:\ \max_{\lambda\in \spec(J_\tau(x^\ast))}|1-\gamma \lambda|=1\right\}.\label{eq:gammaopt}
\end{equation}

The intuition is as follows. The inner maximization problem is over a finite set $\spec(J_\tau(x^\ast))=\{\lambda_1, \ldots, \lambda_\m\}$ where $J_\tau(x^\ast)\in \mb{R}^{\m\times \m}$. As $\gamma$ increases away from zero, each $|1-\gamma\lambda_i|$ shrinks in magnitude. The last $\lambda_i$ such that $1-\gamma\lambda_i$ hits the boundary of the unit circle in the complex plane (i.e., $|1-\gamma\lambda_i|=1$) gives us the optimal value of $\gamma$ and the element of $\spec(J_\tau(x^\ast))$ that achieves it. 
Examining the constraint, we have that for each $\lambda_i$, 
\[\gamma(\gamma|\lambda_i|^2-2\Re(\lambda_i))\leq 0\]
for any $\gamma>0$. As noted this constraint will be tight for one of the $\lambda$,
in which case
$\gamma=2\Re({\lambda})/|{\lambda}|^2$
since $\gamma>0$. Hence, by selecting 
$\gamma=\min_{\lambda\in \spec(J_\tau(x^\ast))} 2\Re(\lambda)/|\lambda|^2$,
we have
 that $|1-\gamma_1 \lambda|< 1$ for all $\lambda\in \spec(J_\tau(x^\ast))$ and any $\gamma_1\in (0,\gamma)$.

From here, we let ${\lambdam}=\arg\min_{\lambda\in \spec(J_\tau(x^\ast))} 2\Re(\lambda)/|\lambda|^2$ so that by fixing any $\alpha\in(0,1)$ and defining $\beta=(2\Re({\lambdam})-\alpha|{\lambdam}|^2)^{-1}$ and $\gamma_1=\gamma-\alpha$, we obtain the claimed convergence rate by standard arguments in numerical analysis. In particular, we apply an argument along the lines of Proposition~\ref{prop:ort} in Appendix~\ref{app_sec:helperlemmas}.
Indeed,
given the $\tau$-{\gda} update, $\rho(I-\gamma_1J_\tau(x^\ast))=1-\eta<1$ for some $\eta\in(0,1)$ implies
there exists a neighborhood $U$ of $x^\ast$ such that if 
$\tau$-{\gda} is initialized in $U$, 
$\|x_k-x^\ast\|\leq (1-\eta/4)^{k/2}\|x_0-x^\ast\|$ thereby giving the desired rate. The full details on this part of the argument can be found in Appendix~\ref{app_sec:convergencerate}.
\end{proof}

The above lemma provides a convergence rate given  a differential Stackelberg equilibrium  $x^\ast$ and a learning rate ratio $\tau$ such that $x^\ast$ is stable with respect to the dynamics $\dot{x}=-\Lambda_{\tau} g(x)$.  

The following theorem---which uses Lemma~\ref{lem:convergencerate} in its proof---characterizes the iteration complexity  for $\tau$-{\gda}. Specifically, the result leverages Theorem~\ref{thm:iffstack} to construct a finite $\tau^\ast\in(0,\infty)$ such that $-J_\tau(x^\ast)$ is (Hurwitz) stable, and then for any $\tau\in(\tau^\ast,\infty)$, Lemma~\ref{lem:convergencerate} implies a local asymptotic convergence rate.

\begin{theorem}
\label{thm:convergencerateDSE}
Consider a zero-sum game $(f_1,f_2)=(f,-f)$ defined by $f\in C^r(X,\mb{R})$ for $r\geq 2$ and let $x^\ast$ be a differential Stackelberg equilibrium of the game. 
There exists a $\tau^\ast\in(0,\infty)$ such that for any $\tau\in (\tau^\ast,\infty)$ and $\alpha\in (0,\gamma)$,  $\tau$-{\gda} with learning rate $\gamma_1=\gamma-\alpha$ converges locally asymptotically at a rate of $O((1-\tfrac{\alpha}{4\beta})^{k/2})$ where 
$\gamma=\min_{\lambda\in \spec(J_\tau(x^\ast))} 2\Re(\lambda)/|\lambda|^2$, ${\lambdam}=\arg\min_{\lambda\in \spec(J_\tau(x^\ast))}2\Re(\lambda)/|\lambda|^2$, and $\beta=(2\Re({\lambdam})-\alpha|{\lambdam}|^2)^{-1}$. Moreover, if $x^\ast$ is a differential Nash equilibrium, $\tau^\ast=0$ so that for any $\tau\in (0,\infty)$ and $\alpha\in(0, \gamma)$, $\tau$-{\gda} with $\gamma_1=\gamma-\alpha$ converges with a rate $O((1-\tfrac{\alpha}{4\beta})^{k/2})$.
\end{theorem}

\begin{proof}
To prove this result, we apply Theorem~\ref{thm:iffstack} to construct $\tau^\ast$ via the guard map $\nu(\tau)=\det(-J_\tau(x^\ast) \boxplus -J_\tau(x^\ast))$ such that for all $\tau\in (\tau^\ast, \infty)$, $\spec(J_\tau(x^\ast))\subset \mb{C}_+^\circ$. This guarantees  that $\spec(-J_\tau(x^\ast))\subset \mb{C}_-^\circ$ for any $\tau\in (\tau^\ast,\infty)$ and hence the nonlinear dynamical system
\[\dot{x}=-\Lambda_{\tau} \g(x)\]
is locally asymptotically (in fact, exponentially) stable by the Hartman-Grobman theorem~(cf. Theorem~\ref{thm:hg}). 
Therefore, for any $\tau\in (\tau^\ast,\infty)$, by Lemma~\ref{lem:convergencerate}, $\tau$-{\gda} converges with a rate of $O((1-\tfrac{\alpha}{4\beta})^{k/2})$.
\end{proof}
\begin{remark}
We note that as $\tau\to \infty$, the eigenvalues of $J_\tau(x^\ast)$ become real. In fact there exists a finite value of $\tau\in(\tau^\ast,\infty)$ after which $\spec(J_\tau(x^\ast))\subset \mb{R}$.  Hence, there is an opportunity to take advantage of momentum-based approaches when timescale separation via $\tau$ is introduced. This may provide some explanation for why the heuristics of timescale separation or unrolled updates in combination with momentum work in practice when training generative adversarial networks. We believe there may be potential to precisely characterize when the eigenvalues of the Jacobian become purely real as a function of $\tau$ by constructing a suitable guard map and this is a direction of future work.
\end{remark}

Theorem~\ref{thm:convergencerateDSE} directly implies a finite time convergence guarantee for obtaining an $\vep$-differential Stackelberg equilibrium---i.e., an point with an $\vep$-ball around a differential Stackelberg $x^\ast$.
\begin{corollary}
Given $\vep>0$, under the assumptions of Theorem~\ref{thm:convergencerateDSE}, $\tau$-{\gda} obtains an $\vep$--differential Stackleberg equilibrium in $\lceil \tfrac{4\beta}{\alpha}\log(\|x_0-x^\ast\|/\vep)\rceil$ iterations for any $x_0\in B_{\delta}(x^\ast)$ with $\delta={\alpha/(4L\beta)}$ where $L$ is the local Lipschitz constant of $I-\gamma J_{\tau}(x^\ast)$.
\label{cor:finitetimebound}
\end{corollary}

\paragraph{Comments on computing the neighborhood $B_\delta(x^\ast)$.}
We note that we have essentially given a proof that there exists a neighborhood 
on which $\tau$-{\gda} converges. Of course, due to the non-convexity of the problem in general, this neighborhood could be arbitrarily small. We provide an estimate of the neighborhood size using the local Lipschitz constant of the local linearization $I-\gamma_1J_\tau(x^\ast)$.  One way to better understand the size of this neighborhood is to use Lyapunov analysis, a tool which is well explored in the singular perturbation theory~\cite{kokotovic1986singular}. In particular, Lyapunov methods can be applied directly to the nonlinear system if one can construct Lyapunov functions for the fast and slow subsystems individually---also known as the boundary layer model and reduced order model. With these Lyapunov functions in hand, one can ``stitch" the two together (via convex combination) and show under some reasonable assumptions that this combined function is a Lyapunov function for the overall singularly perturbed system. The benefit of this analysis is that the Lyapunov function gives one an estimate of the region of attraction (via, e.g., the level sets); however, it is not easy to construct a Lyapunov function for a nonlinear system in general. We leave expanding to such methods to future work.

\paragraph{Comments on avoiding saddle points.}
Before  turning to the stochastic setting, we comment on saddle point avoidance in the deterministic setting. It was shown by \citet{mazumdar2020gradient} that gradient-based learning in continuous games with heterogeneous learning rates  avoids saddles on all but a set of measure zero initializations. Hence, $\tau$-{\gda} avoids saddles for almost every initialization. We also know that all differential Nash equilibria are locally asymptotically stable for zero-sum settings. Hence, there are no differential Nash equilibria that are saddle points of the dynamics $\dot{x}=-\Lambda_\tau g(x)$.
On the other hand, as Example~\ref{example:nonstablestack} shows, there are differential Stackelberg equilibria which correspond to saddle points of the dynamics for some choices of $\tau$---in particular, $\tau=1$ in that example. Theorem~\ref{thm:iffstack} and Corollary~\ref{cor:suptau}, however, implies that for a given zero-sum game (or minmax problem), there exists a finite $\tau^\ast$ such that all locally asymptotically stable equilibria are differential Stackelberg equilibria. Hence,
an `almost sure' saddle point avoidance result together with the local convergence guarantee provided by Theorem~\ref{thm:convergencerateDSE} provides a strong characterization of 
long-run learning behavior.

Avoidance of saddles nor the if and only if convergence guarantee of Theorem~\ref{thm:iffstack} are, however, enough to ensure avoidance of limit cycles. In fact, it is known that limit cycles can exist in zero sum games~\cite{daskalakis2017training, mazumdar2020gradient}. Understanding when such complex phenomena exist in games and determining how to ascribe meaning the behavior is an active area of study (see, e.g., the work of~\citet{papadimitriou2019game}).

\subsection{Convergence of Stochastic GDA with Timescale Separation}
\label{sec:stochastic}
In this section, we analyze convergence when players do not have oracle access to their gradients but instead have an unbiased estimator in the presence of zero mean, finite variance noise. Specifically, we show that the agents will converge locally asymptotically almost surely to a differential Stackelberg equilibrium. 

The stochastic form of the update is given by
\begin{equation}
    x_{k+1}=x_{k}-\gamma_{k}(\Lambda_{\tau} g(x_k)+w_{k+1})
    \label{eq:stochasticgda}
\end{equation}
where $w_{k+1}$ is a zero mean, finite variance random variable and $\{\gamma_k\}$ is the learning rate sequence.
\begin{assumption}
The stochastic process $\{w_{k}\}$ is a martingale difference sequence with respect to the increasing family of $\sigma$-fields defined by 
\[\mc{F}_k=\sigma(x_\ell, w_\ell, \ell\leq k), \ \forall k\geq 0,\]
so that $\mb{E}[w_{k+1}|\ \mc{F}_k]=0$ almost surely (a.s.) for all $k\geq 0$. Moreover, $w_{k}$ is square-integrable so that, for some constant $C>0$,
\[\mb{E}[\|w_{k+1}\|^2|\ \mc{F}_k]\leq C(1+\|x_k\|^2) \ \text{a.s.}, \ \forall k\geq 0.\]
\label{ass:noise}
\end{assumption}
We note that this assumption has been relaxed in the literature (cf.~\citet{thoppe2019concentration}), however simplicity, we state the theorem with the most accessible criteria. We remark below in the paragraph on extensions to concentration bounds on the nature of the relaxed assumptions.

\begin{theorem}
Consider a zero-sum game $(f,-f)$ such that $f\in C^r(X, \mb{R})$ for some $r\geq 2$. Suppose that Assumption \ref{ass:noise} holds and that $\{\gamma_k\}$ is square summable but not summable---i.e., $\sum_k\gamma_k^2<\infty$, yet $\sum_k\gamma_k=\infty$. 
For any  $\tau\in(0,\infty)$, the sequence $\{x_k\}$ generated by \eqref{eq:stochasticgda} converges to a, possibly sample path dependent,  internally chain transitive invariant set of $\dot{x}=-\Lambda_{\tau} g(x)$.
Moreover, if $x^\ast$ is a differential Stackelberg equilibrium, then there exists a finite $\tau^\ast\in(0,\infty)$ such that  %
$\{x_k\}$ almost surely converges locally asymptotically to $x^\ast$ for every
$\tau\in(0,\infty)$. 
\label{thm:stochastic_convergence}
\end{theorem}
\begin{proof}
The convergence of $\{x_k\}$ to a, possibly sample path dependent, compact connected internally chain transitive invariant set of $\dot{x}=-\Lambda_{\tau} g(x)$ follows from classical results in stochastic approximation theory (cf.~\citet[Chap.~2]{borkar2008stochastic}; \citet{benaim1996dynamical}).

Suppose that $x^\ast$ is a differential Stackelberg equilibrium. 
By Theorem~\ref{thm:iffstack}, there exists a finite $\tau^\ast\in (0,\infty)$ such that for all $\tau\in (\tau^\ast, \infty)$, $x^\ast$ is a locally exponentially stable equilibrium of the continuous time dynamics $\dot{x}=-\Lambda_{\tau} g(x)$---that is, $\spec(-J_\tau(x^\ast))\subset \mb{C}_-^\circ$ for all $\tau\in (\tau^\ast,\infty)$. 

Fix arbitrary $\tau\in (\tau^\ast,\infty)$. Since $\spec(-J_\tau(x^\ast))\subset \mb{C}_-^\circ$, 
$\det(-J_\tau(x^\ast))\neq0$  so that $x^\ast$ is an isolated critical point. 
Furthermore, exponentially stability of $x^\ast$ implies that there exists
a (local) Lyapunov function defined on a neighborhood of $x^\ast$ by the converse Lyapunov
theorem (cf.~\citet[Thm.~5.17]{sastry1999nonlinear} or \citet[Thm.~4.3]{krasovskii1963stability}). 
Let $U$ be the neighborhood of $x^\ast$ on which the local Lyapunov function is defined, such that $U$ contains no other critical points (which is possible since $x^\ast$ is isolated). That is, let $\Phi:U\to [0,\infty)$ be the local Lyapunov function defined on $U$ where $x^\ast\in U$, $\Phi$ is positive definite on $U$, and for all $x\in U$, $\tfrac{d}{dt}\Phi(x)\leq 0$ where equality holds for $z\in U$ if and only if $\Phi(z)=0$.    
 By Corollary 3~\cite[Chap.~2]{borkar2008stochastic}, $\{x_k\}$ converges to an internally chain transitive invariant set contained in $U$ almost surely. The only internally chain transitive invariant set in $U$ is $x^\ast$.
\end{proof}
\begin{corollary}
Consider a zero-sum game $(f,-f)$ such that $f\in C^2(X, \mb{R})$.  Suppose that Assumption \ref{ass:noise} holds and that $\{\gamma_k\}$ is square summable but not summable---i.e., $\sum_k\gamma_k^2<\infty$, yet $\sum_k\gamma_k=\infty$. If there exists   a finite $\tau^\ast\in(0,\infty)$ such that $\spec(-J_\tau(x^\ast))\subset \mb{C}_-^\circ$ for all $\tau\in(\tau^\ast,\infty)$, then $x^\ast$ is a differential Stackelberg equilibrium and $\{x_k\}$ almost surely converges locally asymptotically to $x^\ast$.
\end{corollary}

While (local) almost sure convergence in gradient descent-ascent~\cite{chasnov2019uai} to a critical point\footnote{To date it has not been shown that for a sufficient separation in timescale the only critical point attractors are local minmax.} in the stochastic setting, the result requires time varying learning rates with a sufficient separation in timescale. Specifically, the players need to be using learning rate sequences $\{\gamma_{i,k}\}$ for each $i\in \{1,2\}$ such that (without loss of generality) not only is it assumed that $\gamma_{1,k}=o(\gamma_{2,k})$, but also $\sum_{k}\gamma_{1,k}^2+\gamma_{2,k}^2<\infty$ and $\sum_k\gamma_{i,k}=\infty$ for each $i\in \{1,2\}$. The challenge with these assumptions on the learning rate sequences is that empirically the sequences that satisfy them result in poor behavior along the learning path such as getting stuck at saddle points or making no progress.  This is, in essence, due to the fact that the faster player---i.e., player 2 if $\gamma_{1,k}=o(\gamma_{2,k})$---equilibriates too quickly causing progress to stall. This can result in undesirable behavior such as vanishing gradients (so that the discriminator does not provide enough information for the generator to make progress), mode collapse, or failure to converge in practical applications such as generative adversarial networks. 

On the other hand, our convergence result gives a similar guarantee with less restrictive requirements on the stepsize sequence. In particular, only a single stepsize sequence is required (so that the algorithm can be viewed as a single timescale stochastic approximation update) as long as the fast player (who, without loss of generality, is player 2 in this paper) scales their estimated gradient by $\tau\in (\tau^\ast,\infty)$ where $\tau^\ast$ is as in Theorem~\ref{thm:iffstack}. 

\paragraph{Extensions to concentration bounds and relaxed assumptions on stepsizes.} It is possible to obtain concentration bounds and even finite time, high probability guarantees on convergence leveraging recent advances in stochastic approximation \cite{thoppe2019concentration, borkar2008stochastic, kamal2010convergence}. To our knowledge, the concentration bounds in  \cite{thoppe2019concentration}
require the weakest assumptions on learning rates---e.g., the stepsize sequence $\{\gamma_k\}$ needs only to satisfy $\sum_k\gamma_k=\infty$, $\lim_{k\to \infty}\gamma_k=0$, and $\sum_k\gamma_k\leq 1$. Specifically, since it is assumed, for the zero sum game $(f,-f)$, that $f\in C^2(X,\mb{R})$ and $x^\ast$ is a differential Stackelberg equilibrium, Theorem~\ref{thm:iffstack} implies that $x^\ast$ is a locally asymptotically stable attractor of $\dot{x}=-\Lambda_{\tau} g(x)$ for arbitrary fixed $\tau\in (\tau^\ast,\infty)$, and hence, the concentration bounds in Theorem 1.1 and 1.2 of \cite{thoppe2019concentration} directly apply. 

Furthermore, we note that in applications such as generative adversarial networks, while it has been observed that timescale separation heuristics such as unrolling or annealing the stepsize of the discriminator  work well, in the stochastic case,  summmable/square-summable assumptions on stepsizes are generally too restrictive in practice since they lead to a rapid decay in the stepsize 
which, in turn, can stall progress. On the other hand, stepsize sequences such as $\gamma_k=1/(k+1)^{\beta}$ for $\beta\in(0,1]$---a sequence which satisfies the assumptions posed in \cite{thoppe2019concentration}---tend not to have this issue of decaying too rapidly for appropriately chosen $\beta$, while also maintaining the guarantees of the theoretical results.  We state a convergence guarantee under these relaxed assumptions in Proposition~\ref{prop:concentration} which is contained in Appendix~\ref{app_sec:extensionsstochastic}. %

\section{Regularization with Applications to Adversarial Learning}
\label{sec:gans}
In this section, we focus on generative adversarial networks with regularization. Specifically, using the theory developed so far, we  extend the results in~\citet{mescheder2018training} to provide a convergence guarantee for a range of regularization parameters and learning rate ratios. 

As has been repeatedly observed in recent theoretical works on generative adversarial networks, the training dynamics of generative adversarial networks are not well understood even though we have seen impressive practical advances over the last few years.
In an attempt to address this, recent works---e.g., \cite{nagarajan2017gradient,mescheder2017numerics,fiez:2020icml,berard2020closer}
amongst others---study the optimization landscape of generative adversarial networks through the lens of dynamical systems theory which provides analysis tools for convergence based on the eigen-structure of the local linearization of the learning dynamics. \citet{nagarajan2017gradient} show, under suitable assumptions, that gradient-based methods for training generative adversarial networks are locally convergent assuming the data distributions are absolutely continuous. However, as observed by  \citet{mescheder2018training}, such assumptions not only may not be satisfied by many practical generative adversarial network training scenarios such as natural images, but it can often be the case that the data distribution is concentrated on a lower dimensional manifold. The latter characteristic leads to nearly purely imaginary eigenvalues and highly ill-condition problems.  

\citet{mescheder2018training} provide an explanation for observed instabilities consequent of the true data
distribution being concentrated on a lower dimensional manifold using  discriminator gradients orthogonal to the tangent space
of the data manifold.   Further, the authors introduce regularization via gradient penalties that leads to convergence guarantees under less restrictive assumptions than were previously known. Here, we further extend these results to show that convergence to differential Stackelberg equilibria is guaranteed under a wide array of hyperparameter configurations.

 Consider the training objective of the form
\[f(\theta,\omega)=\mb{E}_{p(z)}[\ell(\Dis(\Gen(z;\theta);\omega))]+\mb{E}_{p_{\mc{D}}(x)}[\ell(-\Dis(x;\omega))]\]
where $\Dis_\omega(x)$ and $\Gen_\theta(z)$ are discriminator and generator networks, respectively, and $p_{\mc{D}}(x)$ is the data distribution while $p(z)$ is the latent distribution. The goal of the generator is to minimize the above loss and the discriminator to maximize it. The mapping $\ell\in C^2(\mb{R})$ is some real-value function; e.g., a common choice is $\ell(x)=-\log(1+\exp(-x))$ as in the original generative adversarial network paper \cite{goodfellow2014generative}. To motivate both regularization and time-scale separation, consider the following prototypical example of a Dirac-GAN. 

\paragraph{Example: Dirac-GAN.} The Dirac-GAN is arguably as simple an example as one can construct that posses interesting and non-trivial \emph{degeneracies}. The parameter space for each the generator and discriminator is just $\mb{R}$---that is, the generator and discriminator have scalar ``network".
\begin{definition}
\label{def:diracgan}
The Dirac-GAN consists of a univariate
generator distribution $p_\theta = \delta_\theta$ and a linear discriminator
$\Dis(x; \omega) = \omega^\top x$. The true data distribution $p_{\mD}$ is given by a
Dirac-distribution concentrated at zero.
\end{definition}
The objective of the Dirac-GAN is 
\[f(\theta,\omega)=\ell(\omega\theta)+\ell(0).\]
Several choices exist for the mapping $\ell$ including $\ell(t)=-\log(1+\exp(-t))$ which leads to the Jensen-Shannon divergence between $p_{\theta}$ and $p_{\mc{D}}$.
As described in \cite{mescheder2018training}, when training a Dirac-GAN with vanilla {\gda}, 
the dynamics are oscillatory. This can immediately be verified by examining the eigenvalues of the Jacobian at the (unique) equilibrium $(\theta^\ast,\omega^\ast)=(0,0)$ which are purely complex taking on values $\pm i \ell'(0)$ since
\[J(0)=\bmat{0 & \ell'(0)\\ -\ell'(0) & 0}.\]
It is observed in Lemma 3.1~\citet{mescheder2018training} that {\gda} will not generally converge even with \emph{unrolling} of the discriminator, a proxy for timescale separation. We observe, however, that this is because the equilibrium is not hyperbolic, and hence it is not structurally stable~\cite{broer2010preliminaries}. In fact, introducing regularization remedies these issues. Under reasonable generative adversarial network assumptions, we show that the introduction of a gradient penalty based regularization to the discriminator does not change the set of critical points for the dynamics and, further,  for any learning rate ratio $\tau\in (0,\infty)$ and any positive, finite regularization parameter $\mu$, the continuous time limiting regularized learning dynamics remain stable, and hence, there is a range of learning rates $\gamma_1$ for which the discrete time update locally converges asymptotically.

\paragraph{Introducing Regularization.} Gradient penalties ensure that the discriminator cannot create
a non-zero gradient which is orthogonal to the data manifold without
suffering a loss.
 Introduced by \citet{roth2017stabilizing} and refined in \citet{mescheder2018training}, we  consider training generative adversarial networks with one of two fairly natural gradient-penalties used to regularize the discriminator: 
\begin{align}
R_1(\theta,\omega)&=\frac{\mu}{2}\mb{E}_{p_{\mc{D}}(x)}[\|\nabla_x \Dis(x;\omega)\|^2],\\ R_2(\theta,\omega)&=\frac{\mu}{2}\mb{E}_{p_\theta(x)}[\|\nabla_x\Dis(x;\omega)\|^2],
\label{eq:regularizers}
\end{align}
where, by a slight abuse of notation,
 $\nabla_x(\cdot)$ denotes the partial gradient with respect to $x$ of the argument $(\cdot)$ when the argument is the discriminator $\Dis(\cdot;\omega)$ in order prevent any confusion between the notation $D(\cdot)$ which we use elsewhere for derivatives.

Also following \citet{mescheder2018training}, we use relaxed assumptions---as compared to the work by \citet{nagarajan2017gradient}---which allow us to consider generative adversarial networks with data distributions that do not (locally) have the same support and hence, are concentrated
on lower dimensional manifolds, a commonly observed phenomena in practice~\cite{arjovsky2017wasserstein}.

To state these assumptions, we need some additional notation.
Let  $h_1(\theta)=\mb{E}_{p_{\theta}(x)}[\nabla_\omega \Dis(x;\omega)|_{\omega=\omega^\ast}]$ and $h_2(\omega)=\mb{E}_{p_{\mc{D}}(x)}[|\Dis(x;\omega)|^2+\|\nabla_x\Dis(x;\omega)\|^2]$. Define \emph{reparameterization manifolds} $\mc{M}_{\Gen}=\{\theta:\ p_{\theta}=p_{\mc{D}}\}$ and $\mc{M}_{\Dis}=\{\omega:\ h_2(\omega)=0\}$ and let $T_{\theta^\ast}\mc{M}_{\Gen}$ and $T_{\omega^\ast}\mc{M}_{\Dis}$ denote their respective tangent spaces at $\theta^\ast$ and $\omega^\ast$.
\begin{assumption}
Consider training a generative adversarial network via a zero-sum game defined by $f\in C^2(\mb{R}^{m_1}\times \mb{R}^{m_2},\mb{R})$  where $\Gen(\cdot;\theta)$ and $\Dis(\cdot;\omega)$ are the generator and discriminator networks, respectively, and $x=(\theta,\omega)\in \mb{R}^{m_1}\times \mb{R}^{m_2}$. Suppose that $x^\ast=(\theta^\ast,\omega^\ast)$ is an equilibrium. 
\begin{enumerate}[itemsep=0pt,topsep=2pt]
    \item[a.] At $(\theta^\ast,\omega^\ast)$, $p_{\theta^\ast}=p_{\mc{D}}$ and $\Dis(x;\omega^\ast)=0$ in some neighborhood of $\supp(p_{\mc{D}})$.
    \item[b.] The function $\ell\in C^2(\mb{R})$ satisfies $\ell'(0)\neq 0$ and $\ell''(0)<0$.
    \item[c.] There are $\epsilon$--balls $B_\epsilon(\theta^\ast)$ and $B_\epsilon(\omega^\ast)$ centered around $\theta^\ast$ and $\omega^\ast$, respectively, so that $\mc{M}_{\Gen}\cap B_\epsilon(\theta^\ast)$ and $\mc{M}_{\Dis}\cap B_\epsilon(\omega^\ast)$ define $C^1$-manifolds. Moreover, (i) if $w\notin T_{\theta^\ast}\mc{M}_{\Gen}$, then $w^\top \nabla_wh_1(\theta^\ast) w\neq 0$, and (ii) if $v\notin T_{\omega^\ast}\mc{M}_{\Dis}$, then $v^\top \nabla_\omega^2h_2(\omega^\ast)v\neq 0$.
\end{enumerate}
\label{ass:ganassump}
\end{assumption}
Recalling the observations in \citet{mescheder2018training} used for explaining the last of the three assumptions,  Assumption~\ref{ass:ganassump}.c(i) implies that the discriminator is capable of detecting deviations from the generator distribution in equilibrium, and Assumption~\ref{ass:ganassump}.c(ii)
implies that the manifold $\mc{M}_{\Dis} $ is sufficiently regular and, in particular, its (local) geometry is captured by the second (directional) derivative of $h_2$.

The Jacobian of the regularized dynamics, for either $j=1$ or $2$, is of the form
\begin{equation}
J_{(\tau,\mu)}(\theta,\omega)=\bmat{D_1^2f(\theta,\omega) & D_{12}f(\theta,\omega)\\ -\tau D_{12}^\top f(\theta,\omega) & \tau(-D_2^2f(\theta,\omega)+\mu D_2R_j(\theta,\omega))}
    \label{eq:jacreg}
\end{equation}
where we use the subscript $(\tau,\mu)$ to denote the parameter dependence on the learning rate ratio $\tau$ and regularization parameter $\mu$. For shorthand, we often replace $(\theta,\omega)$ simply with the variable $x$. 

It is straightforward to compute the block components of the Jacobian. 
Observe that Assumption~\ref{ass:ganassump}.a~implies that $D_1^2f(\theta^\ast, \omega^\ast)$ is identically zero, and hence $x^\ast=(\theta^\ast, \omega^\ast)$ is never a differential Nash equilibrium. However, we show that $x^\ast$ is not only a differential Stackelberg equilibrium, but also characterize the learning rate ratio and regularization parameter range for which $x^\ast$ is (locally) stable with respect to $\tau$-{\gda} and give a convergence rate. 
\begin{theorem}
Consider  training a generative adversarial network  via a zero-sum game with generator network $\Gen_\theta$, discriminator network $\Dis_\omega$, and loss $f(\theta,\omega)$ with regularization $R_j(\theta,\omega)$ (for either $j=1$ or $j=2$) and any regularization parameter $\mu\in (0,\infty)$ such that  Assumption~\ref{ass:ganassump}  is satisfied for an equilibrium $x^\ast=(\theta^\ast,\omega^\ast)$ of the regularized dynamics. Then,  $x^\ast=(\theta^\ast,\omega^\ast)$ is a differential Stackelberg equilibrium. Furthermore, 
for any $\tau\in(0,\infty)$,
$\spec(-J_{(\tau,\mu)}(x^\ast))\subset \mb{C}_-^\circ$---i.e.,  $-J_{(\tau,\mu)}(x^\ast)$ is Hurwitz stable. 
\label{thm:ganconvergence}
\end{theorem}
\begin{proof}[Proof Sketch.]
To proof that $x^\ast$ is a differential Stackelberg equilibrium  follows analogous arguments to those in the proof of Theorem 4.1 in \citet{mescheder2018training}. Given any positive regularization parameter $\mu$, to prove the stability of $x^\ast$ for any fixed $\tau\in(0,\infty)$,
  we leverage the concept of the quadratic numerical range of a block operator which is a superset of the spectrum of the operator (cf.~Appendix~\ref{app_sec:helperlemmas}).
 The key for both arguments is that Assumption~\ref{ass:ganassump} implies that the Jacobian of the regularized game has a specific structure. Indeed, observe that 
the structural form of $J_{(\tau,\mu)}(x^\ast)$ is
\begin{equation*}
    J_{(\tau,\mu)}(x^\ast)=\bmat{0 & B\\
-\tau B^\top & \tau(C+\mu R)}
\end{equation*}
where $B=D_{12}f(x^\ast)$, $C=-D_2^2f(x^\ast)$ and $R=D_2^2R_j(x^\ast)$. This follows from Assumption~\ref{ass:ganassump}-a., which implies that $\Dis(x;\omega^\ast)=0$ in some neighborhood of $\supp(p_{\mc{D}})$ and hence, $\nabla_x\Dis(x;\omega^\ast)=0$ and $\nabla^2_x\Dis(x;\omega^\ast)=0$ for $x\in \supp(p_{\mc{D}})$. In turn, we have that $D_1^2f(x^\ast)=0$. 

From here, it is straightforward to argue that $B$ is full rank and  $C+\mu R>0$ away from $T_{\theta^\ast}\mc{M}_\Gen$ (by Assumption~\ref{ass:ganassump}-c). Together these observations imply that $x^\ast$ is a differential Stackelberg equilibrium. Then, arguments analogous to those in the proof of Proposition~\ref{lem:zsspecbdd} (via the quadratic numerical range) imply stability. 
\end{proof}

\begin{corollary}
\label{cor:rate}
Under the assumptions of Theorem~\ref{thm:ganconvergence}, consider  any fixed $\mu\in (0,\mu_1]$ and $\tau\in (0,\infty)$, and let $\gamma=\min_{\lambda\in \spec(J_{(\tau,\mu)}(x^\ast))}2\Re(\lambda)/|\lambda|^2$, and ${\lambdam}=\arg\min_{\lambda\in \spec(J_{(\tau,\mu)}(x^\ast))}2\Re(\lambda)/|\lambda|^2$. For any $\alpha\in (0,\gamma)$, $\tau$-{\gda} with learning rate $\gamma_1=\gamma-\alpha$ converges locally asymptotically at a rate of $O((1-\tfrac{\alpha}{4\beta})^{k/2})$ where $\beta=(2\Re({\lambdam})-\alpha|{\lambdam}|^2)^{-1}$.
\end{corollary}
The proof of the above corollary follows a similar line of reasoning as the proof of Corollary~\ref{cor:asymptoticiffstack}.

Theorem A.7 of \citet{mescheder2018training} shows that matrices of the form 
\begin{equation}-J=\bmat{0 & -B\\ B^\top & -C}\label{eq:jform}
\end{equation}
are stable if $B$ is full rank and $C>0$. The following proposition provides necessary conditions on the sizes of the network architectures for the discriminator and generator network for stability. 
\begin{proposition}\label{prop:dimension}
Consider  training a generative adversarial network  via a zero-sum game with generator network $\Gen_\theta$, discriminator network $\Dis_\omega$, and loss $f(\theta,\omega)$ with regularization $R_j(\theta,\omega)$ (for some $j\in \{1,2\}$) such that  Assumption~\ref{ass:ganassump}  is satisfied for an equilibrium $x^\ast=(\theta^\ast,\omega^\ast)$. Independent of the learning rate ratio and the regularization parameter $\mu$, for $x^\ast$ to be stable it is necessary that the dimension of the discriminator network parameter vector is at least half as large as the corresponding generator network parameter vector---i.e., $\m_2\geq \m_1/2$ where $\theta\in \mb{R}^{\m_1}$ and $\omega\in \mb{R}^{\m_2}$. 
\end{proposition}
The intuition for the why this proposition should hold  follows immediately from observing the structure of the Jacobian: for any matrix of the form
\eqref{eq:jform},
at least one eigenvalue will be purely imaginary if $\m_2<\m_1/2$ where $B\in \mb{R}^{\m_1\times \m_2}$ and $C\in \mb{R}^{\m_2\times \m_2}$.
The full proof is contained in Appendix~\ref{app_sec:proof_dimension}.

\section{Experiments}
\label{sec:experiments}
We now present extensive numerical experiments examining gradient descent-ascent with timescale separation. As we explored theoretically so far, the stability of gradient descent-ascent critical points has an intricate relationship with timescale separation. We begin to investigate this behavior empirically by simulating the gradient descent-ascent dynamics for the games from Examples~\ref{example:nonstablestack} and~\ref{example:stablenoneq} and examining how the spectrum of the game Jacobian evolves as a function of the timescale separation. Then, on a polynomial game, we demonstrate how timescale separation warps the vector field of gradient descent-ascent and consequently shapes the region of attraction around critical points in the optimization landscape. There are a number of both qualitative and quantitative theoretical questions that remain open related to characterizing the region of attraction and how it depends parameterically on $\tau$.  

After exploring the optimization landscape, we focus in on  gradient descent-ascent in the Dirac-GAN game and illustrate the interplay between timescale separation, regularization, and rate of convergence. Finally, we train generative adversarial networks on the CIFAR-10 and CelebA datasets with regularization and show timescale separation can significantly improve stability and performance. Moreover, we find that several of the insights we draw from the Dirac-GAN game carry over to this complex setting. Appendix~\ref{app_sec:experiments} contains several more experimental results including a generative adversarial network formulation to learn a covariance matrix and a torus game. The code for our experiments is available at \texttt{\href{https://github.com/fiezt/Finite-Learning-Ratio}{github.com/fiezt/Finite-Learning-Ratio}}.

\subsection{Quadratic Game: Timescale Separation and Stackelberg Stability}
\label{sec:quadgame_exp}
We now revisit the game from Example~\ref{example:nonstablestack} that demonstrated there exists differential Stackelberg equilibrium that are unstable for choices of the timescale separation $\tau$. To be clear, we repeat the game construction and some characteristics of the game. Let us consider the quadratic zero-sum game
defined by the cost
\begin{equation}
f(x_1,x_2)=\frac{1}{2}\bmat{x_{1}\\ x_2}^\top\bmat{
-v & 0 &-v & 0\\ 
0 & \tfrac{1}{2}v & 0 & \tfrac{1}{2}v\\
-v & 0 & -\tfrac{1}{2}v & 0\\
0 & \tfrac{1}{2}v & 0 & -v} \bmat{x_1\\ x_2} \label{eq:quadgamedef} 
\end{equation}
where $x_1, x_2\in \mb{R}^2$ and $v >0$. The unique critical point of the game given by $x^{\ast} = (x_1^{\ast}, x_2^{\ast})=(0,0)$ is a differential Stackelberg equilibrium. The spectrum of the Jacobian evaluated at the equilibrium is given by
\[\spec(J_{\tau}(x^{\ast}))=\Big\{\frac{v(2\tau+1\pm \sqrt{4\tau^2 -8\tau+1})}{4}, \frac{v(\tau-2\pm \sqrt{\tau^2-12\tau +4})}{4}\Big\}.\]
As mentioned in Example~\ref{example:nonstablestack}, it turns out that $\spec(J_\tau(x^{\ast})\subset\mb{C}_+^\circ$ only when $\tau\in (2, \infty)$. We remark that we computed $\tau^{\ast}$ using the theoretical construction from Theorem~\ref{thm:iffstack} and found that it recovered the precise value of $\tau^{\ast}=2$ such that the equilibrium is stable for all $\tau \in (\tau^{\ast}, \infty)$ with respect to the dynamics $\dot{x}=-\Lambda_{\tau} g(x)$. In the experiments that follow, we consistently observe that the construction of $\tau^{\ast}$ from the theory is tight.

 \begin{figure*}[t!]
  \centering
  \subfloat[][]{\includegraphics[width=.3\textwidth]{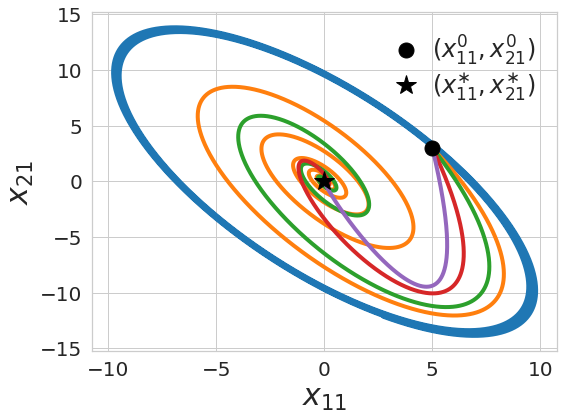}\label{fig:ex_a}}\hfill
  \subfloat[][]{\includegraphics[width=.3\textwidth]{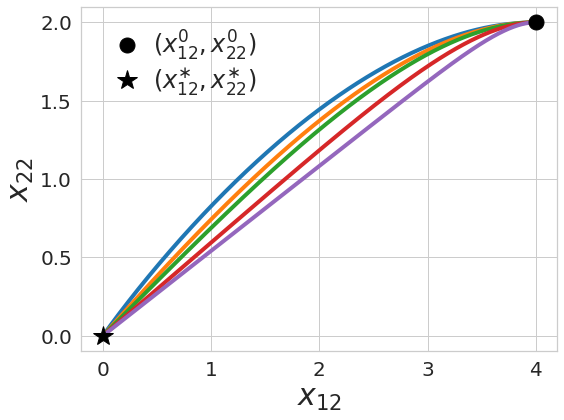}\label{fig:ex_b}}\hfill
  \subfloat[][]{\includegraphics[width=.3\textwidth]{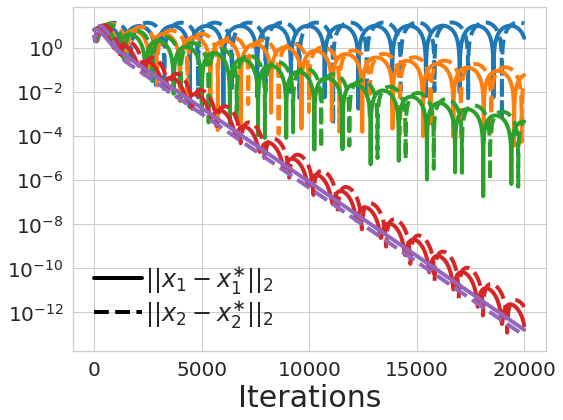}\label{fig:ex_c}}

   \subfloat{\includegraphics[width=.6\textwidth]{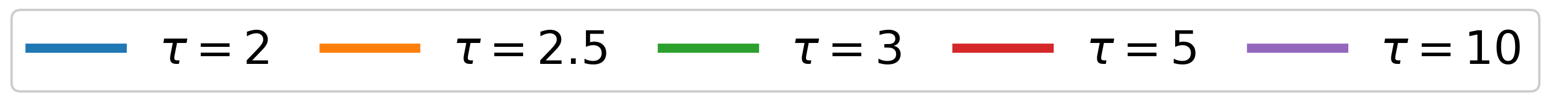}\label{fig:ex_l}}
  
   \subfloat[][]{\includegraphics[width=.33\textwidth]{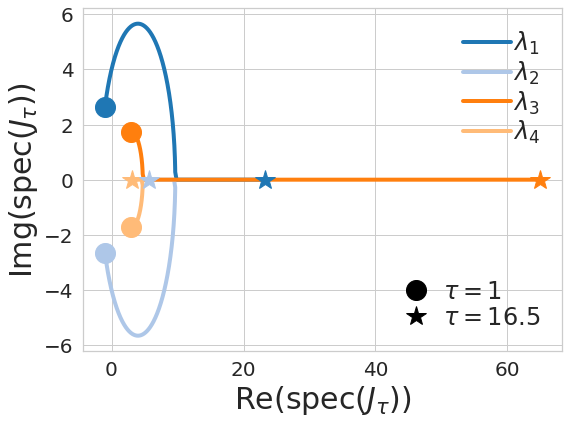}\label{fig:ex_d}}
   \subfloat[][]{\includegraphics[width=.33\textwidth]{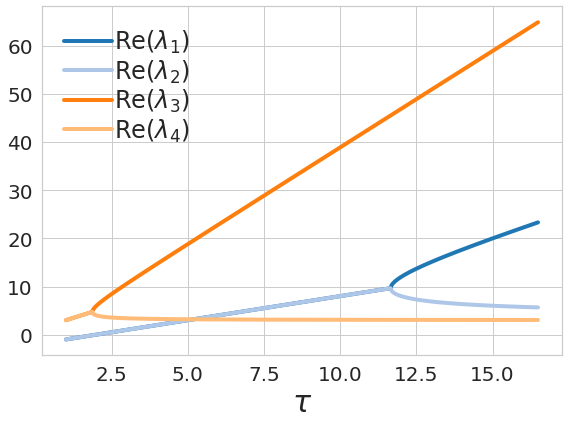}\label{fig:ex_e}}
   \subfloat[][]{\includegraphics[width=.33\textwidth]{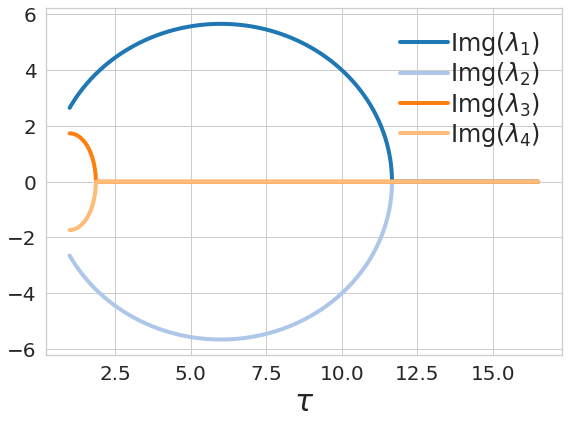}\label{fig:ex_f}}
  \caption{Experimental results for the quadratic game defined in~\eqref{eq:quadgamedef} of Section~\ref{sec:quadgame_exp} and presented in Example~\ref{example:nonstablestack}. 
  Figures~\ref{fig:ex_a} and~\ref{fig:ex_b} show trajectories of the players coordinate pairs $(x_{11}, x_{21})$ and $(x_{21}, x_{22})$ for a range of learning rate ratios, respectively. Figures~\ref{fig:ex_c} shows the distance from the equilibrium along the learning paths. Figures~\ref{fig:ex_d},~\ref{fig:ex_e}, and~\ref{fig:ex_f} show the trajectories of the eigenvalues, the real parts of the eigenvalues, and the imaginary parts of the eigenvalues for the $J_{\tau}(x^{\ast})$ as a function of the $\tau$, respectively.}
  \label{fig:ex}
\end{figure*}

For this experiment, we select $v=4$ and simulate $\tau$-{\gda} from the initial condition $(x_1^0, x_2^0) = (5, 4, 3, 2)$ with $\gamma_1=0.0005$ and  $\tau \in \{2, 2.5, 3, 5, 10\}$.
In Figures~\ref{fig:ex_a} and~\ref{fig:ex_b}, we show the trajectories of the players coordinate pairs $(x_{11}, x_{21})$ and $(x_{21}, x_{22})$, respectively.  We observe that $\tau$-{\gda} cycles around the equilibrium with $\tau=2$ since it is marginally stable with respect to the dynamics. For $\tau \in (2, \infty)$, the equilibrium is stable and $\tau$-{\gda} ends up converging to it at a rate that depends on the choice of $\tau$. We demonstrate how the convergence rate depends on the choice of $\tau$ in Figure~\ref{fig:ex_c} by showing the distance from the equilibrium along the learning path for each of the trajectories. The primary observation is that the cyclic behavior of $\tau$-{\gda} dissipates as $\tau$ grows and as a result the dynamics then rapidly converge to the equilibrium. 

The behavior of the learning dynamics as a function of the timescale separation $\tau$ can be further explained by evaluating the eigenvalues of the game Jacobian at the equilibrium. We show the eigenvalues of the Jacobian at the equilibrium in several forms in Figures~\ref{fig:ex_d},~\ref{fig:ex_e}, and~\ref{fig:ex_f}. Analyzing the spectrum, we are able to verify that for all $\tau \in (2, \infty)$ the equilibrium is indeed stable. Moreover, we see that the imaginary parts of the conjugate pairs of eigenvalues decay after $\tau=1$ and $\tau=6$, and then the eigenvalues of the conjugate pairs eventually become purely real at $\tau=1.87$ and $\tau=11.66$, respectively.
After the eigenvalues of a conjugate pair become purely real, they split so that one of the eigenvalues 
asymptotically converges to an eigenvalue of $\schurtt_1(J(x^\ast))$ by moving back along the real line, while the other eigenvalue tends toward an eigenvalue of $-\tau D_2^2f(x^{\ast})$. This occurrence is exactly what was described in Section~\ref{sec:mainresults} as an immediate implication of Proposition~\ref{prop:simgrad_inf} when the eigenvalues of $\schurtt_1(J(x^\ast))$ and $\tau D_2^2f(x^{\ast})$ are distinct. The convergence rate is in fact limited by the eigenvalues splitting since as $\tau$ grows, the spectrum of the Jacobian is limited by the eigenvalues of the Schur complement which remain constant. A related open question centers on finding the worst case convergence rate as a function of the spectral properties of $\schurtt_1(J(x^\ast))$ and $D_2^2f(x^\ast)$. Finally, the evolution of the eigenvalues as a function of the timescale separation $\tau$ demonstrates that the rotational dynamics in $\tau$-{\gda} vanish as the ratio between the magnitude of the real and imaginary parts of the eigenvalues grows.

 \begin{figure*}[t!]
  \centering
  \subfloat[][]{\includegraphics[width=.4\textwidth]{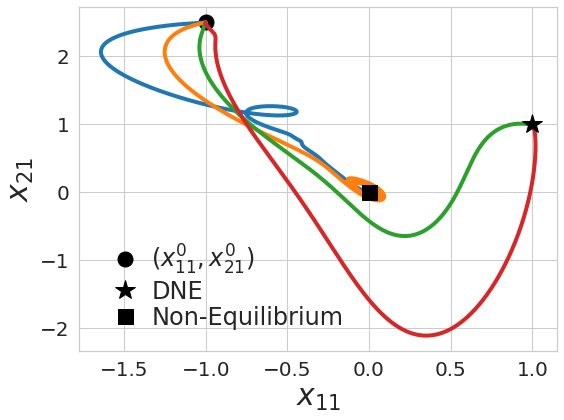}\label{fig:poly_nn_a}} \hspace{7mm}
  \subfloat[][]{\includegraphics[width=.4\textwidth]{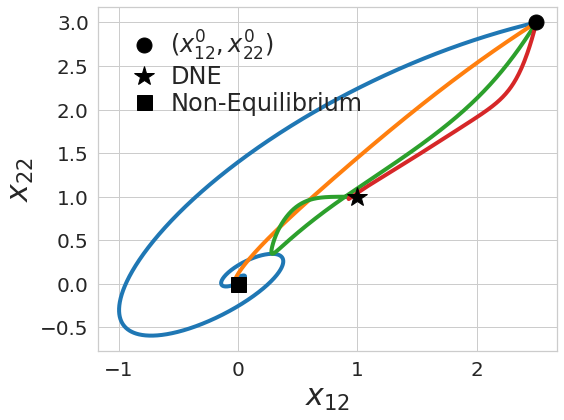}\label{fig:poly_nn_b}}
  
   \subfloat[][]{\includegraphics[width=.5\textwidth]{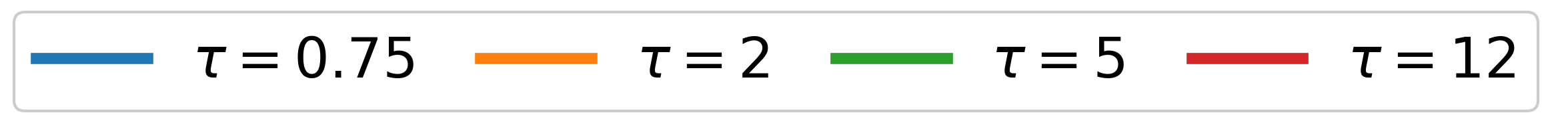}\label{fig:poly_nn_leg}}

   \subfloat{\includegraphics[width=.33\textwidth]{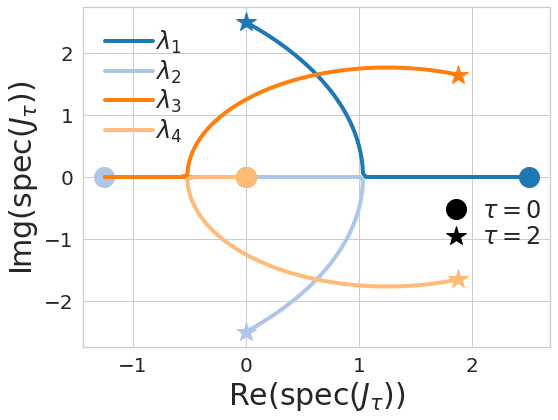}\label{fig:poly_nn_c}}
   \subfloat[][]{\includegraphics[width=.33\textwidth]{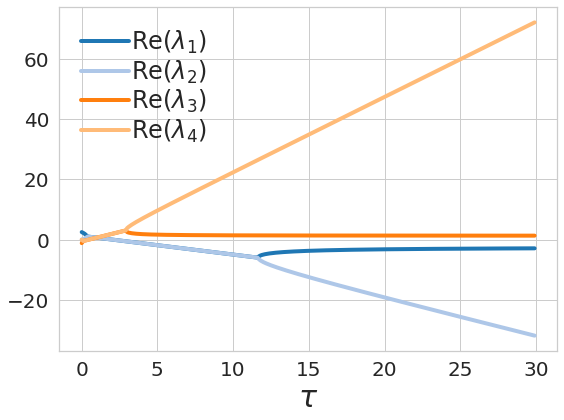}\label{fig:poly_nn_d}} 
   \subfloat[][]{\includegraphics[width=.33\textwidth]{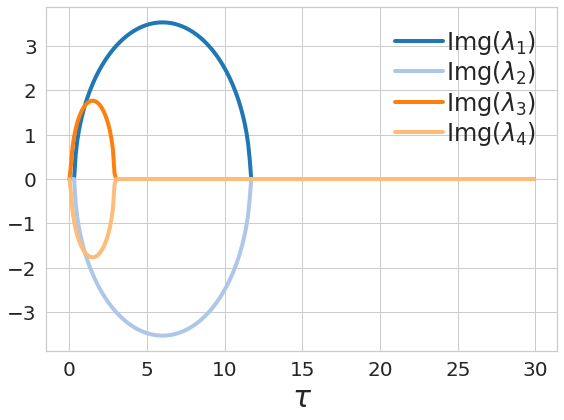}\label{fig:poly_nn_e}}
   
  \caption{Experimental results for the polynomial game defined in~\eqref{eq:nonquadunstable2} of Section~\ref{sec:poly_nn} and presented in Example~\ref{example:stablenoneq}. 
  Figures~\ref{fig:poly_nn_a} and~\ref{fig:poly_nn_b} show trajectories of the players coordinate pairs $(x_{11}, x_{21})$ and $(x_{21}, x_{22})$ for a range of learning rate ratios, respectively.  Figures~\ref{fig:poly_nn_c},~\ref{fig:poly_nn_d}, and~\ref{fig:poly_nn_e} show the trajectories of the eigenvalues, the real parts of the eigenvalues, and the imaginary parts of the eigenvalues for $J_{\tau}(x^{\ast})$ as a function of the $\tau$, respectively where $x^{\ast}$ is the non-equilibrium critical point.}
  \label{fig:poly_nn}
\end{figure*}

\subsection{Polynomial Game: Timescale Separation and Non-Equilibrium Stability}
\label{sec:poly_nn}
We now return to the game from Example~\ref{example:stablenoneq} that showed a non-equilibrium critical point which is stable without timescale separation and becomes unstable for a range of finite learning ratios with multiple equilibria in the vicinity. Again, we repeat the game construction along with some of the key characteristics that were previously presented in Example~\ref{example:stablenoneq}.
Consider a zero-sum game defined by the cost %
\begin{equation}
\begin{split}
f(x_1,x_2)&=\tfrac{5}{4}\left(x_{11}^2+2x_{11}x_{21}+\tfrac{1}{2}x_{21}^2-\tfrac{1}{2}x_{12}^2+2x_{12}x_{22}-x_{22}^2\right)(x_{11}-1)^2  \\
&\textstyle \quad +x_{11}^2\big(\sum_{i=1}^2(x_{1i}-1)^2-(x_{2i}-1)^2\big).
\end{split}
\label{eq:nonquadunstable2}
\end{equation}
This game has critical points at $(0,0,0,0)$, $(1,1,1,1)$, and $(-4.73, 0.28, -92.47, 0.53)$. Among the critical points, only $(1,1,1,1)$ and $(-4.73, 0.28, -92.47, 0.53)$ are game-theoretically meaningful equilibrium. In fact, they are each differential Nash equilibrium and are locally stable for any choice of $\tau \in (0, \infty)$ as a result of Proposition~\ref{lem:zsspecbdd}. On the other hand, the critical point $x^{\ast} = (0,0,0,0)$ is neither a differential Nash equilibrium nor a differential Stackelberg equilibrium. However, $x^{\ast}$ is stable for $\tau\in (0,2)$ and it is marginally stable for $\tau=2$. In general, convergence to the non-equilibrium critical point $x^{\ast}$ in the presence of multiple game-theoretically meaningful equilibrium would be viewed as undesirable. In fact, this is precisely the type of critical point that sophisticated schemes for converging to only differential Nash equilibria or only differential Stackelberg equilibria seek to avoid~\cite{adolphs2019local, mazumdar2019finding, wang2019solving, fiez:2020icml}. We show in this example that the simple inclusion of timescale separation in gradient descent-ascent is sufficient to avoid $x^{\ast}$ and instead converge to a differential Nash equilibrium.

Indeed, for all $\tau\in (2, \infty)$ the non-equilibrium critical point $x^{\ast}$ is unstable with respect to $\dot{x}=-\Lambda_{\tau} g(x)$. We simulate $\tau$-{\gda} from the initial condition $(x_{1}^0, x_{2}^0) = (-1.5, 2.5,$ $2.5, 3)$ with $\gamma_1 = 0.0005$ and $\tau \in \{0.75, 2, 5, 12\}$, where we use the superscript to denote the time index so as not to be confused with the multiple indexes for player choice variables. In Figures~\ref{fig:poly_nn_a} and~\ref{fig:poly_nn_b}, we show the trajectories of the players coordinate pairs $(x_{11}, x_{21})$ and $(x_{21}, x_{22})$, respectively. We observe that $\tau$-{\gda} converges to the non-equilibrium critical point $x^{\ast}$ with $\tau=0.75$ as expected and the dynamics move near it and then cycle around it with $\tau=2$ since the critical point becomes marginally stable. However, for $\tau=5$ and $\tau=12$, $\tau$-{\gda} avoids the non-equilibrium critical point since it becomes unstable and instead the dynamics converge to the nearby differential Nash equilibrium. 
We show the eigenvalues of the Jacobian at the non-equilibrium critical point $x^{\ast}=(0,0,0,0)$ in several forms in Figures~\ref{fig:poly_nn_c}--\ref{fig:poly_nn_e}. Again, we observe that the eigenvalues quickly become purely real as $\tau$ grows and then they split, and asymptotically converge toward the eigenvalues of $\schurtt_1(J(x^\ast))$ and $-\tau D_2^2f(x^{\ast})$. 
Together, this example demonstrates that often there is a reasonable finite learning rate ratio such that non-meaningful critical points become unstable for $\tau$-{\gda}.

 \begin{figure*}[t!]
  \centering
   \subfloat[][]{\includegraphics[width=.4\textwidth]{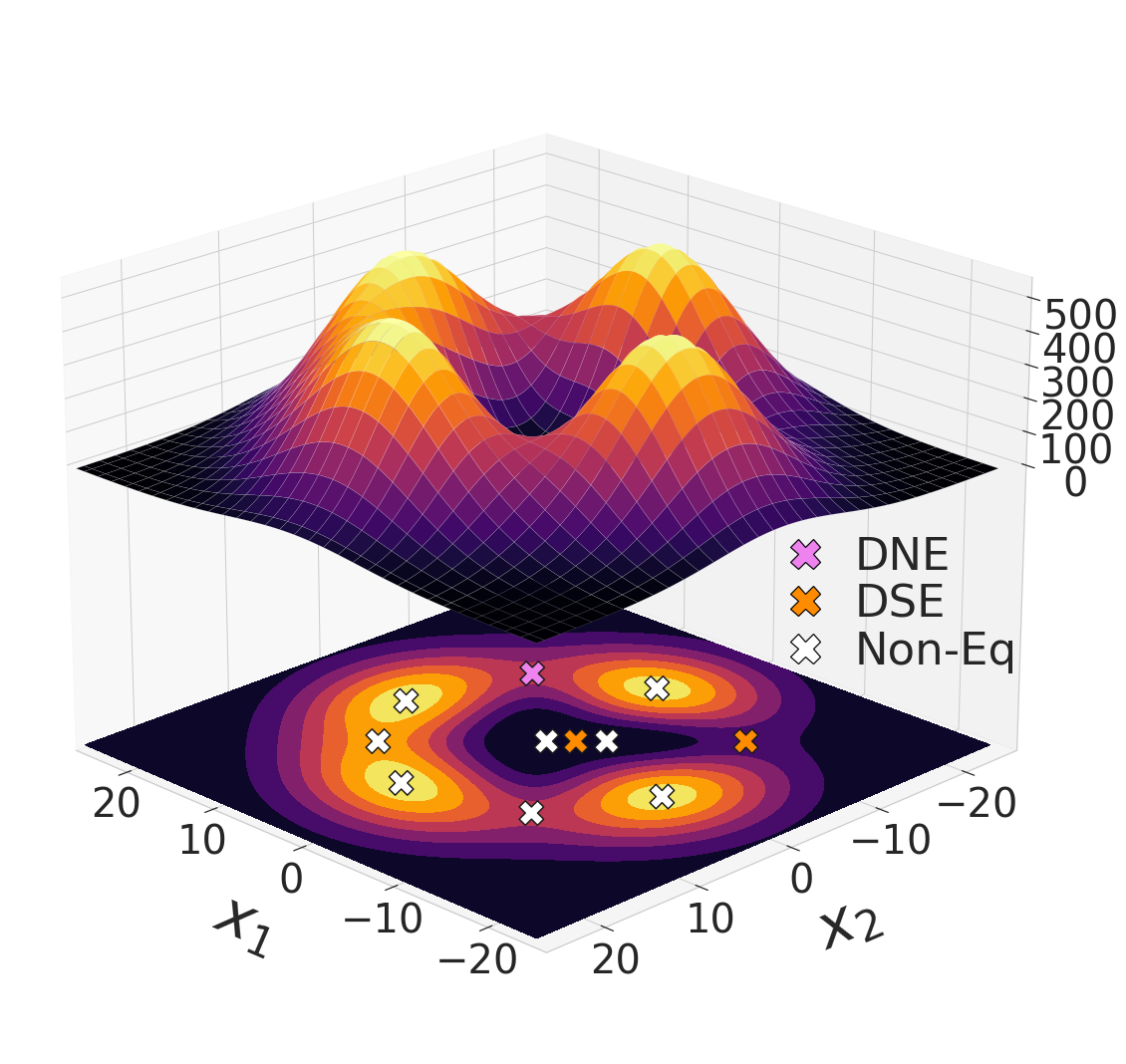}\label{fig:pol_a}}\hspace{7mm}
    \subfloat{\includegraphics[width=.4\textwidth]{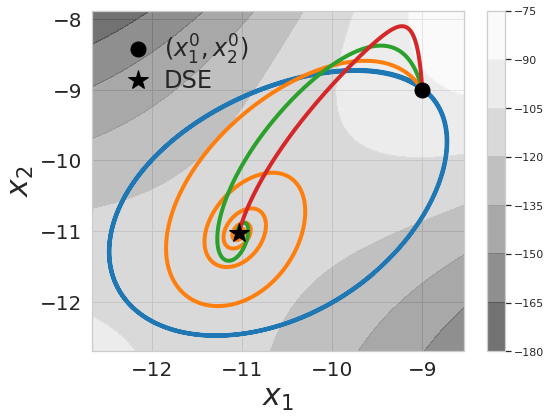}\label{fig:pol_b}} 
    
    \setcounter{subfigure}{1}
    \hspace{80mm}\subfloat[][]{\includegraphics[width=.4\textwidth]{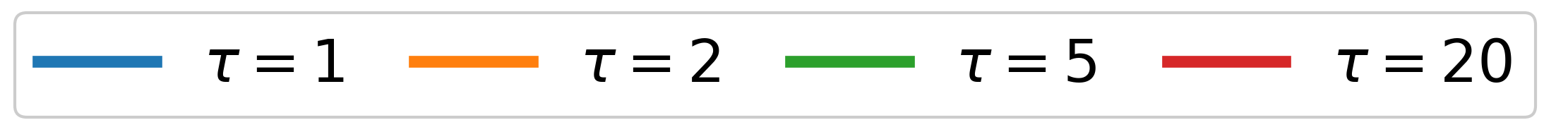}\label{fig:pol_leg}}

\subfloat[][]{\includegraphics[width=.4\textwidth]{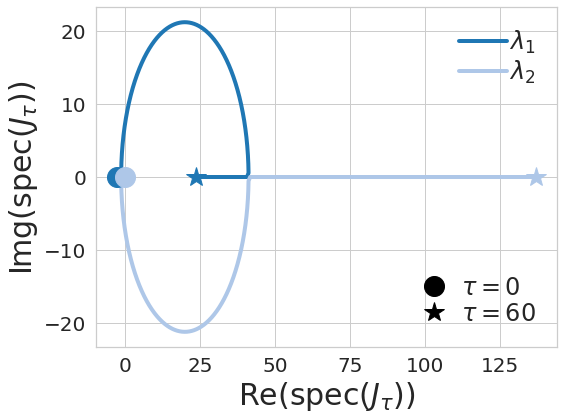}\label{fig:pol_e}}\hspace{7mm}
  \subfloat[][]{\includegraphics[width=.4\textwidth]{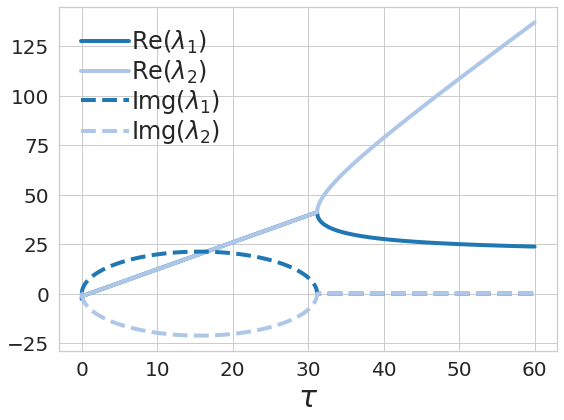}\label{fig:pol_f}} \hfill
  \caption{Experimental results for the polynomial game defined in~\eqref{eq:polygame} of Section~\ref{sec:polywarpgame}. 
  Figures~\ref{fig:pol_a} provides a 3d view of the cost function $-f(x_1, x_2)$ along with the cost contours and critical point locations. Figure~\ref{fig:pol_b} shows trajectories of $\tau$-{\gda} for a range of learning rate ratios given an initialization around the differential Stackelberg equilibrium $(x_1^{\ast}, x_2^{\ast}) = (-11.03, -11.03)$.  Figures~\ref{fig:pol_e} and~\ref{fig:pol_f} show the evolution of the eigenvalues from $J_{\tau}(x^{\ast})$ as a function of $\tau$ where $x^{\ast}$ is the differential Stackelberg equilibrium $(x_1^{\ast}, x_2^{\ast}) = (-11.03, -11.03)$.}
  \label{fig:poly}
\end{figure*} 

\begin{figure*}[t!]
  \centering
  \subfloat[][]{\includegraphics[width=\textwidth]{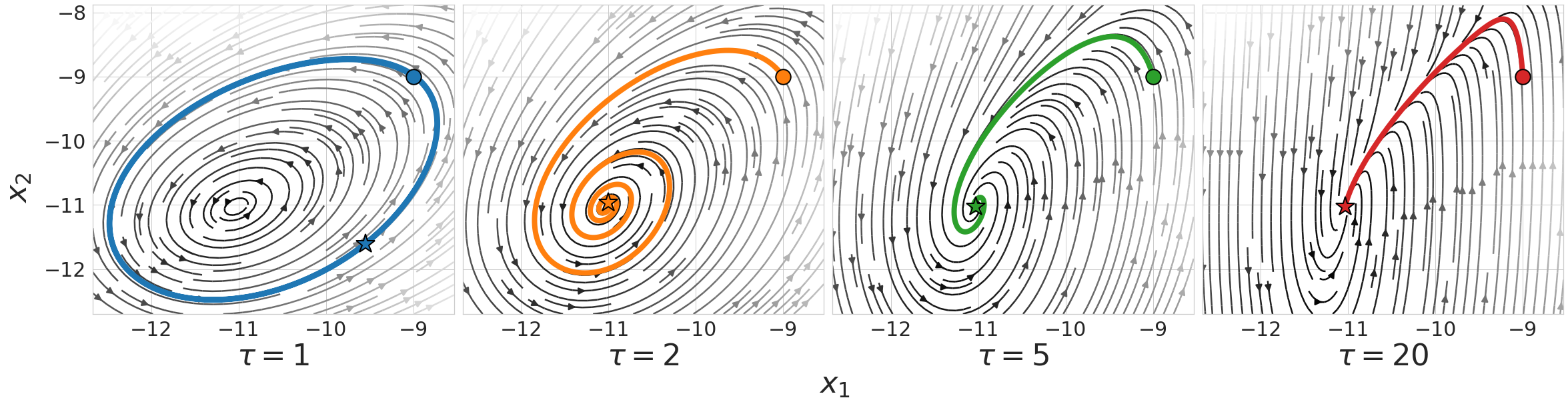}\label{fig:pol_c}}
  
  \subfloat{\includegraphics[width=\textwidth]{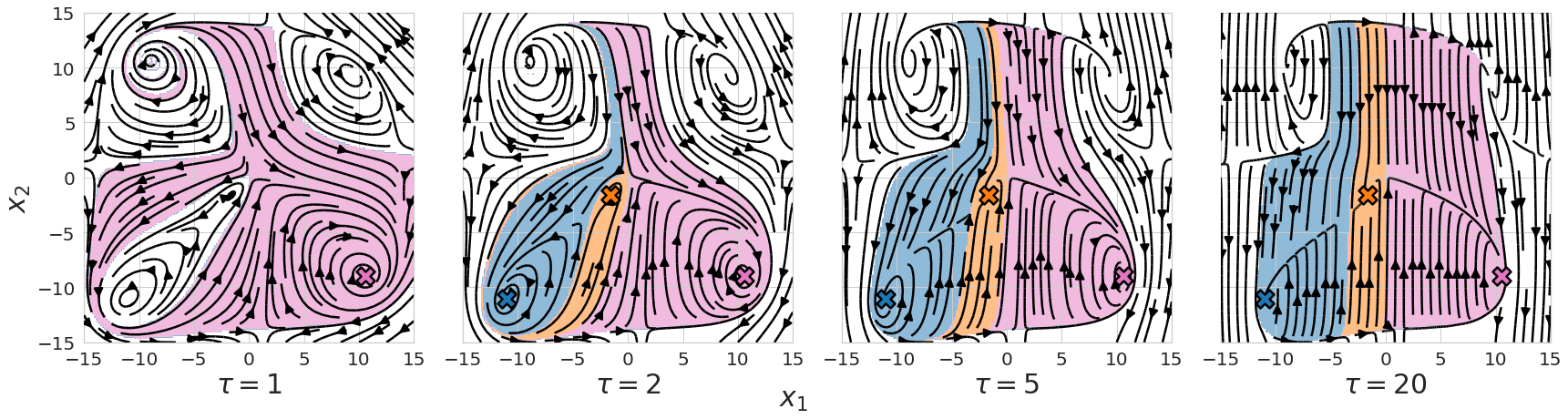}\label{fig:pol_d}}
  
     \setcounter{subfigure}{1}
\subfloat[][]{\includegraphics[width=.8\textwidth]{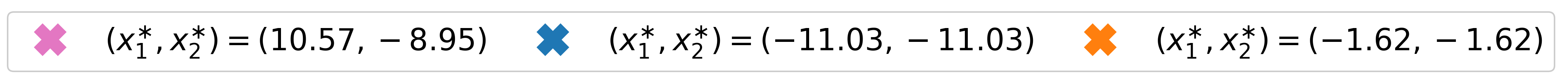}\label{fig:pol_roa_leg}}
  \caption{Experimental results for the polynomial game defined in~\eqref{eq:polygame} of Section~\ref{sec:polywarpgame}. In Figure~\ref{fig:pol_c}, we overlay the trajectories from Figure~\ref{fig:pol_b} produced by $\tau$-{\gda} onto the vector field generated by the choice of timescale separation selection $\tau$. The shading of the vector field is dictated by its magnitude so that lighter shading corresponds to a higher magnitude and darker shading corresponds to a lower magnitude. Figure~\ref{fig:pol_roa_leg} demonstrates the effect of timescale separation on the region of attractions around critical points by coloring points in the strategy space according to the equilibrium $\tau$-{\gda} converges. We remark that areas without coloring indicate where $\tau$-{\gda} did not converge in the time horizon.}
  \label{fig:poly2}
\end{figure*} 

\subsection{Polynomial Game: Vector Field Warping and Region of Attraction}
\label{sec:polywarpgame}
Consider a zero-sum game defined by the cost
\begin{equation}
f(x_1,x_2)=-e^{-\left(0.01x_1^2+0.01x_2^2\right)}\left((0.3x_1+x_2^2)^2 + (0.3x_2 + x_1^2)^2\right).
\label{eq:polygame}
\end{equation}
The cost structure of this game is visualized in Figure~\ref{fig:pol_a}, where we present a three dimensional view of $-f(x_1, x_2)$ along with the cost contours and the locations of critical points. This game has eleven critical points including one differential Nash equilibrium and two differential Stackelberg equilibria that are not a differential Nash equilibrium. The critical points that are neither a differential Nash equilibrium nor a differential Stackelberg equilibrium are unstable for any choice of timescale separation $\tau$. The differential Nash equilibrium is at $(x_1, x_2)=(10.57, -8.95)$ and it is stable for all  $\tau \in (0, \infty)$ by Proposition~\ref{lem:zsspecbdd}. The differential Stackelberg equilibria are at $(x_1, x_2)=(-1.625, -1.625)$ and $(x_1^{\ast}, x_2^{\ast})=(-11.03, -11.03)$; each is stable for all $\tau \in (1, \infty)$.
We computed $\tau^{\ast}$ for the pair of differential Stackelberg equilibrium using the theoretical construction from Theorem~\ref{thm:iffstack} and observed that it properly recovered $\tau^{\ast}=1$ for each equilibrium as the timescale separation such that the continuous time system is stable for all $\tau \in (\tau^{\ast}, \infty)$. Finally, we note that while the set of equilibrium follow a linear translation, this game is generic and the equilibria are in fact isolated. 

In Figure~\ref{fig:pol_b}, we show the trajectories of $\tau$-{\gda} with $\gamma_1=0.0001$ and $\tau \in \{1, 2, 5, 20\}$ given the initialization $(x_{1}^0, x_{2}^0) = (-9, -9)$ near the differential Stackelberg equilibrium at $(x_1^{\ast}, x_2^{\ast})=(-11.03, -11.03)$. Moreover, in Figure~\ref{fig:pol_c}, we overlay the trajectories on the vector field generated by the respective timescale separation parameters. As expected, the choice of $\tau=1$ results in a trajectory that cycles around the equilibrium in a closed curve since it is marginally stable and $J_{\tau}(x^{\ast})$ has purely imaginary eigenvalues. Notably, as $\tau$ grows, the cyclic behavior dissipates as the timescale separation reshapes the vector field until the trajectory moves near directly to the zero derivative line of the maximizing player and then follows a path along that line toward the equilibrium and converges rapidly. The eigenvalues of $J_{\tau}(x^{\ast})$ as a function of $\tau$ are presented in Figures~\ref{fig:pol_e} and~\ref{fig:pol_f}. 
As was the case for the previous experiments, we observe that after the eigenvalues become purely real as $\tau$ grows, they then split and asymptotically converge toward the eigenvalues of $\schurtt_1(J(x^\ast))$ and $-\tau D_2^2f(x^{\ast})$. It is worth noting that much of the rotational behavior in the dynamics and vector field disappears as a result of timescale separation well before the eigenvalues become purely real; this seems to occur after the timescale separation is such that the magnitude of the real part of the eigenvalues is greater than that of the imaginary part.

Finally, in Figure~\ref{fig:pol_roa_leg}, we demonstrate how the choice of timescale separation $\tau$ not only warps the vector field but also shapes the regions of attraction around critical points. The vector field is again shown for each $\tau \in \{1, 2, 5, 20\}$, but now zoomed out to include each of the equilibria. The colors overlayed on the vector field indicate the equilibria that the dynamics converge to given an initialization at that position. Positions in the strategy space without color did not converge to an equilibrium in the fixed horizon of 75000 iterations with $\gamma_1=0.001$. This is explained by the fact that the dynamics are not guaranteed to be globally convergent and may get stuck in limit cycles or may simply move slowly for a long time in flat regions of the optimization landscape. We produced this experiment by running $\tau$-{\gda} for a dense set of initial conditions chosen uniformly over the space of interest. It is clear from the experiment that the choice of timescale separation determines not only the stability of equilibria, but also has a fundamental impact on the equilibria the dynamics converge to from a given initial condition as a result of the warping of the vector field. As a concrete example, given an initialization of $(x_1, x_2)=(-10, -2)$, the dynamics with $\tau=1$ converge to the differential Nash equilibria at $(x_1, x_2)=(10.57, -8.95)$. However, for any $\tau > 1$, the dynamics instead converge to the differential Stackelberg equilibrium at $(x_1, x_2)=(-11.03, -11.03)$ that is significantly closer to the initial condition. 
This example motivates future work on methods for obtaining accurate estimates of the regions of attraction around critical points and techniques to design $\tau$ in order to explicitly shape the region of attraction around an equilibrium of interest. We refer to the end of Section~\ref{sec:gda_determ} for further discussion on potentially relevant analysis methods in this direction.

\begin{figure*}[t!]
  \centering
  \subfloat[][]{\includegraphics[width=.4\textwidth]{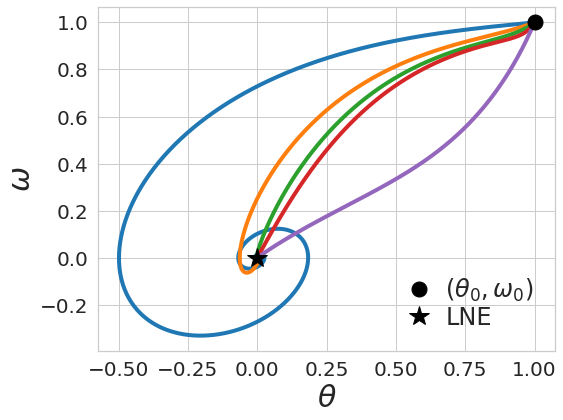}\label{fig:dirac1_a}} \hspace{7mm}
  \subfloat[][]{\includegraphics[width=.4\textwidth]{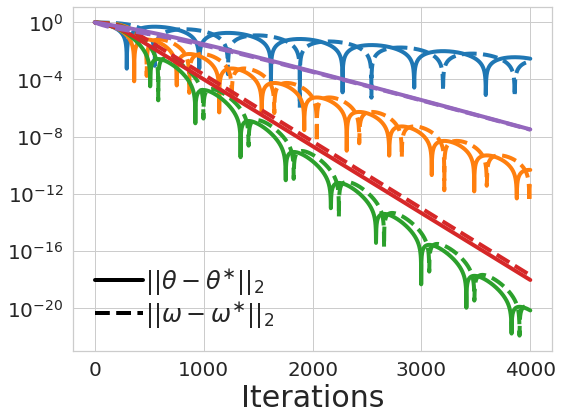}\label{fig:dirac1_b}}
  
  \subfloat{\includegraphics[width=.9\textwidth]{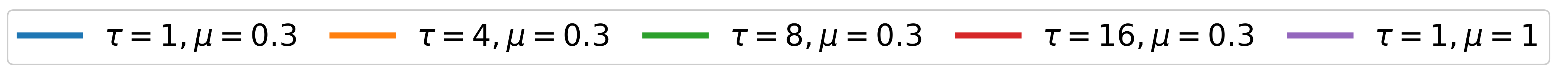}\label{fig:dirac_leg}}

  \subfloat[][]{\includegraphics[width=\textwidth]{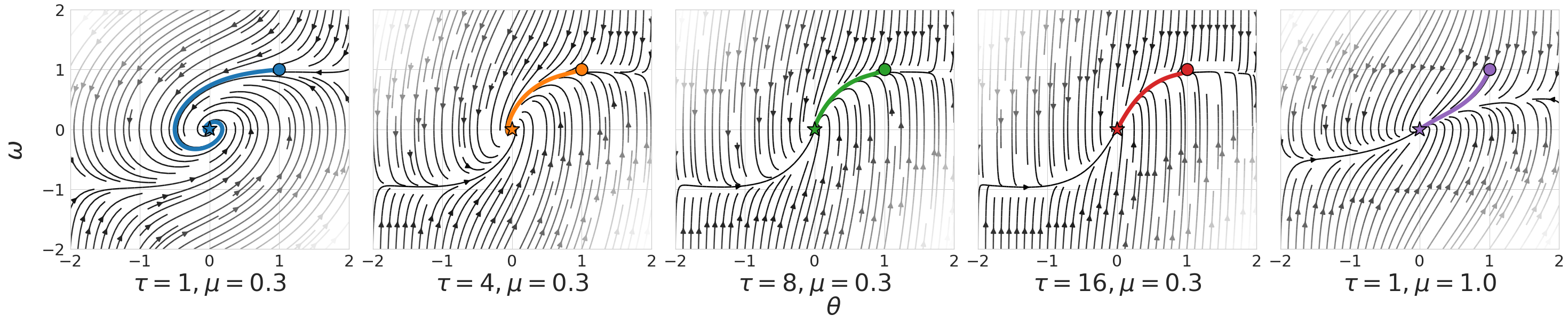}\label{fig:dirac1_e}}
  
  \subfloat[][$\mu=0.3$]{\includegraphics[width=.38\textwidth]{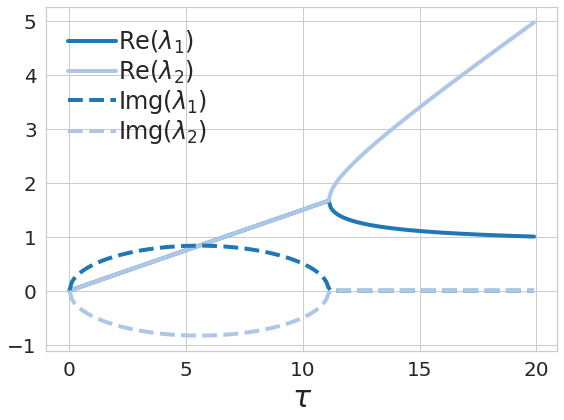}\label{fig:dirac1_c}} \hspace{7mm}
  \subfloat[][$\mu=1$]{\includegraphics[width=.38\textwidth]{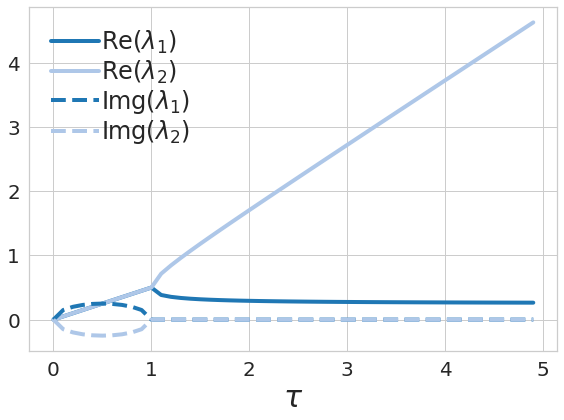}\label{fig:dirac1_d}}
  \caption{Experimental results for the Dirac-GAN game defined in~\eqref{eq:diracreg} of Section~\ref{sec:exp_dirac}. 
  Figure~\ref{fig:dirac1_a} shows trajectories of $\tau$-{\gda} for $\tau \in \{1, 4, 8, 16\}$ with regularization $\mu=0.3$ and $\tau=1$ with regularization $\mu=1$. Figure~\ref{fig:dirac1_b} shows the distance from the equilibrium along the learning paths.  Figure~\ref{fig:dirac1_d} shows the trajectories of $\tau$-{\gda} overlayed on the vector field generated by the respective timescale separation and regularization parameters. The shading of the vector field is dictated by its magnitude so that lighter shading corresponds to a higher magnitude and darker shading corresponds to a lower magnitude. Figures~\ref{fig:dirac1_c} and~\ref{fig:dirac1_d} show the trajectories of the eigenvalues for $J_{\tau}(\theta^{\ast}, \omega^{\ast})$ as a function of $\tau$ with regularization set to $\mu=0.3$ and $\mu=1$, respectively where $(\theta^{\ast}, \omega^{\ast})$ is the unique critical point of the game.}
  \label{fig:dirac1}
\end{figure*}

\subsection{Dirac-GAN: Regularization, Timescale Separation, and Convergence Rate}
\label{sec:exp_dirac}
In Section~\ref{sec:gans}, we studied gradient descent-ascent with regularization in generative adversarial networks and showed that the general theory we provide can be extended to such a formulation.  Recall that the training objective for generative adversarial networks is of the form
\begin{equation}
f(\theta,\omega)=\mb{E}_{p(z)}[\ell(\Dis(\Gen(z;\theta);\omega))]+\mb{E}_{p_{\mc{D}}(x)}[\ell(-\Dis(x;\omega))]
\label{eq:ganobjective}
\end{equation}
where $\Dis_\omega(x)$ and $\Gen_\theta(z)$ are the discriminator and generator networks respectively, $p_{\mc{D}}(x)$ is the data distribution while $p(z)$ is the latent distribution, and $\ell\in C^2(\mb{R})$ is some real-value function such that $\ell'(0)\neq 0$ and $\ell''(0)<0$. The goal of the generator is to minimize~\eqref{eq:ganobjective} while the discriminator seeks to maximize~\eqref{eq:ganobjective}. As a motivating example, we mentioned the Dirac-GAN proposed by~\citet{mescheder2018training}, which constitutes an extremely simple, yet compelling generative adversarial network. Formally described in Definition~\ref{def:diracgan}, the Dirac-GAN consists of a univariate
generator distribution $p_\theta = \delta_\theta$ and a linear discriminator
$\Dis(x; \omega) = \omega x$, where the real data distribution $p_{\mD}$ is given by a
Dirac-distribution concentrated at zero.
The resulting zero-sum game is defined by the cost
\[f(\theta,\omega)=\ell(\theta\omega)+\ell(0).\]
The unique critical point of gradient descent-ascent is a local Nash equilibrium given by $(\theta^\ast,\omega^\ast)=(0,0)$. However, the structure of the game is such that 
\begin{equation*}
J_{\tau}(\theta^{\ast}, \omega^{\ast}) = \begin{bmatrix} 0 & \ell'(0) \\ -\tau \ell'(0) & 0 \end{bmatrix}.
\end{equation*}
Consequently, $\spec(J_{\tau}(\theta^{\ast}, \omega^{\ast})) = \{\pm i \sqrt{\tau}\ell'(0)\}$ so that $\spec(J_{\tau}(\theta^{\ast}, \omega^{\ast})) \not \subset  \mb{C}_{+}^{\circ}$ and regardless of the choice of timescale separation, $\tau$-{\gda} oscillates and fails to converge to the equilibrium. This behavior is expected since the equilibrium is not hyperbolic and corresponds to neither a differential Nash equilibrium nor a differential Stackelberg equilibrium since $D_1^2f(\theta^{\ast}, \omega^{\ast})=0$ and $-D_2^2f(\theta^{\ast}, \omega^{\ast})=0$, but it is undesirable nonetheless since $(\theta^\ast,\omega^\ast)$ is the unique critical point and a local Nash equilibrium. 

\citet{mescheder2018training} proposed to remedy the degeneracy issues of generative adversarial networks by using the following gradient penalties to regularize the discriminator with $\mu>0$:
\[R_1(\theta,\omega)=\frac{\mu}{2}\mb{E}_{p_{\mc{D}}(x)}[\|\nabla_x \Dis(x;\omega)\|^2]\quad \text{and}\quad R_2(\theta,\omega)=\frac{\mu}{2}\mb{E}_{p_\theta(x)}[\|\nabla_x\Dis(x;\omega)\|^2].\]
For the Dirac-GAN, \[R_1(\theta,\omega)= R_2(\theta,\omega) = \frac{\mu}{2}\omega^2.\]
The zero-sum game corresponding to the Dirac-GAN with regularization can be defined by the cost
\begin{equation}
f(\theta,\omega)=\ell(\theta\omega)+\ell(0) - \frac{\mu}{2}\omega^2.
\label{eq:diracreg}
\end{equation}
The unique critical point of the game remains at $(\theta^\ast,\omega^\ast)=(0,0)$, but we now see that
\begin{equation}
J_{\tau}(\theta^{\ast}, \omega^{\ast}) = \begin{bmatrix} 0 & \ell'(0) \\ -\tau \ell'(0) & \tau \mu  \end{bmatrix}
\label{eq:diracjaocb}
\end{equation}
and $\spec(J_{\tau}(\theta^{\ast}, \omega^{\ast})) = \{(\tau\mu \pm \sqrt{\tau^2\mu^2-4\tau (\ell'(0))^2})/2\}$. Observe that for all $\tau \in (0, \infty)$ and $\mu \in (0, \infty)$, we get that $\spec(J_{\tau}(\theta^{\ast}, \omega^{\ast})) \subset \mb{C}_{+}^{\circ}$. This implies that for all timescale separation parameters $\tau>0$ and all regularization parameters $\mu>0$, the local Nash equilibrium of the unregularized game is stable with respect to the dynamics $\dot{x}=-\Lambda_{\tau} g(x)$. As a result, for a suitably chosen learning rate $\gamma_1$, the discrete time update $\tau$-{\gda} converges to the equilibrium. It is worth pointing out that the critical point $(\theta^\ast,\omega^\ast)=(0,0)$ corresponds to a differential Stackelberg equilibrium of the regularized game since $-D_2^2f(\theta^\ast,\omega^\ast) =  \mu >0$ and $\schurtt_1(J(\theta^\ast, \omega^{\ast}))= (\ell'(0))^2/\mu >0$.

We now present experiments with $\tau$-{\gda} for the regularized Dirac-GAN game defined in~\eqref{eq:diracreg} focused on exploring the interplay between timescale separation, regularization, and convergence rate since the equilibrium is always stable for a positive regularization parameter.
We let $\ell(t)=-\log(1+\exp(-t))$, which corresponds to the choice made in the original generative adversarial network formulation proposed by~\citet{goodfellow2014generative}. Figure~\ref{fig:dirac1_a} shows trajectories of $\tau$-{\gda} from the initial condition $(\theta_0, \omega_0) = (1, 1)$ with learning rate $\gamma_1 = 0.01$ for $\tau \in \{1, 4, 8, 16\}$ with regularization $\mu=0.3$ and $\tau=1$ with $\mu=1$. Moreover, Figure~\ref{fig:dirac1_d} shows the trajectories of $\tau$-{\gda} overlayed on the vector field generated by the respective timescale separation and regularization parameters and Figure~\ref{fig:dirac1_b} shows the distance from the equilibrium along the learning paths. The choice of $\mu=0.3$ is arbitrary to a degree, but $\mu^{\ast}=1$ is chosen since it corresponds to the critical regularization parameter such that $\spec(J(\theta^{\ast}, \omega^{\ast})) \subset \mb{R}_{+}$ for all $\mu>\mu^{\ast}$, meaning that the Jacobian without timescale separation has purely real eigenvalues.
 Finally, Figures~\ref{fig:dirac1_c} and~\ref{fig:dirac1_d} show the trajectories of the eigenvalues for $J_{\tau}(\theta^{\ast}, \omega^{\ast})$ as a function of $\tau$ with regularization set to $\mu=0.3$ and $\mu=1$, respectively where $(\theta^{\ast}, \omega^{\ast})$ is the unique critical point of the game. 

From Figures~\ref{fig:dirac1_a} and~\ref{fig:dirac1_e}, we observe that the impact of timescale separation with regularization $\mu=0.3$ is that the trajectory is not as oscillatory since it moves faster to the zero line of $-D_2f(\theta, \omega)$ and then follows along that line until reaching the equilibrium. We further see from Figure~\ref{fig:dirac1_b} that with regularization $\mu=0.3$, $\tau$-{\gda} with $\tau=8$ converges faster to the equilibrium than $\tau$-{\gda} with $\tau=16$, despite the fact that the former exhibits some cyclic behavior in the dynamics while the latter does not. The eigenvalues of the Jacobian with regularization $\mu=0.3$ presented in Figure~\ref{fig:dirac1_c} explains this behavior since the imaginary parts are non-zero with $\tau=8$ and zero with $\tau=16$, but the eigenvalue with the minimum real part is greater at $\tau=8$ than at $\tau=16$. This example highlights that a degree of oscillatory behavior in the dynamics is not always harmful for convergence and it can even speed up the rate of convergence. Building off of this, for regularization $\mu=1$ and timescale separation $\tau=1$, Figures~\ref{fig:dirac1_a} and~\ref{fig:dirac1_b} show that even though $\tau$-{\gda} follows a direct path toward the equilibrium and does not cycle since the eigenvalues of the Jacobian are purely real, the trajectory converges slowly to the equilibrium. While not presented, we ran experiments with $\tau \in \{2, 4, 8, 16\}$ with $\mu=1$ as well and timescale separation only made the convergence rate worse. The eigenvalues of the Jacobian with each regularization parameter presented in Figures~\ref{fig:dirac1_c} and~\ref{fig:dirac1_d} are able to explain this phenomenon. Indeed, for each regularization parameter, the eigenvalues split after becoming purely real and then converge toward the eigenvalues of $\schurtt_1(J(\theta^{\ast}, \omega^{\ast}))$ and $-\tau D_2^2f(\theta^{\ast}, \omega^{\ast})$. Since $\schurtt_1(J(\theta^{\ast}, \omega^{\ast}))\propto 1/\mu$ and $-\tau D_2^2f(\theta^{\ast}, \omega^{\ast}) \propto \tau \mu$, there is a trade-off between the choice of regularization $\mu$ and the timescale separation $\tau$ on the conditioning of the Jacobian matrix. As shown in Figures~\ref{fig:dirac1_c} and~\ref{fig:dirac1_d}, the minimum real part of the eigenvalues with $\mu=0.3$ is significantly larger than with $\mu=1$ after sufficient timescale separation, which makes the convergence rate faster. Together, this example demonstrates that there may often be a delicate relationship between timescale separation, regularization, and convergence rate, where after a certain threshold each parameter choice may inhibit the rate of convergence.

In Appendix~\ref{app_sec:diracgan}, we provide simulation results on the Dirac-GAN game using the non-saturating generative adversarial network objective proposed by~\citet{goodfellow2014generative}. In this formulation, the game is defined by the costs $(f_1, f_2)=(-\ell(-\omega \theta)+\ell(0), -(\ell(\omega \theta) + \ell(0)))$. While the non-saturating objective was motivated by global considerations (avoiding vanishing gradients) rather than local considerations, it turns out that it is locally equivalent in terms of the game Jacobian as the standard formulation for the Dirac-GAN. As a result, the stability characteristics are identical and we draw equivalent conclusions from the experiment regarding the behavior of gradient descent-ascent in the game. Finally, we note that in Appendix~\ref{app_sec:covariance} we explore another simple generative adversarial network formulation using the Wasserstein cost function with a linear generator and quadratic discriminator (each of arbitrary dimension) for the problem of learning a covariance matrix. In that example, we draw analogous conclusions about the interplay between timescale separation, regularization, and the rate of convergence.

\begin{figure*}[t!]
  \centering
  \subfloat[][$\mu=10$]{\includegraphics[width=\textwidth]{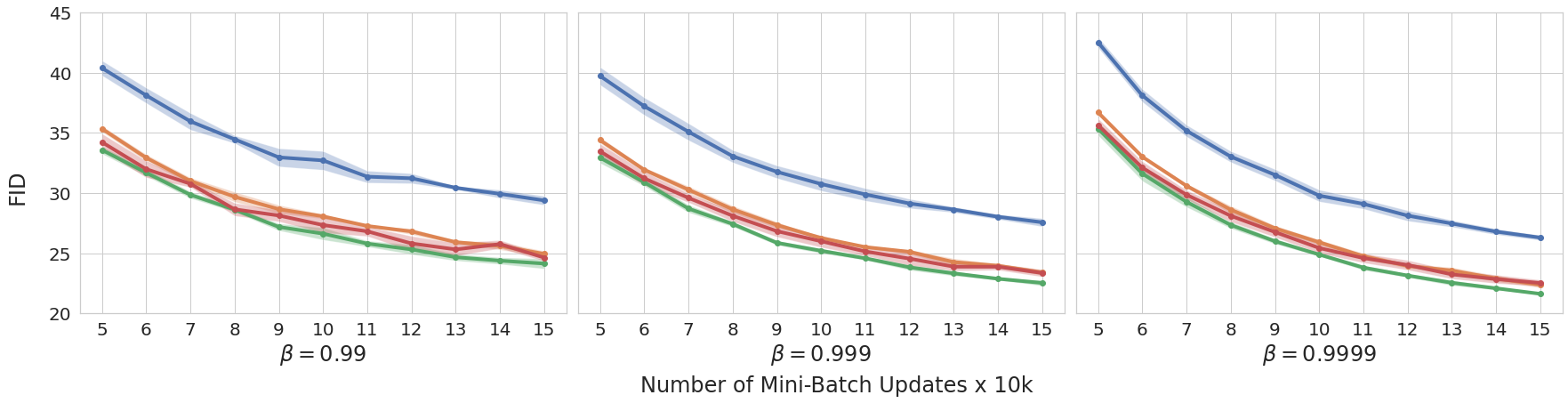}\label{fig:cifar_fida}}
  
   \subfloat[][$\mu=1$]{\includegraphics[width=\textwidth]{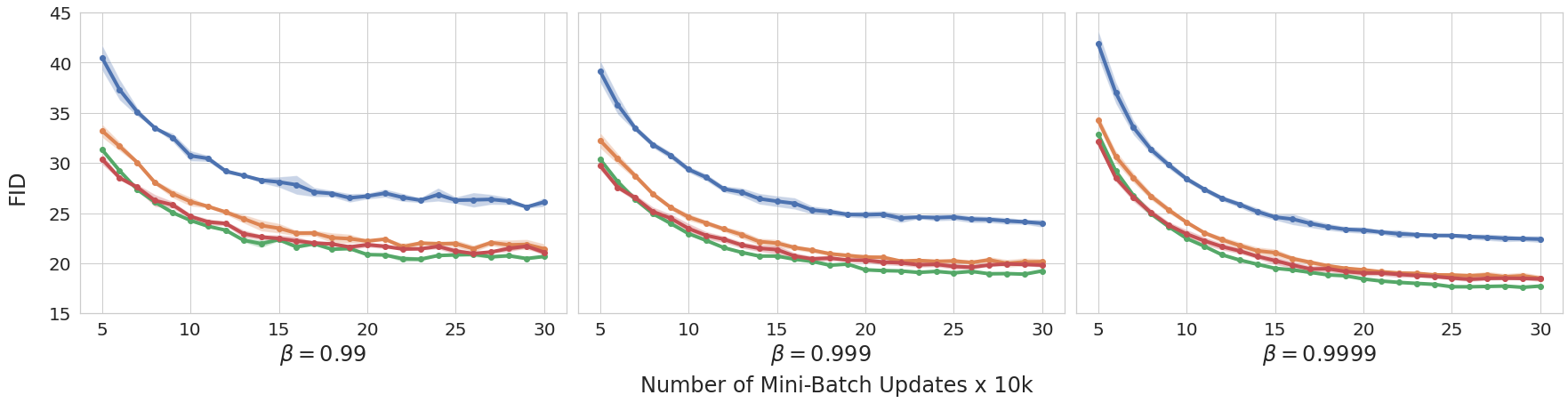}\label{fig:cifar_fidb}}

   \subfloat{\includegraphics[width=.45\textwidth]{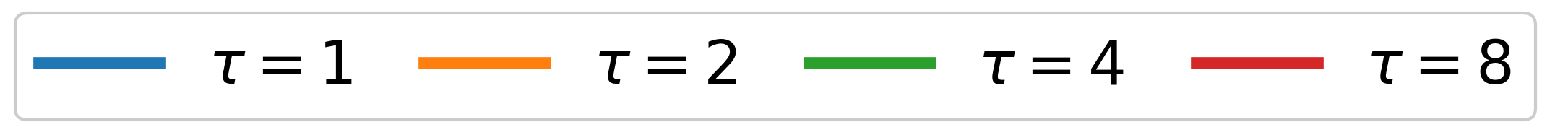}\label{fig:cifar_leg}}
  \caption{CIFAR-10 FID scores with regularization parameter $\mu=10$ in Figure~\ref{fig:cifar_fida} and $\mu=1$ in Figure~\ref{fig:cifar_fidb}.}
  \label{fig:cifar_fid}
\end{figure*} 
\begin{figure*}[t!]
  \centering
  \subfloat[][$\mu=10$]{\includegraphics[width=\textwidth]{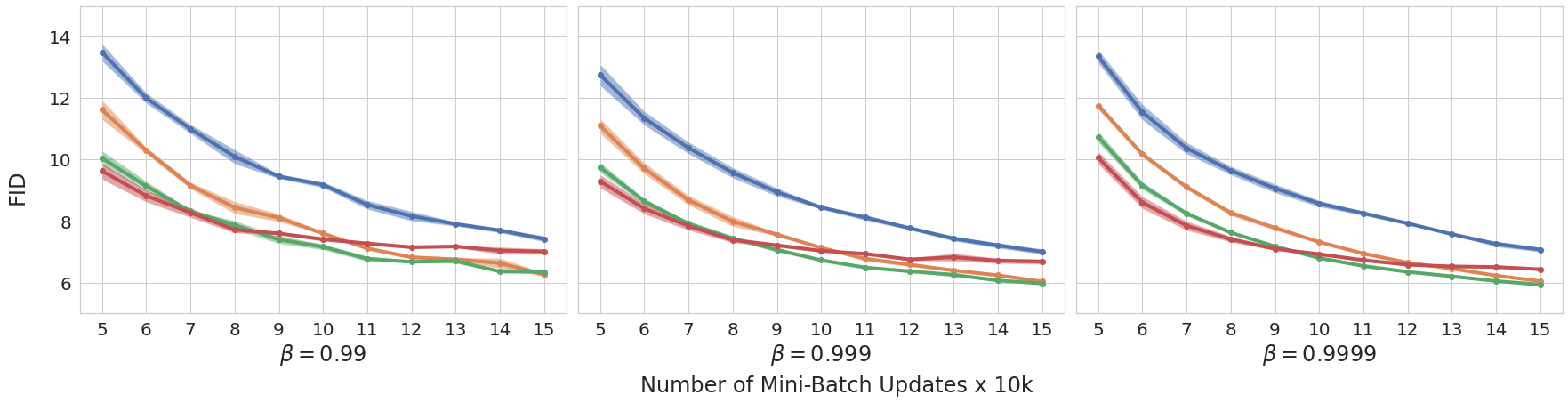}\label{fig:celeb_fida}}
  
  \subfloat[][$\mu=1$]{\includegraphics[width=\textwidth]{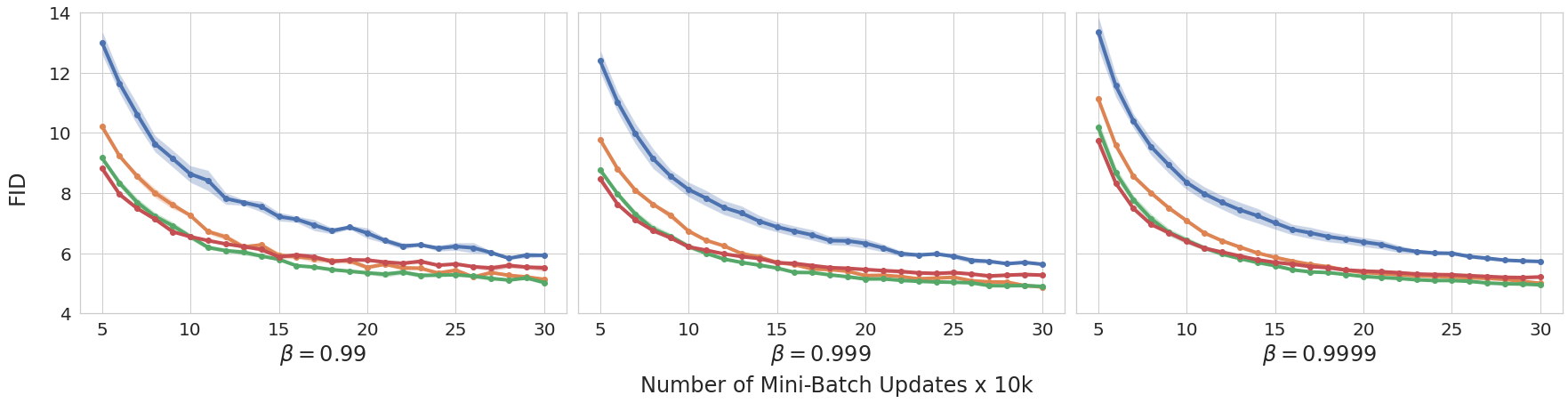}\label{fig:celeb_fidb}}
  
    \subfloat{\includegraphics[width=.45\textwidth]{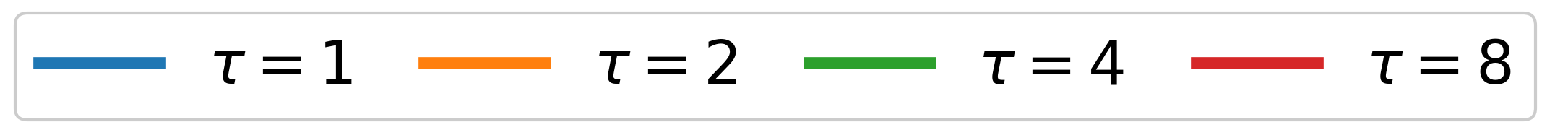}\label{fig:celeb_leg}}
  \caption{CelebA FID scores with regularization parameter $\mu=10$ in Figure~\ref{fig:cifar_fida} and $\mu=1$ in Figure~\ref{fig:cifar_fidb}.}
  \label{fig:celeb_fid}
\end{figure*} 

\subsection{Generative Adversarial Networks: Image Datasets}
\label{sec:ganimage}
We now investigate the role timescale separation has on training generative adversarial networks parameterized by deep neural networks. The empirical benefits of training with a timescale separation have been documented previously. For example,~\citet{heusel2017gans} showed on a number of image datasets that a timescale separation between the generator and discriminator improves generation performance as measured by the Frechet Inception Distance (FID). Since then a significant number of papers have presented results training generative adversarial networks with timescale separation.
Moreover, it is common in the literature for the discriminator to be updated multiple times between each update of the generator~\cite{arjovsky2017wasserstein}. Indeed, it has been widely demonstrated that this heuristic improves the stability and convergence of the training process and locally it has a similar effect as including a timescale separation between the generator and discriminator. The disadvantage of this approach is that the number of gradient calls per generator update increases and consequently the convergence is then slower in terms of wall clock time when a similar effect could potentially be achieved by a learning rate separation between the generator and discriminator. We remark that it appears to be reasonably common for practitioners to fix a shared learning rate for the generator and discriminator along with a pre-selected number of discriminator updates per generator update and not thoroughly investigate the impact timescale separation has on the training process.

The goal of our generative adversarial network experiments is to reinforce the importance of the timescale separation between the generator and the discriminator as a hyperparameter in the training process, demonstrate how it changes the behavior along the learning path, and show that it is compatible with a number of common training heuristics. 
This is to say that our goal is not necessarily to show state-of-the art performance, but rather to perform experiments that allow us to make insights relevant to the theory in this paper. We remark that our empirical work on training generative adversarial networks is distinct from and complimentary to that of~\citet{heusel2017gans} in several ways. The theory given by~\citet{heusel2017gans} only applies to stochastic stepsizes, however in the experiments they implemented constant step sizes. We train with mini-batches and decaying stepsizes, which does satisfy the theory we provide. Moreover, by and large, the experiments by~\citet{heusel2017gans} compare a fixed learning rate ratio between the generator and discriminator to multiple fixed shared learning rates for the generator and discriminator. In contrast, we fix a learning rate for the generator and explore the behavior of the training process as the timescale parameter $\tau$ is swept over a given range.

We build our experiments based on the methods and implementations of~\citet{mescheder2018training} and explore both the CIFAR-10 and CelebA image datasets. We train the generative adversarial networks with the non-saturating objective function
and the $R_1$ gradient penalty proposed by~\citet{mescheder2018training} with regularization parameters $\mu \in \{1, 10\}$. We note that the non-saturating objective results in a game that is not zero-sum, however it is commonly used in practice and under the realizable assumptions is it locally equivalent to the zero-sum objective as discussed at the end of Section~\ref{sec:exp_dirac}.
The network architectures for the generator and discriminator are both based on the ResNet architecture. The initial learning rate for the generator in all of our experiments is fixed to be $\gamma = 0.0001$ and we decay the stepsizes so that at update $k$ the learning rate of the generator is given by $\gamma_{1, k} = \gamma_1/(1+\nu)^k$ where $\nu=0.005$ and the learning rate of the discriminator is $\gamma_{2, k}=\tau \gamma_{1,k}$. For each experiment the batch size is $64$, the latent data is drawn from a standard normal distribution of dimension $256$, and the resolution of the images is $32\times 32\times 3$.  Finally, as an optimizer, we run RMSprop with parameter $\alpha=0.99$.  Again, the theory we provide does not strictly apply to using RMSprop, but it is ubiquitous in practice for training generative adversarial networks and if the timescale separation is sufficiently large so that the eigenvalues are purely real in the Jacobian then the theory we provide is applicable as remarked previously. We provide further details on the network and hyperparameters in Appendix~\ref{app_sec:genad}.
A final heuristic and hyperparameter that we explore in conjunction with the timescale separation $\tau$ is that of using an exponential moving average to produce the model that is evaluated. This means that at each update $k$, given that the parameters of the generator are given by $x_{1, k}$, the moving average $\bar{x}_k = x_{1, k}\beta + \bar{x}_{1, k-1}(1-\beta)$ is kept where $\beta \in (0, 1)$. Experimental studies have shown that this heuristic can yield a significant improvement in terms of both the inception score and the FID~\cite{yazici2019unusual, gidel2018variational}. The success of this method is thought to be a result of dampening both rotational dynamics and the noise from the randomness in the mini-batches of data.

We run the training algorithm with the learning rate ratio $\tau$ belonging to the set $\{1, 2, 4, 8\}$ and the regularization parameter $\mu$ belonging to the set $\{1, 10\}$. For each choice of $\tau$ and $\mu$, we retain exponential moving averages of the generator parameters for $\beta \in \{0.99, 0.999, 0.9999\}$. The training process is repeated $3$ times for each hyperparameter configuration to rule out noise from random seeds and the performance is evaluated along the learning path at every 10,000 updates in terms of the FID score. We report the mean scores and the standard error of the mean over the repeated experiments. We run the experiments with $\mu=1$ for 150k mini-batch updates and the experiments with $\mu=10$ for 300k mini-batch updates. 

The results for each dataset across the hyperparameter configurations are presented in numeric form in Figure~\ref{fig:fid}. Figure~\ref{fig:celeb_samples} shows some generated samples selected at random for each dataset with the hyperparameter configuration that performed best in terms of the FID score at the end of the training process. Figure~\ref{fig:big_samples} in Appendix~\ref{app_sec:genad} has several more generated samples for each dataset selected at random. We now describe the key observations from the experiments for each dataset.

\paragraph{CIFAR-10.}
The FID scores along the learning path for CIFAR-10 with $\mu=10$ and $\mu=1$ are presented in Figures~\ref{fig:cifar_fida} and~\ref{fig:cifar_fidb}, respectively. The corresponding scores in numeric form are given in Figures~\ref{fig:t_a}, \ref{fig:t_c}, and ~\ref{fig:t_e} for $\mu=10$ at 150k iterations and $\mu=1$ at 150k and 300k iterations, respectively. To begin, we observe that the exponential moving average significantly improves performance, and of the parameters considered, $\beta=0.9999$ performed best. Relevant to this work, we observe that the performance gain from using an exponential moving average appears to be maximized when the ratio of learning rates is smallest. This may indicate that some of the dynamics in $\tau$-{\gda} are dampened by timescale separation in this generative adversarial network experiment, similarly to as observed for the simpler experiments presented previously. Moreover, we that timescale separation also has a significant impact on the FID score of the training process. Indeed, even selecting $\tau=2$ versus $\tau=1$ can yield an impressive performance gain. In this experiment for each regularization parameter, $\tau=4$ converges fastest, followed by $\tau=8$, then $\tau=2$, and finally $\tau=1$. Finally, observe that the performance with regularization $\mu=1$ is far superior to that with regularization $\mu=10$. Interestingly, the last pair of conclusions are in line with the insights drawn from the simple Dirac-GAN experiment in Section~\ref{sec:exp_dirac}. In particular, timescale separation only speeds up to convergence until hitting a limiting value and there is a fundamental interplay between timescale separation, regularization, and convergence rate. This indicates that it may be possible to transfer some of the insights made on simple generative adversarial network formulations to the much more complex problem where players are parameterized by neural
networks.

\begin{figure*}[t!]
  \centering
    \subfloat[CIFAR-10]{\includegraphics[width=.48\textwidth]{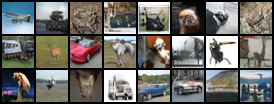}\label{fig:cifar_4}}\hfill 
    \subfloat[CelebA]{\includegraphics[width=.48\textwidth]{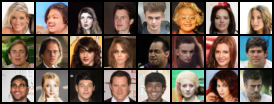}\label{fig:celeb_4}}
  \caption{CIFAR-10 and CelebA samples from generator at $300$k iterations with $\beta=0.9999$ and $\tau =4$.}
  \label{fig:celeb_samples}
\end{figure*} 
\begin{figure*}[t!]
\centering
\small

\subfloat[CIFAR-10 FID at 150k updates with $\mu=10$]{\begin{tabular}{|c||c|c|c|}\hline
 $\tau\backslash\beta$  &  0.99 & 0.999 & 0.9999   \\\hline\hline
 1 &  29.39 $\pm$ 0.37  & 27.56 $\pm$ 0.34 & 26.29 $\pm$ 0.2 \\ \hline
2 &  24.94 $\pm$ 0.25  & 23.4 $\pm$ 0.22  & 22.36 $\pm$ 0.23 \\ \hline
 4 &  \textbf{24.14 $\pm$ 0.42}   & \textbf{22.53 $\pm$ 0.23} & \textbf{21.62 $\pm$ 0.05} \\ \hline
 8 &  24.63 $\pm$ 0.28   & 23.35 $\pm$ 0.3 & 22.52 $\pm$ 0.27\\ \hline
\end{tabular}\label{fig:t_a}}\hfill 
\subfloat[CelebA FID at 150k updates with $\mu=10$]{\begin{tabular}{|c||c|c|c|}\hline
 $\tau\backslash\beta$  &  0.99 & 0.999 & 0.9999   \\\hline\hline
 1 &  7.42 $\pm$ 0.1   & 7.01 $\pm$ 0.09  & 7.07 $\pm$ 0.08 \\ \hline
2 &   \textbf{6.26 $\pm$ 0.06} & 6.04 $\pm$ 0.02  & 6.05 $\pm$ 0.01 \\ \hline
 4 &  6.34 $\pm$ 0.03   & \textbf{5.97 $\pm$ 0.03}  & \textbf{5.93 $\pm$ 0.002} \\ \hline
 8 &  7.01 $\pm$ 0.09   & 6.69 $\pm$ 0.09 & 6.43 $\pm$ 0.04 \\ \hline
\end{tabular}\label{fig:t_b}}

\subfloat[CIFAR-10 FID at 150k updates with $\mu=1$]{\begin{tabular}{|c||c|c|c|}\hline
 $\tau\backslash\beta$ &  0.99 & 0.999 & 0.9999   \\\hline\hline
1 & 28.1  $\pm$ 0.52 & 26.18 $\pm$ 0.54 & 24.6  $\pm$ 0.3  \\\hline
2 & 23.49 $\pm$ 0.49 & 22.0  $\pm$ 0.38 & 21.05 $\pm$ 0.38  \\ \hline
4 & \textbf{22.35 $\pm$ 0.15} & \textbf{20.71 $\pm$ 0.06} & \textbf{19.49 $\pm$ 0.06}  \\ \hline
8 & 22.46 $\pm$ 0.37 & 21.36 $\pm$ 0.44 & 20.27 $\pm$ 0.4  \\ \hline
\end{tabular}\label{fig:t_c}}\hfill 
\subfloat[CelebA FID at 150k updates with $\mu=1$]{\begin{tabular}{|c||c|c|c|}\hline
$\tau\backslash\beta$ &  0.99 & 0.999 & 0.9999   \\\hline\hline
1 & 7.22 $\pm$ 0.15 & 6.87 $\pm$ 0.17 & 7.01 $\pm$ 0.22 \\\hline
2 & 5.93 $\pm$ 0.12 & 5.69 $\pm$ 0.03 & 5.86 $\pm$ 0.04 \\ \hline
4 & \textbf{5.8  $\pm$ 0.04} & \textbf{5.51 $\pm$ 0.04} & \textbf{5.59 $\pm$ 0.06} \\ \hline
8 & 5.88 $\pm$ 0.05 & 5.68 $\pm$ 0.03 & 5.7  $\pm$ 0.05 \\ \hline
\end{tabular}\label{fig:t_d}}

\subfloat[CIFAR-10 FID at 300k updates with $\mu=1$]{\begin{tabular}{|c||c|c|c|}\hline
 $\tau\backslash\beta$ &  0.99 & 0.999 & 0.9999   \\\hline\hline
1 & 26.1  $\pm$ 0.44 & 23.98 $\pm$ 0.39 & 22.4 $\pm$ 0.35 \\\hline
2 & 21.44 $\pm$ 0.5  & 20.15 $\pm$ 0.32 & 18.5 $\pm$ 0.31 \\ \hline
4 & \textbf{20.67 $\pm$ 0.04} & \textbf{19.23 $\pm$ 0.11} & \textbf{17.72 $\pm$ 0.05} \\ \hline
8 & 21.09 $\pm$ 0.33 & 19.81 $\pm$ 0.22 & 18.45 $\pm$ 0.25 \\ \hline
\end{tabular}\label{fig:t_e}}\hfill
\subfloat[CelebA FID at 300k updates with $\mu=1$]{\begin{tabular}{|c||c|c|c|}\hline
$\tau\backslash\beta$ &  0.99 & 0.999 & 0.9999   \\\hline\hline
1 & 5.93 $\pm$ 0.06 & 5.63 $\pm$ 0.01 & 5.72 $\pm$ 0.02 \\\hline\
2 & 5.13 $\pm$ 0.06 & \textbf{4.88 $\pm$ 0.02} & 5.0 $\pm$ 0.01 \\ \hline
4 & \textbf{5.03 $\pm$ 0.06} & 4.9 $\pm$ 0.03 & \textbf{4.95 $\pm$ 0.05} \\ \hline
8 & 5.51 $\pm$ 0.11 & 5.27 $\pm$ 0.04 & 5.21 $\pm$ 0.05 \\ \hline
\end{tabular}\label{fig:t_f}}
\caption{FID Scores on CIFAR-10 and CelebA. }
\label{fig:fid}
\end{figure*}

\paragraph{CelebA.}
The FID scores along the learning path for CIFAR-10 with $\mu=10$ and $\mu=1$ are presented in Figures~\ref{fig:celeb_fida} and~\ref{fig:celeb_fidb}, respectively. The corresponding scores in numeric form are given in Figures~\ref{fig:t_b}, \ref{fig:t_d}, and ~\ref{fig:t_f} for $\mu=10$ at 150k iterations and $\mu=1$ at 150k and 300k iterations, respectively. In this experiment we that while the exponential moving average helps performance, the gain is not as drastic as it was for CIFAR-10. However, timescale separation in combination with the regularization does have a major effect on the the FID score of the training process in this experiment. 
For regularization $\mu=10$, the timescale parameters of $\tau=2$ and $\tau=4$ outperform $\tau=1$ and $\tau=8$ by a wide margin, again highlighting that timescale separation can speed up convergence until a certain point where it can potentially slow it down owing to the effect on the conditioning of the problem locally. A similar trend can be observed with regularization $\mu=1$, but with $\tau=8$ performing closer to $\tau=2$ and $\tau=4$. We again observe in this experiment that for all timescale separation parameters, the performance is significantly improved with regularization $\mu=1$ as compared with $\mu=10$. This once again highlights the importance of considering how this the hyperparameters of regularization and timescale interact and dictate the local convergence rates.

In summary, we took a well-performing method and implementation for training generative adversarial networks and demonstrated that timescale separation is an extremely important, and easy to implement, hyperparameter that is worth careful consideration since it can have a major impact on the convergence speed and final performance of the training process.

 \section{Related Work}
\label{sec:related}
In this section, we provide a review of related work at the intersection machine learning and game theory, as well as connections to dynamical systems theory and control. 

\subsection{Machine Learning and Learning in Games} 
Given the extensive work on the topic of learning in games in machine learning that has gone on over the last several years, we cannot cover all of it and instead focus our attention on only the most relevant to this paper. We begin by reviewing solution concepts developed for the class of games under consideration and then discuss some learning dynamics studied in the literature beyond gradient descent-ascent. Following this, we delineate the related work studying gradient descent-ascent in non-convex, non-concave zero-sum games and finish by making note of the literature on non-convex, concave optimization.

\paragraph{Solution Concepts.} Owing to the numerous applications in machine learning, a significant portion of the modern work on learning in games has focused on the zero-sum formulation with non-convex, non-concave cost functions. Most recently, \citet{daskalakis2020complexity} tout the importance and significance of this class of games in a paper on the complexity of finding equilibria (in particular, in the constrained setting) in such games. Consequently, local solution concepts have been broadly adopted. Compared to the standard game-theoretic notions of equilibrium that characterize player's incentive to deviate given the game and information structure, local equilibrium concepts restrict the deviation search space to a suitable local neighborhood. Following the standard game-theoretic viewpoint, a vast number of works in machine learning study the local Nash equilibrium concept and critical points satisfying gradient-based sufficient conditions for the equilibrium, which are often referred to as differential Nash equilibria~\cite{ratliff2013allerton, ratliff2014acc, ratliff:2016aa}. Based on the observation that in non-convex, non-concave zero-sum games the order of play is fundamental in the definition of the game, there has been a push toward considering local notions of the Stackelberg equilibrium concept, which is the usual game-theoretic equilibrium when there is an explicit order of play between players. In the zero-sum formulation, Stackelberg equilibrium are often referred to as minmax equilibria. Similar to as for the Nash equilibrium, gradient-based sufficient conditions for local minmax/Stackelberg equilibrium have been given~\cite{fiez:2020icml, jin2019local,daskalakis:2018aa} and such critical points have been referred to as differential Stackelberg equilibria~\cite{fiez:2020icml}. 
We remark that it has been shown that local/differential Nash equilibria are a subset of local/differential Stackelberg equilibria~\cite{jin2019local, fiez:2020icml}.
Following past works, we adopt the terminology of differential Nash equilibrium and differential Stackelberg equilibrium in this paper as the meaning of strict local Nash equilibrium and strict local minmax/Stackelberg equilibrium, respectively. Finally we mention the proximal equilibria proposed by~\citet{farniagans}, which we do not consider in this work, that depending on a regularization parameter can interpolate between the local Nash and local Stackelberg equilibrium notions.

\paragraph{Learning Dynamics.} 
Given that the focus of this work is on gradient descent-ascent, we center our coverage of related work on papers analyzing its behavior. Nonetheless, we mention that a significant number of learning dynamics for zero-sum games have been developed in the past few years, in some cases motivated by the shortcomings of gradient descent-ascent without timescale separation. The methods include optimistic and extra-gradient algorithms~\cite{daskalakis2017training, gidel2018variational, mertikopoulos2018optimistic}, negative momentum~\cite{gidel2019negative}, gradient adjustments~\cite{mescheder2017numerics, balduzzi2018mechanics, letcher2019differentiable}, and opponent modeling methods~\cite{ zhang2010multi, metz2016unrolled, foerster2018learning, letcher2018stable, schafer2019competitive}, among others. While the aforementioned learning dynamics possess some desirable characteristics, they cannot guarantee that the set of stable critical points coincide with a set of local equilibria for the class of games under consideration. However, there have been a select few learning dynamics proposed that can guarantee the stable critical points coincide with either the set of differential Nash equilibria~\cite{adolphs2019local, mazumdar2019finding} or the set of differential Stackelberg equilibria~\cite{fiez:2020icml, wang2019solving}---effectively solving the problem of guaranteeing local convergence to only a class of local equilibria. However, since each of the algorithms achieving the equilibrium stability guarantee require solving a linear equation in each update step, they are not efficient and can potentially suffer from degeneracies along the learning path in applications such as generative adversarial networks. These practical shortcomings motivate either proving that existing learning dynamics using only first-order gradient feedback achieve analogous theoretical guarantees or developing novel computationally efficient learning dynamics that can match the theoretical guarantee of interest.

\paragraph{Gradient Descent-Ascent.} 
Gradient descent-ascent has been studied extensively in non-convex, non-concave zero-sum games since it is a natural analogue to gradient descent from optimization, is computationally efficient, and has been shown to be effective in practice for applications of interest when combined with common heuristics. A prevailing approach toward gaining understanding of the convergence characteristics of gradient descent-ascent has been to analyze the local stability around critical points of the continuous time limiting dynamical system. The majority of this work has not considered the impact of timescale separation. Numerous papers have pointed out that the stable critical points of gradient descent-ascent without timescale separation may not be game-theoretically meaningful. In particular, it has been shown that there can exist stable critical points that are not differential Nash equilibrium~\cite{daskalakis:2018aa, mazumdar2020gradient}. Furthermore, it is known that there can exist stable critical points that are not differential Stackelberg equilibria~\cite{jin2019local}. The aforementioned results rule out the possibility that gradient descent-ascent without timescale separation can guarantee equilibrium convergence. In terms of the stability of equilibria, it is known that differential Nash equilibrium are stable for gradient descent-ascent without timescale separation~\cite{daskalakis:2018aa, mazumdar2020gradient}, but that there can exist differential Stackelberg equilibria which are not stable with respect to gradient descent-ascent without timescale separation. 

The work of~\citet{jin2019local} is the most relevant exploring how the aforementioned stability properties of gradient descent-ascent change with timescale separation. In particular,~\citet{jin2019local} investigate whether the desirable stability characteristics (stability of differential Nash equilibria) and undesirable stability characteristics (stability of non-equilibrium critical points and instability of differential Stackelberg equilibria) of gradient descent without timescale separation are maintained and remedied, respectively with timescale separation. In terms of the former query, extending the examples shown in \citet{mazumdar2020gradient} and \citet{daskalakis:2018aa},~\citet{jin2019local} show that differential Nash equilibrium are stable for gradient descent-ascent with any amount of timescale separation. 

On the other hand, for the latter query,~\citet{jin2019local} shows (in Proposition 27) two interesting examples: (a) for an a priori fixed $\tau$, there exists a game with a differential Stackelberg equilibrium that is not stable and (b) for an a priori fixed $\tau$, there exists a game with a stable critical point that is not a differential Stackelberg equilibrium.
 However, (a) does not imply that for the constructed game, there does not exist another (finite) $\tau$---independent of the game parameters---such the differential Stackelebrg equilibrium is stable for all larger $\tau$. In simple language, the result summarized in (a) says the following: \emph{if a bad timescale separation is chosen, then convergence may not be guaranteed}. Similarly, (b)  does not imply that there is no $\tau$ such that for all larger $\tau$ for the constructed game instance, the critical point becomes unstable. Again, in simple language, the result summarized in (b) says the following: \emph{if a bad timescale separation is chosen, then non-game theoretically meaningful equilibria may persist}. 
While at first glance this set of results may appear to indicate that the undesirable stability characteristics of gradient descent without timescale separation cannot be averted by any finite timescale separation, it is important to emphasize that these results \emph{do not} answer the questions of whether there (a) exists a game with a critical point that is not a differential Stackelberg equilibrium which is stable with respect to gradient descent-ascent without timescale separation and remains stable for all finite timescale separation ratios or (b) exists a game with a differential Stackelberg equilibrium that is not stable for all finite timescale separation ratios. The preceding questions are left open from previous work and are exactly the focus of this paper. In Appendix~\ref{app_sec:related}, we go into greater detail on the comparison between Proposition 27 of \citet{jin2019local} as we believe this to be an important point of distinction between Theorem~\ref{thm:iffstack} and \ref{prop:instability} in this paper.

Finally,~\citet{jin2019local} study the an infinite timescale separation ratio and show that the stable critical points of gradient descent-ascent coincide with the set of differential Stackelberg equilibria in this regime. This result effectively shows that gradient descent-ascent can guarantee only equilibrium convergence with timescale separation, albeit infinite. We remark that an equivalent result in the context of general singularly perturbed systems has been known in the literature~\cite[Chap.~2]{kokotovic1986singular} as we discuss further in Section~\ref{sec:prelim_obs}. Finally, we point out that since an infinite timescale separation does not result in an implementable algorithm, fully understanding the behavior with a finite timescale separation is of fundamental importance and the motivation for our work.  

Beyond the work of~\citet{jin2019local} considering timescale separation in gradient descent-ascent, it is worth mentioning the work of~\citet{chasnov2019uai} and~\citet{heusel2017gans}.~\citet{chasnov2019uai} study the impact of timescale separation on gradient descent-ascent, but focus on the convergence rate as a function of it given an initialize around a differential Nash equilibrium and do not consider the stability questions examined in this paper.~\citet{heusel2017gans} study stochastic gradient descent-ascent with timescale separation and invoke the results of~\citet{borkar2008stochastic} for analysis. The stochastic approximation results the claims rely on guarantee the convergence of the system locally to a stable critical point. Consequently, the claim of convergence to differential Nash equilibria of stochastic gradient descent-ascent given by~\citet{heusel2017gans} only holds given an initialization in a local neighborhood around a differential Nash equilibrium. In this regard, the issue of the local stability of the types of critical point is effectively assumed away and not considered. In contrast, we are able to combine our stability results for gradient descent-ascent with timescale separation together with the stochastic approximation theory  of~\citet{borkar2008stochastic} to guarantee local convergence to a differential Stackelberg equilibrium in Section~\ref{sec:stochastic}. We remark that~\citet{heusel2017gans} empirically demonstrate that timescale separation can significantly improve the performance of gradient descent-ascent when training generative adversarial networks. 

The final relevant line of work studying gradient descent-ascent is specific to generative adversarial networks. The results from this literature develop assumptions relevant to generative adversarial networks and then analyze the stability and convergence properties of gradient descent-ascent under them (see, e.g., works by \citet{metz2016unrolled,goodfellow2014generative,daskalakis2017training,nagarajan2017gradient,mescheder2018training}). Within this body of work, there has been a significant amount of effort focusing on how the stability (and, hence, convergence properties) of gradient descent-ascent in generative adversarial networks can be enhanced with regularization methods. 
\citet{nagarajan2017gradient} show, under suitable assumptions, that gradient-based methods for training generative adversarial networks are locally convergent assuming the data distributions are absolutely continuous. However, as observed by  \citet{mescheder2018training}, such assumptions not only may not be satisfied by many practical generative adversarial network training scenarios such as natural images, but it can often be the case that the data distribution is concentrated on a lower dimensional manifold. The latter characteristic leads to nearly purely imaginary eigenvalues and highly ill-condition problems.  
\citet{mescheder2018training} provide an explanation for observed instabilities consequent of the true data
distribution being concentrated on a lower dimensional manifold using  discriminator gradients orthogonal to the tangent space
of the data manifold.  Further, the authors introduce regularization via gradient penalties that leads to convergence guarantees under less restrictive assumptions than were previously known. In this paper, we further extend these results to show that convergence to differential Stackelberg equilibria is guaranteed under a wide array of hyperparameter configurations (i.e., learning rate ratio and regularization).

\paragraph{Nonconvex-Concave Optimization.} 
A final related line of work is on nonconvex-concave optimization~\cite{nouiehed2019solving, rafique2018non, lin2019gradient, lin2020near, lu2020hybrid, ostrovskii2020efficient}. The focus in this set of works (among many others on the topic) is on characterizing the iteration complexity to stationary points, rather than stability and asymptotic convergence as in the non-convex, non-concave zero-sum game setting. The focus on stationary points in this body of work is reasonable since, to our knowledge, the results obtained are for \emph{$\ell$-weakly convex-concave games}, a subclass of non-convex-concave games in which the minimizing player faces an $\ell$-weakly convex optimization problem for each fixed choice of the maximizing player. The primary relevance of work on this problem is that a number of the algorithms rely on timescale separation and variations of gradient descent-ascent. Moreover, the methods for obtaining fast convergence rates may be relevant to future work attempting to characterize fast rates in the non-convex, non-concave setting after there is a more fundamental understanding of the stability and asymptotic convergence.

\subsection{Historical Perspective: Dynamical Systems and Control} The study of gradient descent-ascent dynamics with timescale separation between the minimizing and maximizing players is closely related to that of singularly perturbed dynamical systems~\cite{kokotovic1986singular}. Such systems arise in classical control and dynamical systems in the context of physical systems that either have multiple states which evolve on different timescales due to some underlying immutable physical process or property, or a single dynamical system which evolves on a sub-manifold of the larger state-space. For example, robot manipulators or end effectors often have have slower mechanical dynamics than  electrical dynamics. On the other hand, 
in electrical circuits or mechanical systems, certain resistor-capacitor circuits or spring-mass systems have a state which evolves subject to a constraint equation~\cite{sastry1981jump,lagerstrom1972basic}. Due to their prevalence, singularly perturbed systems have been studied extensively with one of the outcomes being a number of works on determining the range of perturbation parameters for which the overall system is stable~\cite{kokotovic1986singular,saydy1990guardian,saydy1996newstability}. We exploit these results and analysis techniques to develop novel results for learning in games. 
One of contributions of this work is the introduction of the algebraic analysis techniques to the machine learning and game theory communities. These tools open up new avenues for algorithm synthesis; we comment on potential directions in the concluding discussion section.

This being said, there are a couple key difference between the present setting and that of the classical literature including the following:
\begin{enumerate}[itemsep=0pt,topsep=0pt]
    \item 
    \textbf{The perturbation parameter is no longer an immutable characteristic of the physical system, but rather a hyperparameter subject to design.} 
    Indeed, in singular perturbation theory, the typical dynamical system studied takes the form
    \begin{equation}
             \dot{x}=g_1(x,y)  \ \ \epsilon \dot{y}=g_2(x,y) 
             \label{eq:gensingpert}
    \end{equation}
    where $\epsilon$ is a small parameter that abstracts some physical characteristics of the state variables. On the other hand, in learning in games, the continuous time limiting dynamical system of gradient descent-ascent for a zero-sum game defined by $f\in C^2(X\times Y,\mb{R})$ takes the form
    \begin{equation}
             \dot{x}=-D_1f(x,y)  \ \ \dot{y}=\tau D_2f(x,y) 
       \label{eq:gdactlimit}
    \end{equation}
    where the $x$--player seeks to minimize $f$ with respect to $x$ and the $y$--player seeks to maximize $f$ with respect to $y$, and $\tau$ is the ratio of learning rates (without loss of generality) of the maximizing to the minimizing player. These learning rates---and hence the value of $\tau$---are hyperparameters subject to design in most machine learning and optimization applications. Another feature of \eqref{eq:gdactlimit} as compared to \eqref{eq:gensingpert}, is that the dynamics $D_if$ are partial derivatives of a function $f$, which leads to the second key difference.
     \item 
     \textbf{There is structure in the dynamical system that arises from gradient-play which reflects the underlying game theoretic interactions between players.} This structure can be exploited in obtaining convergence guarantees in machine learning and optimization applications of game theory. For instance, minmax optimization is analogous to a zero sum game for which the local linearization of gradient descent-ascent dynamics has the structure
    \[J=\bmat{A & B\\
    -\tau B^\top & -\tau C}\]
    where $A=A^\top$ and $C=C^\top$ and $\tau$ is the learning rate ratio or timescale separation parameter.
    Such block matrices have very interesting properties. In particular, second order optimality conditions for a minmax equilibrium correspond to positive definiteness of the first Schur complement $\schurtt_1(J)=A-BC^{-1}B^\top>0$, and of $-C>0$ \cite{fiez:2020icml}. This turns out to be keenly important for understanding convergence of  gradient descent-ascent. Furthermore, due to the structure of $J$, tools from the theory of block operators (see, e.g., works by~\citet{tretter2008spectral, magnus1988linear,lancaster1985theory}) such as the quadratic numerical range can be exploited (and combined with singular perturbation theory) to understand the effects of hyperparameters such as $\tau$ (the learning rate ratio) and regularization (which is common in applications such as generative adversarial networks) on convergence. 
\end{enumerate}

\section{Discussion}
\label{sec:discussion}
In this paper, we prove a necessary and sufficient condition for the convergence of gradient descent-ascent with timescale separation to differential Stackelberg equilibria  in zero sum games. This answers a long standing open question about provable convergence of first order methods for zero-sum games to local minimax equilibria. Specifically, we provide necessary and sufficient conditions for the convergence of $\tau$-{\gda} to differential Stackelberg equilibria. A key component of the proof is the construction of a (tight) finite lower bound on the learning rate ratio $\tau$ for which stability of the game Jacobian is guaranteed, and hence local asymptotic convergence of $\tau$-{\gda}.
In addition, we provide results on iteration complexity and convergence rate and apply the results to generative adversarial networks under mild assumptions on the data distribution.  For both differential Nash equilibira and the superset of differential Stackelberg equilibria, we provide estimates on the neighborhood on which convergence is guaranteed. 

This being said, the question of the size of the region of attraction remains open. As commented on earlier in the paper, an alternative but related technique tackles the nonlinear system directly. The downside of this technique is that one needs to have in hand (or be able to construct) Lyapunov  functions for both  the \emph{boundary layer model} (i.e., the system that arises from treating the choice variable of the slow player as being `static') and  the \emph{reduced order model} (i.e., the system that arises from plugging in the implicit mapping from the fast player's action to the slow player's action into the slow player's dynamics). 
A convex combination of these functions provides   a Lyapunov function for the original system $\dot{x}=-\Lambda_\tau g(x)$. The level sets of this combined Lyapunov function then give a better sense of the region of attraction and, in fact, one can optimize over the weighting in the convex combination in order to obtain better estimates of the region of attraction. This is an interesting avenue to explore in the context of learning in games with lots of intrinsic structure that can potentially be exploited to improve both the rate of convergence and the region on which convergence is guaranteed.

Another significant contribution of this work is the fact that we introduce tools that are arguably new to the machine learning and optimization communities and expose interesting new directions of research. In particular, the notion of a guard map, which is arguably even an obscure tool in modern control and dynamical systems  theory, is `rediscovered' in this paper. The is  potential to leverage this concept in not only providing certificates for performance (e.g., beyond stability to  robustness) but also in synthesizing algorithms with performance guarantees. For instance, one observation from our empirical analysis is that convergence rate is not only limited by the eigenvalues of the Schur complement of the Jacobian, but the fastest convergence appears to occur when there are complex components of the eigenvalues. In short, some cycling is beneficial. Better understanding this fact from a theoretical perspective is an open question, as is optimizing the rate of convergence by exploiting these observations.

Finally, another set of related open questions center on practical considerations for the  efficient use of  first order methods.  For instance, with respect to generative adversarial networks, the exponential moving average is known to empirically reduce the negative effects of cycling.  Additionally, increasing the learning rate ratio does lead to predominantly real eigenvalues which in turn reduces cycling. Understanding the trade offs between not only these two hyperparameters but also regularization is very important for practical implementations.  Empirically, we study the tradeoffs between the learning rate ratio, regularization parameter, and the parameter controlling the degree of ``smoothness'' in the exponential moving average, another common heuristic that performs well in practice. There is an open line of research related to analytically characterizing the tradeoffs between these three hyperparameters. However,  in the absence of theoretical tools for exploring these issues, what are reasonable and principled heuristics?

To conclude, while we arguably definitively address a  standing open question for first order methods for learning in zero-sum games/minmax optimization problems, there a many open directions exposed by the tools introduced and empirical observations discovered in this work.

\section*{Acknowledgements}
This work is funded by the Office of Naval Research (YIP Award) and National Science Foundation CAREER Award (CNS-1844729). Tanner Fiez is also funded by a National Defense Science and Engineering Graduate Fellowship. We thank Daniel Calderone for the helpful discussions, in particular on linear algebra results as they pertain to the results in this paper. Finally, we thank~\citet{mescheder2018training} for providing a high quality open source implementation of the generative adversarial network experiments they performed, which facilitated and expedited the experiments we performed.

\bibliographystyle{plainnat}
\bibliography{2020minmaxfiniteraterefs}

\begin{appendices}
\renewcommand\contentsname{Table of Contents}
\tableofcontents
\addtocontents{toc}{\protect\setcounter{tocdepth}{1}}

\section{Helper Lemmas and Additional Mathematical Preliminaries}
\label{app_sec:helperlemmas}
In this appendix, we present a handful of technical lemmas and review some additional mathematical preliminaries excluded from the main body but which are important in proving the results in the paper.

The following technical lemma is used in proving an upper bound on the spectral radius of the linearization of the discrete time update $\tau$-{\gda} a requirement for obtaining the convergence rate results. 
\begin{lemma} The function $c(z)=(1-z)^{1/2}+\tfrac{z}{4}-(1-\tfrac{z}{2})^{1/2}$ satisfies $c(x)\leq 0$ for all $z\in [0,1]$.
\label{lem:bdd}
\end{lemma}
\begin{proof}
 Since $c(0)=0$ and $c(1)=\tfrac{1}{4}-\tfrac{1}{\sqrt{2}}\leq 0$, we simply need to show that $c'(z)\leq 0$ on $(0,1)$ to get that $c(z)$ is a decreasing function on $[0,1]$, and hence negative on $[0,1]$. Indeed,
$c'(z)=\tfrac{1}{4}+\tfrac{1}{2\sqrt{4-2z}}-\tfrac{1}{2\sqrt{1-z}}\leq0$
since $(1-z)^{-1/2}-(4-2z)^{-1/2}\geq 1/2$ for all $z\in (0,1)$.
\end{proof}

The following proposition is a well-known result in numerical analysis and can be found in a number of books and papers on the subject. Essentially, it provides an asymptotic convergence guarantee for a discrete time update process or dynamical system.
\begin{proposition}[Ostrowski's Theorem~\cite{argyros:1999aa}; Theorem 10.1.2~\cite{ortega:1970aa}]
Let $x^\ast$ be a fixed point for the discrete dynamical system
$x_{k+1}=F(x_k)$. If the spectral radius of the Jacobian satisfies
$\rho(DF(x^\ast))<1$, then $F$ is a contraction at $x^\ast$ and hence, $x^\ast$
is asymptotically stable.\label{prop:ort}
\end{proposition}

The following technical lemma, due to~\citet{mustafa1994generalized}, is used in constructing the finite learning rate ratio.
\begin{lemma}[{\cite[Lem.~15]{mustafa1994generalized}}]
Let $V,Z\in \mb{R}^{p\times p}$, $W\in \mb{R}^{p\times q}$ and $Y\in \mb{R}^{q\times q}$. If $V$ and $Y-XV^{-1}W$ are non-singular, then
\[\det\left(\bmat{V+Z & W\\ X & Y}\right)=\det(V)\det(Y-XV^{-1}W)\det(I+V^{-1}(I+W(Y-XV^{-1}W)^{-1}XV^{-1})Z)\]
\label{lem:detformula}
\end{lemma}
For completeness (and because there is a typo in the original manuscript), we provide the proof here.
\begin{proof}
 Suppose that $V$ and $Y-XV^{-1}W$ are non-singular so that the  partial Schur decomposition
 \[\bmat{V & W\\ X & Y}=\bmat{V& 0\\
 X & Y-XV^{-1}W}\bmat{I & V^{-1}W\\ 0 & I}\]
 holds, and 
 \begin{equation}\det\left(\bmat{V & W\\ X & Y}\right)=\det(V)\det(Y-XV^{-1}W).\label{eq:detzero}\end{equation}

 Further, 
 \[\bmat{V & W\\ X & Y}^{-1}=\bmat{I & -V^{-1}W\\ 0 & I}\bmat{V^{-1} & 0 \\
 -(Y-XV^{-1}W)^{-1}XV^{-1} & (Y-XV^{-1}W)^{-1}}.\]
 Applying the determinant operator, we have that
 \begin{equation}\det\left(\bmat{V+Z & W\\ X & Y}\right)=\det\left(\bmat{V & W\\ X & Y}\right)\det\left(\bmat{I& 0\\ 0& I}+\bmat{V & W\\ X & Y}^{-1}\bmat{Z& 0\\0 & 0}\right)\label{eq:detone}\end{equation}
 so that
 \begin{align}
    \det\left(\bmat{I& 0\\ 0& I}+\bmat{V & W\\ X & Y}^{-1}\bmat{Z& 0\\0 & 0}\right)&=\det\left(\bmat{V^{-1}(I+W(Y-XV^{-1}W)^{-1}XV^{-1})Z+I & 0\\
    -(Y-XV^{-1}W)V^{-1}Z & I}\right) \\
    &=\det(V^{-1}(I+W(Y-XV^{-1}W)^{-1}XV^{-1})Z+I).\label{eq:dettwo}
 \end{align}
 Combining \eqref{eq:detzero} with \eqref{eq:dettwo} in \eqref{eq:detone} gives exactly the claimed result.
\end{proof}

The following lemma is Theorem 2~\citet[Chap.~13.1]{lancaster1985theory}. We use this lemma several times in the proofs of Theorem~\ref{thm:iffstack} and \ref{prop:instability} so we include it here for ease of reference.
For  a given matrix $A$, $\upsilon_+(A)$, $\upsilon_-(A)$, and $\zeta(A)$ are the number of eigenvalues of the argument that have positive, negative and zero real parts, respectively.
\begin{lemma}
Consider a matrix $A\in \mb{R}^{n\times n}$. 
\begin{enumerate}
    \item[(a)] If $P$ is a symmetric matrix such that $AP+PA^\top=Q$ where $Q=Q^\top>0$, then $P$ is nonsingular and $P$ and $A$ have the same inertia---i.e.,
\begin{equation}\upsilon_+(A)=\upsilon_+(P), \ \upsilon_-(A)=\upsilon_-(P), \ \zeta(A)=\zeta(P).
\label{eq:inertialemma}
\end{equation}
\item[(b)]  On the other hand, if $\zeta(A)=0$, then there exists a matrix $P=P^\top$ and a matrix $Q=Q^\top>0$ such that $AP+PA^\top=Q$ and $P$ and $A$ have the same inertia (i.e., \eqref{eq:inertialemma} holds).
\end{enumerate}
\label{lem:landtis}
\end{lemma}

\paragraph{Numerical and Quadratic Numerical Range.} 
The numerical range and quadratic numerical range of a block operator matrix are particularly useful for proving results about the spectrum of a block operator matrix as they are supersets of the spectrum~\cite{tretter2008spectral}.
Given a matrix $A\in \mb{R}^{\m\times \m}$, the numerical range is defined by
\[\W(A)=\{z\in \mb{C}^{\m}:\ \langle A z,z\rangle, \ \|z\|=1\},\]
and is a convex subset of $\mb{C}$. Define spaces $W_i=\{z\in \mb{C}^{\m_i}:\ \|z\|=1\}$ for each $i\in\{1,2\}$. Consider a block operator  
\[A=\bmat{A_{11} & A_{12}\\ A_{21} & A_{22}},\]
where $A_{ii}\in \mb{R}^{\m_i\times \m_i}$ and $A_{ij}\in \mb{R}^{\m_i\times \m_j}$ for each $i,j\in \{1,2\}$. Given 
$v\in W_1$ and $w\in W_2$, 
let $A^{v,w}\in \mb{C}^{2\times 2}$ be defined by 
\[A^{v,w}=\bmat{\langle A_{11}v,v\rangle & \langle A_{12}w,v\rangle\\ \langle A_{21}v,w\rangle & \langle A_{22}w,w\rangle}.\]
The quadratic numerical range of $A$ is defined by
\[\W^2(A)=\bigcup_{v\in W_1, w\in W_2 }\spec(A^{v,w})\]
 where $\spec(\cdot)$ denotes the spectrum of its argument.

The quadratic numerical range 
can be described as the set of solutions of 
the characteristic polynomial
\begin{equation}
  \lambda^2 
  - \lambda(\langle A_{11} v,v\rangle  + \langle A_{22} w,w\rangle ) \\
  + \langle A_{11} v,v\rangle\langle A_{22} w,w\rangle  
  - \langle A_{12} v,w\rangle \langle A_{21} w,v\rangle  = 0
\end{equation}
for $v\in W_1$ and $w\in W_2$. We use the notation $\langle A v,w\rangle=\bar{v}^\top A w$ to denote the inner product. Note that $\W^2(A)$ is a (potentially non-convex) subset of $\W(A)$ and contains $\spec(A)$.

\section{Proof of Proposition~\ref{lem:zsspecbdd}}
\label{sec:stabilityofDNE}

 Before proving this result, we note that the result has already been shown in the literature by \citet{jin2019local}. We included the result primarily because the proof approach is different and the tools we use (in particular, the quadratic numerical range) have not been utilized before in this type of analysis. Hence, we view the proof technique itself as a contribution.
 
\paragraph{Proof of Proposition~\ref{lem:zsspecbdd}} We leverage the quadratic numerical range to show that $\spec(J_\tau(x^\ast))\subset \mb{C}_+^\circ$  for any $\tau\in(0,\infty)$. Indeed, the quadratic numerical range of a block operator matrix contains its spectrum~\cite{tretter2008spectral}.

Recall that \[\W^2(J_\tau(x^\ast))=\bigcup_{v\in W_1,w\in W_2}\spec(J_\tau^{v,w}(x^\ast))\]
where
\[J_\tau^{v,w}(x^\ast)=\bmat{\la D_1^2f(x^\ast)v,v\ra & \la D_{12}f(x^\ast)w,v\ra\\
\la -\tau D_{12}^\top f(x^\ast)v,w\ra & \la -\tau D_2^2f(x^\ast)w,w\ra}\]
and $W_i=\{z\in \mb{C}^{\m_i}:\ \|z\|=1\}$ for each $i=1,2$. Fix $v\in W_1$ and $w\in W_2$ and consider
\[J_\tau^{v,w}(x^\ast)=\bmat{a & b\\
-\tau \bar{b} & \tau d}\]
Then, the elements of $\W^2(J_\tau(x^\ast))$ are of the form
\[\lambda_\tau=\tfrac{1}{2}(a+\tau d)\pm \tfrac{1}{2}\sqrt{( a-\tau d)^2-4\tau |b|^2}\]
where $a=\la D_{1}^2f(x^\ast)v,v\ra$, $b=\la D_{12}f(x^\ast)w,v\ra$ and $d=\la -D_{2}^2f(x^\ast)w,w\ra$ for vectors $v\in W_1$ and $w\in W_2$.

 We claim that for any $\tau\in (0,\infty)$, $\text{Re}(\lambda_\tau)>0$ for all $a$,$b$, and $d$ where $a>0$ and $d>0$ since $x^\ast$ is a differential Nash equilibrium. 
 
  \ifoldproof
 We know that for a scalar $2\times 2$ matrices $A$, the map
 \[\nu(A)=\det(A)\tr(A)\]
 guards the space of Hurwitz stable matrices. Hence, the eigenvalues of $J_\tau^{v,w}(x^\ast)$ lie in the open right-half complex plane if $\det(J^{v,w}_\tau(x^\ast))>0$ and $\tr(J^{v,w}_\tau(x^\ast))>0$.
 \fi

 Indeed, we argue this by considering the two possible cases: (1) $(a -\tau d)^2\leq 4|b|^2\tau $ or (2) $(a -\tau d)^2> 4\tau |b|^2 $.

\begin{itemize}
\item \textbf{Case 1}: 
Suppose $\tau \in(0,\infty)$ is such that $(a -\tau d)^2\leq 4|b|^2\tau $. Then,
$\text{Re}(\lambda_\tau)=\tfrac{1}{2}(a+\tau d)>0$ trivially since $a+d>0$. 

\item \textbf{Case 2}: Suppose $\tau\in(0,\infty)$ is such that $(a -\tau d)^2> 4\tau |b|^2$. In this case,
we want to ensure
that 
\[\text{Re}(\lambda_\tau)>\tfrac{1}{2}(a+\tau d)- \tfrac{1}{2}\sqrt{(a-\tau d)^2-4\tau |b|^2}>0.\]
 The last inequality is equivalent to $-ad<|b|^2$. Indeed,
\[(a+\tau d)^2>(a-\tau d)^2-4\tau |b|^2 \ \Longleftrightarrow\ 4\tau ad>-4\tau |b|^2 \ \Longleftrightarrow \ -ad<|b|^2.\] 
Moreover, $-ad<|b|^2$ holds for any pair of vectors $(v,w)$ such that $v\in W_1$ and $w\in W_2$ since $a>0$ and $d>0$.
\end{itemize}
Hence, for any $\tau\in (0,\infty)$, $\spec(J_\tau(x^\ast))\subset \mb{C}_+^\circ$ since the spectrum of an operator is contained in its quadratic numerical range and the above argument shows that $\W^2(J_\tau(x^\ast))\subset \mb{C}_+^\circ$.

\section{Proof of Lemma~\ref{lem:convergencerate-asymptotic} and Lemma~\ref{lem:convergencerate}}
\label{app_sec:convergencerate}
In this appendix section, we prove Lemma~\ref{lem:convergencerate-asymptotic} and Lemma~\ref{lem:convergencerate} from Section~\ref{sec:convergencerates}. 
We note that the proof of Lemma~\ref{lem:convergencerate} starts where the proof of Lemma~\ref{lem:convergencerate-asymptotic} leaves off.  
\subsection{Proof of Lemma~\ref{lem:convergencerate-asymptotic}}
Suppose that $x^\ast$ is a differential Stackelberg or Nash equilibrium and that $0<\tau<\infty$ is such that $\spec(-J_\tau(x^\ast))\subset\mb{C}_-^\circ$.
For the discrete time dynamical system $x_{k+1}=x_k-\gamma_1\Lambda_{\tau} \g(x_k)$, it is well known that if $\gamma_1$ is chosen such that $\rho(I-\gamma_1J_\tau(x^\ast))<1$, then $x_{k}$ locally (exponentially) converges to $x^\ast$ \cite{ortega:1970aa}. With this in mind, we formulate an optimization problem to find the upper bound $\gamma$ on the learning rate $\gamma_1$ such that for all $\gamma_1\in (0,\gamma)$, the spectral radius of the local linearization of the discrete time map is a contraction which is precisely $\rho(I-\gamma_1J_\tau(x^\ast))<1$. The optimization problem is given by \begin{equation}\gamma=\min_{\gamma>0}\left\{\gamma:\ \max_{\lambda\in \spec(J_\tau(x^\ast))}|1-\gamma \lambda|\leq 1\right\}.\label{eq:gammaopt-app}
\end{equation}

The intuition is as follows. The inner maximization problem is over a finite set $\spec(J_\tau(x^\ast))=\{\lambda_1, \ldots, \lambda_\m\}$ where $J_\tau(x^\ast)\in \mb{R}^{\m\times \m}$. As $\gamma$ increases away from zero, each $|1-\gamma\lambda_i|$ shrinks in magnitude. The last $\lambda_i$ such that $1-\gamma\lambda_i$ hits the boundary of the unit circle in the complex plane (i.e., $|1-\gamma\lambda_i|=1$) gives us the optimal value of $\gamma$ and the element of $\spec(J_\tau(x^\ast))$ that achieves it. 
Examining the constraint, we have that for each $\lambda_i$, 
$\gamma(\gamma|\lambda_i|^2-2\Re(\lambda_i))\leq 0$
for any $\gamma>0$. As noted this constraint will be tight for one of the $\lambda$,
in which case
$\gamma=2\Re({\lambda})/|{\lambda}|^2$
since $\gamma>0$. Hence, by selecting 
$\gamma=\min_{\lambda\in \spec(J_\tau(x^\ast))} 2\Re(\lambda)/|\lambda|^2$,
we have
 that $|1-\gamma_1 \lambda|< 1$ for all $\lambda\in \spec(J_\tau(x^\ast))$ and any $\gamma_1\in (0,\gamma)$.

To see this is the case, let $\gamma=\min_{\lambda\in \spec(J_\tau(x^\ast))} 2\Re(\lambda)/|\lambda|^2$ and
${\lambdam}=\arg\min_{\lambda\in \spec(J_\tau)} 2\Re(\lambda)/|\lambda|^2$.
Using the expression for $\gamma$, we have that
\begin{align*}
    1-2\gamma \Re(\lambda)+\gamma^2(\Re(\lambda)^2+\Im(\lambda)^2)&=1-2\frac{2\Re({\lambdam})}{|{\lambdam}|^2}\Re(\lambda)+\left(\frac{2\Re({\lambdam})}{|{\lambdam}|^2}\right)^2|\lambda|^2.
    \end{align*}
Now, using the fact that $\Re(\lambda)/|\lambda|^2>\Re({\lambdam})/|{\lambdam}|^2$, we have 
\begin{align*}
  1-4\frac{\Re({\lambdam})}{|{\lambdam}|^2}\Re(\lambda)+\left(\frac{2\Re({\lambdam})}{|{\lambdam}|^2}\right)^2|\lambda|^2
    &\leq 1-2\frac{2\Re({\lambdam})}{|{\lambdam}|^2}\Re(\lambda)+\left(\frac{2\Re({\lambdam})}{|{\lambdam}|^2}\right)^2\frac{|{\lambdam}|^2\Re(\lambda)}{\Re({\lambdam})}\\
    &=1-4\frac{\Re({\lambdam})}{|{\lambdam}|^2}\Re(\lambda)+4\frac{\Re({\lambdam})}{|{\lambdam}|^2}\Re(\lambda)\\
    &=1
\end{align*}
as claimed. From this argument, it is clear that for any $\gamma_1\in (0,\gamma)$, $|1-\gamma_1\lambda|<1$ for all $\lambda\in \spec(J_\tau(x^\ast))$.

Now, consider any $\alpha\in (0,\gamma)$ and let 
$\beta=(2\Re({\lambdam})-\alpha|{\lambdam}|^2)^{-1}$.
Observe that  $\gamma_1=\gamma-\alpha$ so that $\gamma_1\in (0,\gamma)$. Hence, 
\begin{align*}
    |1-(\gamma-\alpha){\lambdam}|^2&= \left(1-\left(\frac{2\Re({\lambdam})}{|{\lambdam}|^2}-\alpha\right)\Re({\lambdam})\right)^2+\left(\frac{2\Re({\lambdam})}{|{\lambdam}|^2}-\alpha\right)^2\Im({\lambdam})^2\\
    &=1-4\frac{\Re({\lambdam})^2}{|{\lambdam}|^2}+2\alpha \Re({\lambdam})+4\frac{\Re({\lambdam})^2}{|{\lambdam}|^2}-4\alpha\Re({\lambdam})+\alpha^2|{\lambdam}|^2\\
    &=1-2\alpha \Re({\lambdam})+\alpha^2|{\lambdam}|^2\\
    &=1-\frac{\alpha}{\beta}
\end{align*}
so that
\[\rho(I-\gamma_1J_\tau(x^\ast))<\left(1-\frac{\alpha}{\beta}\right)^{1/2}.\]
Hence, the $\rho(I-\gamma_1J_\tau(x^\ast))<1$ so that an application of Proposition~\ref{prop:ort} gives us the desired result.

\subsection{Proof of Lemma~\ref{lem:convergencerate}}
To prove this lemma, we build directly on the conclusion of the proof of Lemma~\ref{lem:convergencerate-asymptotic}. Indeed, since
\[\rho(I-\gamma_1J_\tau(x^\ast))<\left(1-\frac{\alpha}{\beta}\right)^{1/2},\]
given $\varepsilon=\tfrac{\alpha}{4\beta}>0$ there exists a norm $\|\cdot\|$ (cf.~Lemma~5.6.10 in \citet{horn1985matrix})\footnote{The norm that exists can easily be constructed as essentially a weighted induced $1$-norm. Note that the norm construction is not unique. The proof in \citet{horn1985matrix} is by construction and the construction of this norm can be found there.} such that
\[\|I-\gamma_1J_\tau(x^\ast)\|\leq \left(1-\frac{\alpha}{\beta}\right)^{1/2}+\frac{\alpha}{4\beta}\leq \left(1-\frac{\alpha}{2\beta}\right)^{1/2}\]
where the last inequality holds by Lemma~\ref{lem:bdd}. 
Taking the Taylor expansion of $I-\gamma_1g_\tau(x)$ around $x^\ast$, we have
\begin{equation*}
    I-\gamma_1g_\tau(x)=(I-\gamma_1g_\tau(x^\ast))+(I-\gamma_1J_{\tau}(x^\ast))(x-x^\ast)+R_2(x-x^\ast)
\end{equation*}
where $R_2(x-x^\ast)$ is the remainder term satisfying $R_2(x-x^\ast)=o(\|x-x^\ast\|)$ as $x\rar x^\ast$.\footnote{The notation $R_2(x-x^\ast)=o(\|x-x^\ast\|)$ as $x\rar x^\ast$ means $\lim_{x\rar x^\ast}\|R_2(x-x^\ast)\|/\|x-x^\ast\|=0$.} This implies that there is a $\delta>0$ such that $\|R_2(x-x^\ast)\|\leq \frac{\alpha}{8\beta}\|x-x^\ast\|$ whenever $\|x-x^\ast\|<\delta$. Hence,
\begin{align*}
    \|I-\gamma_1g_{\tau}(x)-(I-\gamma_1g_{\tau}(x^\ast))\|&\leq\left(\|I-\gamma_1J_{\tau}(x^\ast)\|+ \frac{\alpha}{4\beta}\right)\|x-x^\ast\|\\
    &\leq \left( \left(1-\frac{\alpha}{2\beta}\right)^{1/2}+\frac{\alpha}{8\beta}\right)\|x-x^\ast\|\\
    &\leq \left(1-\frac{\alpha}{4\beta}\right)^{1/2}\|x-x^\ast\|
\end{align*}
where the last inequality holds again by Lemma~\ref{lem:bdd}.
Hence,
\begin{equation}\|x_k-x^\ast\|\leq \left(1-\frac{\alpha}{4\beta}\right)^{k/2}\|x_0-x_\ast\|\label{eq:iterationcomplexity}\end{equation}
whenever $\|x_0-x^\ast\|<\delta$ which verifies the claimed convergence rate.

\section{Proof of Corollary~\ref{cor:finitetimebound}}
\label{app_sec:finiteconvergencerate}
Let $\|\cdot\|$ be the norm that exists (via construction a la \citet[Lem.~5.6.10]{horn1985matrix}) in the proof of Lemma~\ref{lem:convergencerate} which is given in Appendix~\ref{app_sec:convergencerate}. Following standard arguments,   \eqref{eq:iterationcomplexity} in the proof of Lemma~\ref{lem:convergencerate} implies a finite time convergence guarantee.  Indeed, let $\vep>0$ be given. Since $0<\tfrac{\alpha}{4\beta}<1$ we have that $(1-\alpha/(4\beta))^k<\exp(-k\alpha/(4\beta))$. Hence, 
 \[\|x_k-x^\ast\|\leq \exp(-k\alpha/(4\beta))\|x_0-x^\ast\|.\]
In turn, this implies that $x_k\in B_\vep(x^\ast)$, meaning that $x_k$ is a $\vep$-differential Stackelberg equilibrium for all $k\geq\lceil \frac{4\beta}{\alpha}\log(\|x_0-x^\ast\|/\vep)\rceil$ whenever $\|x_0-x^\ast\|<\delta$. 

Now, given that $f_i\in C^r(X,\mb{R})$ for $r\geq 2$,  $I-\gamma_1J_{\tau}(x)$ is locally Lipschitz with constant $L$ so that we can find an explicit expression for $\delta$ in terms of $L$. Indeed, 
recall that $R_2(x-x^\ast)=o(\|x-x^\ast\|)$ as $x\rar x^\ast$ which means $\lim_{x\rar x^\ast}\|R_2(x-x^\ast)\|/\|x-x^\ast\|=0$ so that 
\[\|R_2(x-x^\ast)\|\leq \int_0^1\|I-\gamma_1J_{\tau}(x^\ast+\eta(x-x^*))-(I-\gamma_1J_{\tau}(x^*))\|\|x-x^*\|\ d\eta\leq \frac{L}{2}\|x-x^*\|^2\]
Observing that
\[\|R_2(x-x^\ast)\|\leq \frac{L}{2}\|x-x^\ast\|^2=\frac{L}{2}\|x-x^\ast\|\|x-x^\ast\|,\]
 we have that the $\delta>0$ such that $\|R_2(x-x^\ast)\|\leq \alpha/(8\beta)\|x-x^\ast\|$ is
$\delta={\alpha/(4L\beta)}$.

\section{Proof of Theorem~\ref{thm:iffstack} and Corollary~\ref{cor:asymptoticiffstack}}
\label{app_sec:iffstack}
To prove Theorem~\ref{thm:iffstack} and Corollary~\ref{cor:asymptoticiffstack}, we introduce some techniques that are arguably new to the machine learning and artificial intelligence communities. The first is the notion of a guard map. A guard map can be used to provide a certificate of a particular behavior for a dynamical system as a parameter(s) varies. 
A critical point of a dynamical systems is known to be stable if the spectrum of the Jacobian at the critical point lies in the open left-half complex plane, denoted $\mb{C}_-^\circ$. Hence, we construct a guard map as a function of $\tau$ and show that it guards $\mb{C}_-^\circ$. Specifically we show that the existence of a  $\tau^\ast\in(0,\infty)$ such that $\nu(\tau^\ast)=0$ and $\nu(\tau)\neq 0$ for all $\tau \in (\tau^\ast,\infty)$ is equivalent to $\schurtt_1(J(x^\ast))>0$ and $-D_2^2f(x^\ast)>0$
where 
\[\schurtt_1(J(x^\ast))=\schurtt_1(J_\tau(x^\ast))=D_1^2f(x^\ast)-D_{12}f(x^\ast)(D_2^2f(x^\ast))^{-1}D_{21}f(x^\ast).\]
Towards this end, we need to introduced some notation as well as formal definitions for important concepts such as the guard map.

\subsection{Notation and Preliminaries}  Given a matrix $A\in \mb{R}^{\m_1\times \m_2}$, let $\vec(A)\in \mb{R}^{\m_1\m_2}$ be the vectorization of $A$. We use the convention that rows are transposed and stacked in order. That is,
\[\vec:\bmat{ \text{---}\ a_1 \ \text{---}\\ \vdots\\ \text{---}\ a_{\m_1} \ \text{---}}\mapsto \bmat{a_1^\top\\ \vdots\\ a_{\m_1}^\top}\]
Let $\otimes$ and $\oplus$ denote the Kronecker product and Kronecker sum respectively. Recall that $A\oplus B=A\otimes B+B\otimes A$. A less common operator, we  define $\boxplus$ as an operator that generates an $\tfrac{1}{2}\m(\m+1)\times \tfrac{1}{2}\m(\m+1)$ matrix from a matrix $A\in \mb{R}^{\m\times \m}$ such that
\[A\boxplus A=H_\m^{+} (A\oplus A)H_\m\]
where $H_\m^{+}=(H_\m^\top H_\m)^{-1}H_\m^\top$ is the (left) pseudo-inverse of $H_\m$, a full column rank dupplication matrix. A duplication matrix $H_\m\in \mb{R}^{\m^2\times \m(\m+1)/2}$ is a clever linear algebra tool for mapping a $\tfrac{\m}{2}(\m+1)$ vector to a $\m^2$ vector generated by applying $\vec(\cdot)$ to a symmetric matrix  and it is designed to respect the vectorization map $\vec(\cdot)$. In particular, if $\vech(X)$ is the half-vectorization map of any symmetric matrix $X\in \mb{R}^{\m\times \m}$, then $\vec(X)=H_\m\vech(X)$ and $\vech(X)=H_\m^+\vec(X)$.

Given a square matrix $A$, let $\lambda_{\max}^+(A)$ be the  largest positive real eigenvalue of $A$ and if $A$ does not have a positive real eigenvalue then it is zero.

\paragraph{Guardian maps.} The use of guardian maps for studying stability of parameterized families of dynamical systems was arguably introduced by \citet{saydy1990guardian}. Guardian or guard maps act as a certificate for a performance criteria such as stability. 

Formally, let $\mc{X}$ be the set of all $\m\times \m$ real matrices or the set of all polynomials of degree $n$ with real coefficients.  Consider $\mc{S}$ an open subset of $\mc{X}$ with closure $\bar{\mc{S}}$ and boundary $\partial \mc{S}$.
\begin{definition}\label{def:guardmap}
The map $\nu: \mc{X}\rar \mb{C}$ is said to be a guardian map for $\mc{S}$ if for all $x\in \bar{\mc{S}}$, 
$\nu(x)=0 \ \Longleftrightarrow\ x\in \partial \mc{S}$.
\end{definition}
Consider an open subset $\Omega$ of the complex plane that is symmetric with respect to the real axis. Then, elements of $\mc{S}(\Omega)=\{A\in \mb{R}^{n\times n}:\ \spec(A)\subset \Omega\}$ are said to be stable relative to $\Omega$.

The following result gives a necessary and sufficient condition for stability of parameterized families of matrices relative to some open subset of the complex plane. 
\begin{proposition}[Proposition 1~\cite{saydy1990guardian}; Theorem 2~\cite{abed1990generalized}]
Consider $U$ to be a pathwise connected subset of $\mb{R}$ and $A(\tau)\in \mb{R}^{n\times n}$ a matrix which depends continuously on $\tau$. Let $\mc{S}(\Omega)$ be guarded by the map $\nu$. The family $\{A(\tau):\ \tau \in U\}$ is stable relative to $\Omega$ if and only if $(i)$ it is nominally stable---i.e., $A(\tau_0)\in \mc{S}(\Omega)$ for some $\tau_0\in U$---and $(ii)$ $\nu(A(\tau))\neq 0$ for all $\tau\in U$.
\label{prop:guardmapbdd}
\end{proposition}

In proving Theorem~\ref{thm:iffstack}, we define a guard map for the space of $\m\times \m$ Hurwitz stable matrices which is denoted by $\mc{S}(\mb{C}_-^\circ)$.
\begin{lemma}
The map $\nu: A \mapsto \det(A\boxplus A)$ guards the set of $\m\times \m$ Hurwitz stable matrices $\mc{S}(\mb{C}_-^\circ)$.
\label{lem:boxplus}
\end{lemma}
\begin{proof}
This follows from the following observation: for $A\in \mb{R}^{\m\times \m}$,
\[\vech(AX+XA^\top)=H_\m^+\vec(AX+XA^\top)=H_\m^+(A\oplus A)\vec(X)=H_\m^+(A\oplus A)H_\m\vech(X)\]
from which it can be shown that the eigenvalues of $A\boxplus A$ are $\lambda_i+\lambda_j$ for $1\leq j\leq i\leq \m_1+\m_2$ where $\lambda_i$ for $i=1,\ldots, \m$ are the eigenvalues of $A$.

Indeed, let $S$ be a non-singular matrix such that $S^{-1}AS=M$ where $M$ is upper triangular with $\lambda_1,\ldots,\lambda_\m$ on its diagonal. Observe that 
for any $\m\times \m$ matrix $P$, $H_\m H_\m^+(P\otimes P)H_\m=(P\otimes P)H_\m$ and $H_\m^+(P\otimes P)H_\m H_\m^+=H_\m^+(P\otimes P)$.
Hence, using properties of the Kronecker product (namely, that $(A_1\otimes A_2)(B_1\otimes B_2)=(A_1B_1\otimes A_2B_2)$), we have that
\[H_\m^+(S^{-1}\otimes S^{-1})H_\m H_\m^+(I\otimes A+A\otimes I)H_\m H_\m^+(S\otimes S)H_\m=H_\m^+(I\otimes M+M\otimes I)H_\m\]
so that the spectrum of $H_\m^+(I\otimes A+A\otimes I)H_\m$ and $H_\m^+(I\otimes M+M\otimes I)H_\m$ coincide.  Now, since $M$ is upper triangular,
$H_\m^+(I\otimes M+M\otimes I)H_\m$ is upper triangular with diagonal elements $\lambda_i+\lambda_j$ ($1\leq j\leq i\leq \m$) which can be verified by direct computation and using the definition of $H_\m$. This implies that
$\lambda_i+\lambda_j$ ($1\leq j\leq i\leq \m$) are exactly the eigenvalues of $H_\m^+(I\otimes A+A\otimes I)H_\m$.
\end{proof}

We note that there are several other guard maps for the space of Hurwtiz stable matrices including $\nu:A\mapsto \det(A\oplus A)$. 
To give some intuition for this map, it is fairly straightforward to see that the Kronecker sum $A\oplus A=A\otimes I+I \otimes A$ has spectrum $\{\lambda_j+\lambda_i\}$ where $\lambda_i,\lambda_j\in \spec(A)$. The operator $A\boxplus A$ is simply a more computationally efficient expression of $A\oplus A$, and as such the eigenvalues of $A\boxplus A$ are those of $A\oplus A$ removing redundancies. We use $A\boxplus A$ specifically because of its computational advantages in computing $\tau^\ast$.

\subsection{Proof of Theorem~\ref{thm:iffstack}}

We first prove that if $x^\ast$ is a differential Stackelberg equilibrium (i.e., $\schurtt_1(J_\tau(x^\ast))>0$ and $-D_2^2f(x^\ast)>0$), then there exists a finite $\tau^\ast\in(0,\infty)$ such that for all $\tau\in (\tau^\ast,\infty)$, $x^\ast$ is exponentially stable for $\dot{x}=-\Lambda_{\tau} g(x)$ (i.e., $\spec(-J_\tau(x^\ast))\subset \mb{C}_-^\circ$). Towards this end, we construct a guard map for the space of $\m\times\m$ Hurwtiz stable matrices and explicitly construct the $\tau^\ast$ using it. 

Then we prove the other direction. That is, if there exists a finite $\tau^\ast\in(0,\infty)$ such that for all $\tau\in (\tau^\ast,\infty)$, $x^\ast$ is exponentially stable for $\dot{x}=-\Lambda_{\tau} g(x)$, then $x^\ast$ is a differential Stackelberg equilibrium. We prove this by contradiction.

\subsubsection{Proof that if $x^\ast$ is a differential Stackelberg then finite $\tau^\ast$ exists}
Towards this end, 
for a critical point $x^\ast$, let
\[-J_\tau(x^\ast)=\bmat{-D_1^2f(x^\ast) & -D_{12}f(x^\ast)\\ \tau D_{12}^\top f(x^\ast) & \tau D_2^2f(x^\ast)}=\bmat{A_{11} & A_{12} \\
 -\tau A_{12}^\top & \tau A_{22}}\]
 and define
 \[\schur_1=\schurtt_1(-J_\tau(x^\ast))=A_{11}-A_{12}A_{22}^{-1}A_{12}^\top.\]
 Note that this is equivalent to the first Schur complement of $-J(x^\ast)$ (i.e., when $\tau=1$) since the $\tau$ and $\tau^{-1}$ cancel, and by assumption the first Schur complement of $-J(x^\ast)$ is positive definite. 
Suppose that $x^\ast$ is a differential Stackelberg equilibrium so that $-\schur_1>0$ and $-A_{22}>0$.

\paragraph{Polynomial guard map with family of matrices parameterized by $\tau$.} By Lemma~\ref{lem:boxplus}, $\nu:A\mapsto \det(A\boxplus A)$ is a guard map for $\mc{S}(\mb{C}_-^\circ)$. 
Indeed,  using the fact that the determinant is the product of the eigenvalues of a matrix and the fact that $\spec(A\boxplus A)=\{\lambda_i+\lambda_j, 1\leq i\leq j\leq \m, \lambda_i,\lambda_j\in \spec(A)\}$, we have that
\[\det(A\boxplus A)=\prod_{1\leq j\leq i\leq \m}(\lambda_i+\lambda_j)=\prod_{1\leq i\leq \m}2\Re(\lambda_i)(4\Re^2(\lambda_i)+4\Im^2(\lambda_i))\prod_{\stackrel{1<i<j<\m:}{\lambda_i\neq \bar{\lambda}_j}}(\lambda_i+\lambda_j).\]
Hence, consider $\bar{\mc{S}}(\mb{C}_-^\circ)$, $\det(A\boxplus A)=0$ if and only if $A\boxplus A$ is singular if and only if $A$ has a purely imaginary eigenvalue---that is, if and only if $A\in \partial \mc{S}(\mb{C}_-^\circ)$.\footnote{Indeed, this holds since the only scenarios in which $\det(A\boxplus A)=0$ are such that the eigenvalues of $A$ do not lie in $\bar{\mc{S}}(\mb{C}_-^\circ)$.}
Now, consider the parameterized family of matrices $-J_\tau(x^\ast)$, parameterized by $\tau$. By an abuse of notation, let $\nu(\tau)=\det(-J_\tau(x^\ast)\boxplus -J_\tau(x^\ast))$. If we consider the subset of this family of matrices that lies in $\mc{S}(\mb{C}_-^\circ)$ (this subset could a priori be empty thought we show it is not), then for any $\tau$ such that $-J_\tau(x^\ast)$ is in this subset, we have that $\nu(\tau)=0$ if and only if $-J_\tau(x^\ast)\boxplus (-J_\tau(x^\ast))$ is singular if and only if $-J_\tau(x^\ast)\in \partial\mc{S}(\mb{C}_-^\circ)$. Hence, $\nu(\tau)=\det(-J_\tau(x^\ast)\boxplus -J_\tau(x^\ast))$ guards $\mc{S}(\mb{C}_-^\circ)$.

In particular, if we envision $-J_\tau(x^\ast)$ as the input to $\nu:A\mapsto \det(A\boxplus A)$ and simply vary $\tau$ (holding all the entries of $-J_\tau(x^\ast)$ otherwise fixed), then $\nu:\tau\mapsto \det(-J_\tau(x^\ast)\boxplus (-J_\tau(x^\ast)))$ can be thought of simply as a function of $\tau$ which guards the set of Hurwitz stable matrices via the reasoning describe above. Indeed,
by slightly overloading the notation for $\nu$, 
\[\nu(\tau):=\nu_0+\nu_1\tau+\cdots +\nu_{p-1}\tau^{p-1}+\nu_p\tau^p=\nu(-J_\tau(x^\ast))\]
Hence, for intuition, observe that as $\tau$ decreases (towards zero) stability is first lost when at least one eigenvalue of $-J_\tau(x^\ast)$ reaches the imaginary axis, at which point $\nu(\tau)=0$. 

There are two cases to consider: 
\begin{description}[itemsep=0pt,topsep=2pt]
\item[\textbf{Case 1}:] \emph{$\nu(\tau)$ is an identically zero polynomial.} In this case, $-J_\tau(x^\ast)$ is in the interior of the complement of the set of Hurwitz stable matrices for all values of $\tau>0$---that is, $-J_\tau(x^\ast)\in \inter(\Hstablecomp)$ for all $\tau\in \mb{R}_+=(0,\infty)$. 
\item[\textbf{Case 2}:] \emph{$\nu(\tau)$ is not an identically zero polynomial.} In this case, $\nu(\tau)$ has finitely many zeros. If $\nu(\tau)$ has no positive real roots, then as $\tau$ varies in ${\mb{R}}_+$, $-J_\tau(x^\ast)$ does not cross $\partial \mc{S}(\mb{C}_-^\circ$---i.e., the boundary of the space of $\m\times \m$ Hurwitz stable matrices. Hence, $\{-J_\tau(x^\ast):\ \tau \in {\mb{R}}_+\}\subset \Hstable$ or $\{-J_\tau(x^\ast):\ \tau \in {\mb{R}}_+\}\subset \inter(\Hstablecomp)$. It suffices to check $-J_\tau(x^\ast)\in \Hstable$ or $-J_\tau(x^\ast)\in \inter(\Hstablecomp)$ for an arbitrary $\tau\in {\mb{R}}_+$. 

On the other hand, if $\nu(\tau)$ has $\ell\geq 1$ real positive zeros, say $0<\tau_1<\cdots<\tau_\ell=\tau^\ast$, then by Proposition~\ref{prop:guardmapbdd}, $-J_\tau(x^\ast)\in \mc{S}(\mb{C}_-^\circ)$ for all $\tau>\tau^\ast$ if and only if $-J_\tau(x^\ast)\in \mc{S}(\mb{C}_-^\circ)$ for arbitrarily chosen $\tau>\tau^\ast$. We choose the largest positive root $\tau_\ell$ because we are guaranteed that $\nu(\tau)$ stops changing sign for $\tau>\tau^\ast$.   Further, the largest neighborhood in $\mb{R}_+$ for which $-J_\tau(x^\ast)\in \mc{S}(\mb{C}_-^\circ)$ is $(\tau_\ell, \infty)$. 
\end{description}
Recall that we have assumed that $x^\ast$ is a differential Stackelberg equilibrium (i.e.,
$\schur_1>0$ and $-A_{22}>0$). We will show next (by way of explicit construction of $\tau^\ast$) that we are always in case 2.

\paragraph{Construction of $\tau^\ast$.} We note that there are more elegant, simpler constructions, but to our knowledge this construction gives the tightest bound on the range of $\tau$ for which $-J_\tau(x^\ast)$ is guarnateed to be Hurwitz stable.
Recall that
\[-J_\tau(x^\ast)=\bmat{-D_1^2f(x^\ast) & -D_{12}f(x^\ast)\\ \tau D_{12}^\top f(x^\ast) & \tau D_2^2f(x^\ast)}=\bmat{A_{11} & A_{12} \\
 -\tau A_{12}^\top & \tau A_{22}}\]
 and 
 \[\schur_1=A_{11}-A_{12}A_{22}^{-1}A_{12}^\top.\]

Let $I_m$ denote the $m\times m$ identity matrix. 

\begin{claim}
The finite learning rate ratio is  $\tau^\ast=\lambda_{\max}^+(Q)$
where 
\begin{equation}
    Q=-(A_{11}\otimes A_{22}^{-1})
 +2\bmat{ (A_{12}\otimes A_{22}^{-1})H_{\m_2} & (I_{\m_1}\otimes A_{22}^{-1}A_{12}^\top)H_{\m_1}}\bmat{\bar{A}_{22}^{-1} & 0\\ 0 & -\bar{\schur}_1^{-1}}\bmat{H_{\m_2}^+(A_{12}^\top \otimes I_{\m_2})\\ H_{n_1}^+(\schur_1\otimes A_{12}A_{22}^{-1})} 
\end{equation}
with $\bar{A}_{22}=A_{22}\boxplus A_{22}$ and $\bar{\schur}_1=\schur_1\boxplus \schur_1$.
\end{claim}

\begin{proof}
Recall that $\nu(\tau)=\det(-J_\tau(x^\ast) \boxplus (-J_\tau(x^\ast)))$ is a guard map for $\mc{S}(\mb{C}_-^\circ)$.

We apply basic properties of the Kronecker product and sum as well as Schur's determinant formula to obtain a reduced form of the guard map. To this end, we have that
\[-J_\tau(x^\ast)\boxplus (-J_\tau(x^\ast))=\bmat{A_{11} \boxplus A_{11} & 2 H_{\m_1}^+(I_{\m_1}\otimes A_{12}) & 0\\
\tau(I_{\m_1}\otimes (-A_{12}^\top))H_{\m_1} & A_{11}\oplus \tau A_{22} & (A_{12}\otimes I_{\m_2})H_{\m_2}\\
0 
& 2\tau H_{\m_2}^+(-A_{12}^\top \otimes I_{\m_2}) & \tau (A_{22}\boxplus A_{22})}\]
Now, we apply Schur's determinant formula to get that
\begin{equation}
    \nu(\tau)=\tau^{\m_2(\m_2+1)/2}\det(A_{22}\boxplus A_{22})\det\left(\bmat{A_{11}\boxplus A_{11} & 2 H_{\m_1}^+(I_{\m_1}\otimes A_{12})\\
\tau(I_{\m_1}\otimes (-A_{12}^\top))H_{\m_1} & A_{11}\oplus \tau A_{22}+M_1}
\right)
\label{eq:guardmapfirstreduction}
\end{equation}
where  %
\[M_1=-2H_{\m_2}^+(-A_{12}^\top \otimes I_{\m_2})(A_{22}\boxplus A_{22})^{-1}(A_{12}\otimes I_{\m_2})H_{\m_2}\]

From here, we apply Lemma~\ref{lem:detformula} to further reduce the guard map. First, note that
\[A_{11}\oplus \tau A_{22}=A_{11}\otimes I_{\m_2}+I_{\m_1}\otimes \tau A_{22}.\]
Let $V=I_{\m_1}\otimes \tau A_{22}$, $Z=A_{11}\otimes I_{\m_2}+M_1$, $Y=A_{11}\boxplus A_{11}$, $W=-\tau(I_{\m_1}\otimes A_{12}^\top)H_{\m_1}$, and $X=2H_{\m_1}^+(I_{\m_1}\otimes A_{12})$. 
Using the two properties of the Kronecker product  $(B_1\otimes B_2)(B_3\otimes B_4)=(B_1B_3\otimes B_2B_4)$ and $(B_1\otimes B_2)^{-1}=(B_1^{-1}\otimes B_2^{-1})$, we have that
\begin{align}
    Y-XV^{-1}W&=A_{11}\boxplus A_{11}+2H_{\m_1}^+(I_{\m_1}\otimes A_{12})(I_{\m_1}\otimes  A_{22})^{-1}(I_{\m_1}\otimes A_{12}^\top)H_{\m_1}\\
    &=A_{11}\boxplus A_{11}+2H_{\m_1}^+(I_{\m_1}\otimes A_{12}A_{22}^{-1}A_{12}^\top)H_{\m_1}\\
    &=A_{11}\boxplus A_{11}+H_{\m_1}^+((I_{\m_1}\otimes A_{12}A_{22}^{-1}A_{12}^\top)+(A_{12}A_{22}^{-1}A_{12}^\top\otimes I_{\m_1}))H_{\m_1}\label{eq:symmetrykron}\\
    &=S_1\boxplus S_1
\end{align}
where \eqref{eq:symmetrykron} holds since
$H_{\m_1}^+(I_{\m_1}\otimes A_{12}A_{22}^{-1}A_{12}^\top)H_{\m_1}=H_{\m_1}^+(A_{12}A_{22}^{-1}A_{12}^\top\otimes I_{\m_1})H_{\m_1}$.
Now, define $V^{-1}+V^{-1}W(Y-XV^{-1}W)^{-1}XV^{-1}=\tau^{-1}M_2$ where
\[M_2=I_{\m_1}\otimes A_{22}^{-1}-2(I_{\m_1}\otimes A_{22}^{-1}A_{12}^\top)H_{\m_1}(\schur_1\boxplus \schur_1)^{-1}H_{\m_1}^+(I_{\m_1}\otimes A_{12}A_{22}^{-1})\]
so that applying Lemma~\ref{lem:detformula} 
we have
\begin{equation}
        \nu(\tau)=\tau^{\m_2(\m_2+1)/2}\det(A_{22}\boxplus A_{22})\det(\schur_1\boxplus \schur_1)\det(I_{\m_1}\otimes A_{22})\det(\tau I_{\m_1\m_2}+ M_2(A_{11}\otimes I_{\m_2}+M_1))
\label{eq:guardmapsecondreduction}
\end{equation}
The assumptions that $\schur_1>0$ and $-A_{22}>0$ together imply that $\det(\schur_1\boxplus \schur_1)\neq 0$ and $\det(I_{\m_1}\otimes A_{22})\neq0$. Hence, $\nu(\tau)=0$ if and only if $\det(\tau I_{\m_1\m_2}+ M_2(A_{11}\otimes I_{\m_2}+M_1))=0$ since $0<\tau <\infty$. The determinant expression is exactly an eigenvalue problem.  

Since by assumption the Schur complement of $J(x^\ast)$ and the individual Hessian $-D_2^2f(x^\ast)$ are positive definite (i.e., $x^\ast$ is a differential Stackelberg equilibrium), 
Thus, the largest positive real root of $\nu(\tau)=0$ is 
\[\tau^\ast =\lambda_{\max}^+(-M_2(A_{11}\otimes I_{\m_2}+M_1))\]
where $\lambda_{\max}^+(\cdot)$ is the largest positive real eigenvalue of its argument if one exists and otherwise its zero.
Using properties of the Kronecker product and duplication matrices, it can easily be seen that $Q=-M_2(A_{11}\otimes I_{\m_2}+M_1)$.
\end{proof}
The result of this claim concludes the proof that if $x^\ast$ is a differential Stackelberg, then there exists a finite $\tau^\ast\in [0,\infty)$ such that for all $\tau\in (\tau^\ast,\infty)$, $\spec(-J_\tau(x^\ast))\subset \mb{C}_-^\circ$.

\subsubsection{Proof that existence of finite $\tau^\ast$ implies that $x^\ast$ is a differential Stackelberg}
\label{app_sec:iffstacksuff}
The proof of this direction is argued by contradiction.  
Consider a  critical point $x^\ast$ (i.e., where $g(x^\ast)=0$ such that $-C\equiv -D_2^2f(x^\ast)$ and $S_1\equiv\schurtt_1(J(x^\ast))=D_1^2f(x^\ast)-D_{12}f(x^\ast)(D_2^2f(x^\ast))^{-1}D_{12}^\top f(x^\ast)$ have no zero eigenvalues---that is, $\det(S_1)\neq 0$ and $\det(C)\neq 0$.

Suppose  that there exists a $\tau^\ast\in (0,\infty)$ such that for all $\tau\in(\tau^\ast,\infty)$, $\spec(-J_{\tau}(x^\ast))\subset \mb{C}_-^\circ$, %
 yet $x^\ast$ is not a differential Stackelberg equilibrium. That is, either $-S_1$ or $C$ have at least one positive eigenvalue. 
 Without loss of generality, let $-S_1$ have at least one positive eigenvalue.
 
 Since  $\det(S_1)\neq 0$ and $\det(C)\neq 0$, by Lemma~\ref{lem:landtis}.b, there exists non-singular Hermitian matrices $P_1,P_2$ and positive definite Hermitian matrices $Q_1,Q_2$ such that $-S_1P_1-P_1S_1=Q_1$ and $CP_2+P_2C=Q_2$. Further, $-S_1$ and $P_1$ have the same inertia, meaning 
\[\upsilon_+(-S_1)=\upsilon_+(P_1), \ \upsilon_-(-S_1)=\upsilon_-(P_1), \ \zeta(-S_1)=\zeta(P_1)\]
where for a given matrix $A$, $\upsilon_+(A)$, $\upsilon_-(A)$, and $\zeta(A)$ are the number of eigenvalues of the argument that have positive, negative and zero real parts, respectively.
Similarly, $C$ and $P_2$ have the same inertia:
\[\upsilon_+(C)=\upsilon_+(P_2), \ \upsilon_-(C)=\upsilon_-(P_2), \ \zeta(C)=\zeta(P_2).\]
Since $-S_1$ has at least one  strictly  positive eigenvalue, $\upsilon_+(P_1)=\upsilon_+(-S_1)\geq 1$.

Define
\begin{equation}P=\bmat{I & L_0^\top\\ 0 & I}\bmat{P_1 & 0 \\ 0& P_2}\bmat{I & 0\\ L_0 & I}\label{eq:lyapp}
\end{equation}
where $L_0=(D_2^2f(x^\ast))^{-1}D_{12}^\top f(x^\ast)=CD_{12}^\top f(x^\ast)$. Since $P$ is congruent to $\mathrm{blockdiag}(P_1,P_2)$, by Sylvester's law of inertia~\cite[Thm.~4.5.8]{horn1985matrix}, $P$ and $\mathrm{blockdiag}(P_1,P_2)$ have the same inertia, meaning that $\upsilon_+(P)=\upsilon_+(\mathrm{blockdiag}(P_1,P_2))$, $\upsilon_-(P)=\upsilon_-(\mathrm{blockdiag}(P_1,P_2))$, and $\zeta(P)=\zeta(\mathrm{blockdiag}(P_1,P_2))$. 
Consider the matrix equation
$-PJ_\tau(x^\ast)-J_\tau^\top(x^\ast)P=Q_\tau$
for $-J_\tau(x^\ast)$ where
\[Q_\tau=\bmat{I & L_0^\top\\ 0 & I}\underbrace{\bmat{Q_1 & P_1D_{12}f(x^\ast)-S_1L_0^\top P_2 \\ (P_1D_{12}f(x^\ast)-S_1L_0^\top P_2)^\top & P_2L_0D_{12}f(x^\ast)+(P_2L_0D_{12}f(x^\ast))^\top+\tau Q_2}}_{B_\tau}\bmat{I & 0\\ L_0 & I}\]
which can be verified by straightforward calculations.

Observe that $Q_\tau>0$ is equivalent to $B_\tau>0$ and both matrices are symmetric so that $B_\tau>0$ if and only if $Q_1>0$ and $\schurtt_2(B_\tau)>0$ where
\begin{align*}
    \schurtt_2(B_\tau)&=P_2L_0D_{12}f(x^\ast)+(P_2L_0D_{12}f(x^\ast))^\top+\tau Q_2\\
    &\qquad-(P_1D_{12}f(x^\ast)-S_1L_0^\top P_2)^\top Q_1^{-1} (P_1D_{12}f(x^\ast)-S_1L_0^\top P_2).
\end{align*}
Now, $\schurtt_2(B_\tau)$ is also a real symmetric matrix, and hence, it is positive definite if and only if all its eigenvalues are positive. To determine the range of $\tau$ such that $\schurtt_2(B_\tau)$ is positive definite,  we can formulate an eigenvalue problem to determine the value of $\tau$ such that the matrix 
$\schurtt_2(B_\tau)$ becomes singular. This is analogous to the guard map approach used in the proof in the previous subsection for the other direction of the proof, and in this case, we are varying $\tau$ from zero to infinity and finding the point such that for all larger $\tau$, $\schurtt_2(B_\tau)$ is positive definite. Intuitively, such an argument works since $\tau$ scales the positive definite matrix $Q_2$. Towards this end, consider the eigenvalue problem in $\tau$ given by
\begin{align*}
    0=&\det\Big(\tau I-Q_2^{-1}\big((P_1D_{12}f(x^\ast)-S_1L_0^\top P_2)^\top Q_1^{-1} (P_1D_{12}f(x^\ast)-S_1L_0^\top P_2)\\
    &\qquad-P_2L_0D_{12}f(x^\ast)-(P_2L_1D_{12}f(x^\ast))^\top\big)\Big).
\end{align*}
Let $\tau_0$ be the maximum positive eigenvalue, and zero otherwise.
Then, since eigenvalues vary continuously, for all $\tau\in(\tau_0,\infty)$, $Q_\tau>0$ so that by Lemma~\ref{lem:landtis}.a we conclude that $P$ and $-J_\tau(x^\ast)$ have the same inertia, but this contradicts the stability of $-J_\tau(x^\ast)$ for all $\tau\in(\tau^\ast,\infty)$ since $\upsilon_+(P)\geq 1$.

\subsection{Proof of Corollary~\ref{cor:asymptoticiffstack}}
Suppose that $x^\ast$ is a differential Stackelberg equilibrium so that by Theorem~\ref{thm:iffstack}, there exists a $\tau^\ast\in(0,\infty)$ such that $\spec(-J_\tau(x^\ast))\subset \mb{C}_-^\circ$ for all $\tau\in (\tau^\ast,\infty)$. Now that we have a guarantee that $-J_\tau(x^\ast)$ is Hurwitz stable for any $\tau\in(\tau^\ast,\infty)$, we apply Hartman-Grobman to get that the nonlinear system $\dot{x}=-\Lambda_{\tau} g(x)$ is stable in a neighborhood of $x^\ast$. Fix any $\tau\in (\tau^\ast,\infty)$ and
let $\gamma=\arg\min_{\lambda\in \spec(J_\tau(x^\ast))}2\Re(\lambda)/|\lambda|^2$.
Then, applying Lemma~\ref{lem:convergencerate-asymptotic}, for any $\gamma_1\in(0,\gamma)$, $\tau$-{\gda} converges locally asymptotically to $x^\ast$. 

On the other hand, suppose that 
there exists a $\tau^\ast\in(0,\infty)$ such that $\spec(-J_\tau(x^\ast))\subset \mb{C}_-^\circ$ for all $\tau\in (\tau^\ast,\infty)$. Then by Theorem~\ref{thm:iffstack}, 
 if $x^\ast$ is a differential Stackelberg equilibrium. Furthermore, since $\spec(-J_\tau(x^\ast))\subset \mb{C}_-^\circ$ for all $\tau\in (\tau^\ast,\infty)$, if we let $\gamma=\arg\min_{\lambda\in \spec(J_\tau(x^\ast))}2\Re(\lambda)/|\lambda|^2$, then by Lemma~\ref{lem:convergencerate-asymptotic} $\tau$-{\gda} converges locally asymptotically to $x^\ast$ for any choice of $\gamma_1\in(0,\gamma)$.

\section{Proof of Proposition~\ref{prop:simgrad_inf}}
\label{app_sec:simgrad_inf}
The structure of this proof is as follows. We begin by introducing general background for analyzing general singularly perturbed systems. Following this, we consider the linearization of the singularly perturbed system that approximates the simultaneous gradient dynamics and describe how insights made about this system translate to the corresponding nonlinear system. Finally, we analyze the stability of the linear system around a critical point to arrive at the stated result. The analysis is primarily from~\citet{kokotovic1986singular}.

\paragraph{Analysis of General Singularly Perturbed Systems.}
Let us begin by considering a general singularly perturbed system for $x \in \mathbb{R}^n$, $z \in \mathbb{R}^m$, and a sufficiently small parameter $\varepsilon>0$ given by
\begin{equation}
\begin{split}
\dot{x} &= f(x,z,\varepsilon, t),\quad x(t_0, \varepsilon)=x_0, x \in \mathbb{R}^n\\
\vep \dot{z}  &= g(x,z, \varepsilon, t),\quad z(t_0, \varepsilon)=z_0, z \in \mathbb{R}^m
\end{split}
\label{eq:general_system}
\end{equation}
where $f$ and $g$ are assumed to be sufficiently many times continuously differential functions of the arguments $x$, $z$, $\varepsilon$, and $t$. 
Observe that when $\varepsilon=0$, the dimension of the system given in~\eqref{eq:general_system} drops from $n+m$ to $n$ since $\dot{z}$ degenerates into the equation
\begin{equation}
\begin{split}
0  &= g(\bar{x},\bar{z}, 0, t)
\end{split}
\label{eq:degen_system}
\end{equation}
where the notation of $\bar{x}, \bar{z}$ indicates that the variables belong to the system with $\varepsilon=0$. We further require the assumption that~\eqref{eq:degen_system} has $k\geq 1$ isolated roots, which for each $i\in \{1, \dots, k\}$ are given by
\begin{equation*}
\bar{z} = \bar{\phi}_i(\bar{x}, t).
\end{equation*}
We now define an $n$-dimensional manifold $M_{\varepsilon}$ for any $\varepsilon>0$ characterized by the expression
\begin{equation}
z(t, \varepsilon)=\phi(x(t, \varepsilon), \varepsilon),
\label{eq:manifold}
\end{equation}
where $\phi$ is sufficiently many times continuously differentiable function of $x$ and $\varepsilon$. For $M_{\varepsilon}$ to be an invariant manifold of the system given in~\eqref{eq:general_system}, the expression in~\eqref{eq:manifold} must hold for all $t> t^{\ast}$ if it holds for $t=t^{\ast}$. Formally, if
\begin{equation}
z(t^{\ast}, \varepsilon) = \phi(x(t^{\ast}, \varepsilon), \varepsilon) \rightarrow z(t, \varepsilon) = \phi(x(t, \varepsilon), \varepsilon) \quad \forall t\geq t^{\ast},
\label{eq:invariant_manifold}
\end{equation}
then $M_{\varepsilon}$ is an invariant manifold for~\eqref{eq:general_system}.
Differentiating the expression in~\eqref{eq:invariant_manifold} with respect to $t$, we obtain 
\begin{equation}
\dot{z} = \frac{d}{dt} \phi(x(t, \varepsilon), \varepsilon) = \frac{d\phi}{\partial x}\dot{x}.
\label{eq:invariant_manifold2}
\end{equation}
Now, multiplying the expression in~\eqref{eq:invariant_manifold2} by $\varepsilon$ and substituting in the forms of $\dot{x}$, $\dot{z}$, and $z$ from~\eqref{eq:general_system} and~\eqref{eq:manifold}, the manifold condition becomes
\begin{equation}
 g(x, \phi(x, \varepsilon), \varepsilon, t)=\varepsilon \frac{\partial \phi}{\partial x} f(x, \phi(x, \varepsilon), \varepsilon, t),
\label{eq:manifold_condition}
\end{equation}
which $\phi(x, \varepsilon)$ must satisfy for all $x$ of interest and all $\varepsilon \in [0, \varepsilon^{\ast}]$, where $\varepsilon^{\ast}$ is a positive constant.

We now define
\begin{equation*}
\eta = z-\phi(x, \varepsilon).
\end{equation*}
Then, in terms of $x$ and $\eta$, the system becomes 
\begin{align*}
\dot{x} &= f(x, \phi(x, \varepsilon)+\eta, \varepsilon, t)  \\
\vep \dot{\eta}  &= g(x, \phi(x, \varepsilon)+\eta, \varepsilon, t) - \varepsilon \frac{\partial\phi}{\partial x}f(x, \phi(x, \varepsilon)+\eta, \varepsilon, t).
\end{align*}
\begin{remark}
One interesting observation is that the above system is exactly the continuous time limiting system for the $\tau$-Stackelberg learning update in \citet{fiez:2020icml} under a simple transformation of coordinates.
\end{remark}
Observe that the invariant manifold $M_{\varepsilon}$ is characterized by the fact that $\eta =0$ implies $\dot{\eta} =0$ for all $x$ for which the manifold condition from~\eqref{eq:manifold_condition} holds. This implies that if $\eta(t_0, \varepsilon)=0$, it is sufficient to solve the system
\begin{equation*}
\dot{x} = f(x, \phi(x, \varepsilon), \varepsilon, t), x(t_0, \varepsilon) = x_0.
\end{equation*}
This system is often referred to as the exact slow model and is valid for all $x, z \in M_{\varepsilon}$ and $M_{\varepsilon}$ known as the slow manifold of~\eqref{eq:linearized_sytem}. 

\paragraph{Linearization of Simultaneous Gradient Descent Singularly Perturbed System.}
We now consider the singularly perturbed system for simulataneous gradient descent given by
\begin{equation}
\begin{split}
\dot{x} &= -D_1^2f_1(x,z) \\
\vep \dot{z}  &= -D_2f_2(x,z). 
\end{split}
\label{eq:simgrad_perturb}
\end{equation}
Let us linearize the system around a point $(x^{\ast}, z^{\ast})$. Then,\footnote{Here, the $\approx$ means, e.g.,  $D_1f_1(x,z) = D_1f_1(x^{\ast}, z^{\ast}) +D_1^2f_1(x^{\ast},z^{\ast})(x-x^{\ast}) + D_{12}f_1(x^{\ast},z^{\ast})(z-z^{\ast})+O(\|x-x^\ast\|^2+\|z-z^\ast\|^2)$, and similarly for $D_2f_2(x,z)$.}
\begin{equation}
\begin{split}
D_1f_1(x,z) &\approx D_1f_1(x^{\ast}, z^{\ast}) +D_1^2f_1(x^{\ast},z^{\ast})(x-x^{\ast}) + D_{12}f_1(x^{\ast},z^{\ast})(z-z^{\ast})\\
D_2f_2(x,z)  &\approx D_2f_2(x^{\ast}, z^{\ast}) + D_{21}f_2(x^{\ast},z^{\ast})(x-x^{\ast}) +D_2^2f_2(x^{\ast},z^{\ast})(z-z^{\ast}). 
\end{split}
\end{equation}
Defining $u = (x-x^{\ast})$ and $v = (z-z^{\ast})$ and considering a point $(x^{\ast}, z^{\ast})$ such that $D_1f_1(x^{\ast}, z^{\ast}) =0$ and $D_2f_2(x^{\ast}, z^{\ast}) =0$, then linearized singularly perturbed system is given by 
\begin{equation}
\begin{split}
\dot{u} &= -D_1^2f_1(x^{\ast},z^{\ast})u - D_{12}f_1(x^{\ast},z^{\ast})v\\
\vep \dot{v}  &= - D_{21}f_2(x^{\ast},z^{\ast})u -D_2^2f_2(x^{\ast},z^{\ast})v. 
\end{split}
\label{eq:linearized_sytem}
\end{equation}
To simplify notation, let us define $J_{\tau}$ as follows 
\begin{equation*}
J_{\tau} = 
\begin{bmatrix} D_1^2f_1(x^{\ast},z^{\ast}) & D_{12}f_1(x^{\ast},z^{\ast}) \\
  \vep^{-1}D_{21}f_2(x^{\ast},z^{\ast})  & \vep^{-1}D_2^2f_2(x^{\ast},z^{\ast}) \end{bmatrix} =  \begin{bmatrix} A_{11}  & A_{12} \\
  \vep^{-1}A_{21}  & \vep^{-1}A_{22} \end{bmatrix}
\label{eq:jacob}
\end{equation*}
along with
\begin{equation*}
  \dot{w} = \begin{bmatrix}\dot{u} \\ \dot{v}\end{bmatrix} \quad \text{and} \quad  w = \begin{bmatrix}u \\ v\end{bmatrix}.
\end{equation*}
Then, an equivalent form of~\eqref{eq:linearized_sytem} is given by
\begin{equation}
\dot{w}  = - J_{\tau} w. 
\label{eq:linearized_sytem2}
\end{equation}
In what follows, we make insights about the behavior of the nonlinear system given in~\eqref{eq:simgrad_perturb} around a critical point $(x^{\ast}, z^{\ast})$ by analyzing the linear system given in~\eqref{eq:linearized_sytem2}. Recall that if $(x^{\ast}, z^{\ast})$ is asymptotically stable with respect to the linear system in~\eqref{eq:linearized_sytem2}, then it is also asymptotically stable with respect to the nonlinear system from~\eqref{eq:simgrad_perturb}. Moreover, to determine asymptotic stability, it is sufficient to prove that $\spec(J_\tau((x^\ast, z^{\ast}))\subset \mb{C}_+^\circ$. In what follows, we specialize the general analysis of singularly perturbed systems to the singularly perturbed linear system given in~\eqref{eq:linearized_sytem2}.

\paragraph{Stability of Critical Points of Simutaneous Gradient Descent.}
The manifold condition from~\eqref{eq:manifold_condition} for the system in~\eqref{eq:linearized_sytem2} is given by
\begin{equation}
A_{21}u +A_{22}\phi(u, \varepsilon) =\varepsilon \frac{\partial \phi}{\partial u} (A_{11}u + A_{12}\phi(u, \varepsilon)).
\label{eq:manifold_condition2}
\end{equation}
We claim that~\eqref{eq:manifold_condition2} can be satisfied by a function $\phi$ that is linear in $u$. Indeed, defining 
\begin{equation*}
v = \phi(u, \varepsilon) = -L(\varepsilon) u
\end{equation*}
and then substituting back into~\eqref{eq:manifold_condition}, we get the simplified manifold condition of 
\begin{equation}
A_{21} - A_{22}L(\varepsilon) = -\varepsilon L(\varepsilon) A_{11}+\varepsilon L(\varepsilon)A_{12}L(\varepsilon).
\label{eq:manifold_condition3}
\end{equation}
Before we prove that an $L(\varepsilon)$ always exists to satisfy~\eqref{eq:manifold_condition3}, consider the change of variables 
\begin{equation*}
\eta = v + L(\varepsilon)u.
\end{equation*}
The change of variables transforms the system from~\eqref{eq:linearized_sytem2} into the equivalent representation
\begin{equation}
\begin{bmatrix} \dot{u} \\ \dot{\eta} \end{bmatrix} = \begin{bmatrix} A_{11} - A_{12}L(\varepsilon) & A_{12} \\ R(L, \varepsilon) & A_{22} + \varepsilon L(\varepsilon)A_{12}\end{bmatrix}\begin{bmatrix} u \\ \eta  \end{bmatrix}
\label{eq:transformed}
\end{equation}
where
\begin{equation}
R(L, \varepsilon) = A_{21}-A_{22}L(\varepsilon) + \varepsilon L(\varepsilon) A_{11} - \varepsilon L(\varepsilon) A_{12}L(\varepsilon).
\label{eq:r_def}
\end{equation}
Consider that $R(L, \varepsilon) = 0$.  Then, the system from~\eqref{eq:transformed} has the upper block-triangular form 
\begin{equation}
\begin{bmatrix} \dot{x} \\ \dot{\eta} \end{bmatrix} = \begin{bmatrix} A_{11} - A_{12}L(\varepsilon) & A_{12} \\ 0& A_{22} + \varepsilon L(\varepsilon)A_{12}\end{bmatrix}\begin{bmatrix} x \\ \eta  \end{bmatrix},
\label{eq:transformed_triangular}
\end{equation}
which has the effect of generating a replacement fast subsystem given by 
\begin{equation*}
\varepsilon\dot{\eta} = (A_{22}+\varepsilon L A_{12})\eta.
\end{equation*}
We now proceed to show that an $L(\varepsilon)$ such that $R(L, \varepsilon)=0$ always exists.
\begin{lemma}
\label{lemma:exist_lemma}
If $A_{22}$ is such that $\det(A_{22})\neq 0$, there is an $\varepsilon^{\ast}$ such that for all $\varepsilon \in [0, \varepsilon^{\ast}]$, there exists a solution $L(\varepsilon)$ to the matrix quadratic equation 
\begin{equation}
R(L, \varepsilon) = A_{21}-A_{22}L(\varepsilon) + \varepsilon L(\varepsilon) A_{11} - \varepsilon L(\varepsilon) A_{12}L = 0
\label{eq:r_equation}
\end{equation}
which is approximated according to 
\begin{equation}
L(\varepsilon) = A_{22}^{-1}A_{21} + \varepsilon A_{22}^{-2}A_{21}A_{0} + O(\varepsilon^2),
\label{eq:l_solution}
\end{equation}
where 
\begin{equation}
A_{0} = A_{11} - A_{12}A_{22}^{-1}A_{21}.
\label{eq:a0}
\end{equation}
\end{lemma}
\begin{proof}
To begin, observe that for $\varepsilon = 0$, the unique solution to~\eqref{eq:r_equation} is given by $L(0)=A_{22}^{-1}A_{21}$. Now, differentiating $R(L, \varepsilon)$ from~\eqref{eq:r_equation} with respect to $\varepsilon$, we find 
\begin{equation*}
A_{22} + \varepsilon L(\varepsilon)A_{12}\frac{dL}{d\varepsilon}-\varepsilon\frac{dL}{d\varepsilon}(A_{11}-A_{12}L(\varepsilon)) = L(\varepsilon)A_{11} - L(\varepsilon)A_{12}L(\varepsilon).
\end{equation*}
The unique solution of this equation at $\varepsilon$ is 
\begin{equation*}
\frac{dL}{d\varepsilon}\Big|_{\varepsilon = 0} = A_{22}^{-1}L(0)(A_{11}-A_{12}L(0)) = A_{22}^{-2}A_{21}A_{0}.
\end{equation*}
Accordingly,~\eqref{eq:l_solution} represents the first two terms of the MacLaurin series for $L(\varepsilon)$.
\end{proof}
We remark that $L(\varepsilon)$ as defined in~\eqref{eq:l_solution} is unique in the sense that even though $R(L, \epsilon)$ as given in~\eqref{eq:r_equation} may have several real solutions, only one is approximated by~\eqref{eq:l_solution}.

The characteristic equation of~\eqref{eq:transformed_triangular} is equivalent to that for the system from~\eqref{eq:linearized_sytem2} owing to the similarity transform between the systems. The block-triangular form of~\eqref{eq:linearized_sytem2} admits a characteristic equation given by 
\begin{equation}
\psi(s, \varepsilon) = \frac{1}{\varepsilon^m} \psi_{s}(s, \varepsilon)\psi_f(p, \varepsilon) = 0,
\label{eq:charac}
\end{equation}
where 
\begin{equation}
\psi_{s}(s, \varepsilon) = \det(sI - (A_{11} - A_{12}L(\varepsilon)))
\label{eq:slow_charac}
\end{equation}
is the characteristic polynomial of the slow subsystem, and 
\begin{equation}
\psi_f(p, \varepsilon) = \det(pI - (A_{22} +\varepsilon A_{12}L(\varepsilon)))
\label{eq:fast_charac}
\end{equation}
is the characteristic polynomial of the fast subsystem in the timescale $p=s\varepsilon$. Consequently, $n$ of the eigenvalues of~\eqref{eq:linearized_sytem2} denoted by $\{\lambda_1, \dots, \lambda_n\}$ are the roots of the slow characteristic equation $\psi_{s}(s, \varepsilon) = 0$ and the rest of the eigenvalues $\{\lambda_{n+1}, \dots, \lambda_{n+m}\}$ are denoted by $\lambda_{i} = \nu_{j}/\varepsilon$ for $i=n+j$ and $j\in \{1, \dots, m\}$ where $\{\nu_{1}, \dots, \nu_{m}\}$ are the roots of the fast characteristic equation $\psi_f(p, \varepsilon) = 0$.

The roots of $\psi_{s}(s, \varepsilon)$ at $\varepsilon = 0$, given by the solution to
\begin{equation}
\psi_{s}(s, 0) = \det(sI - (A_{11} - A_{12}L(0)))=0,
\label{eq:slow_charac_sol}
\end{equation}
are the eigenvalues of the matrix $A_0$ defined in~\eqref{eq:a0} since $L(0) = A_{22}^{-1}A_{21}$ as shown in Lemma~\ref{lemma:exist_lemma}. The roots of the fast characteristic equation at $\varepsilon=0$, given by the solution to 
\begin{equation}
\psi_f(p, 0) = \det(pI - A_{22}) =0
\label{eq:fast_charac_sol}
\end{equation}
are the eigenvalues of the matrix $A_{22}$. We now proceed by characterizing how closely the eigenvalues of the system at $\varepsilon=0$ approximate the eigenvalues of the system from~\eqref{eq:linearized_sytem2} as $\varepsilon\rightarrow 0$.

If $\det(A_{22}) \neq 0$, then as $\varepsilon\rightarrow 0$, $n$ eigenvalues of the system given in~\eqref{eq:linearized_sytem2} tend toward the eigenvalues of the matrix $A_0$ while the remaining $m$ eigenvalues of the system from~\eqref{eq:linearized_sytem2} tend to infinity with the rate $1/\varepsilon$ along asymptotes defined by the eigenvalues of $A_{22}$ given as $\spec(A_{22})/\varepsilon$ as a result of the continuity of coefficients of the polynomials from~\eqref{eq:slow_charac} and~\eqref{eq:fast_charac} with respect to $\varepsilon$. 

Now, consider the special (but generic) case in which the eigenvalues of $A_{0}$ are distinct and the eigenvalues of $A_{22}$ are distinct, but $A_0$ and $A_{22}$ may have common eigenvalues. Then, taking the total derivative of~\eqref{eq:charac} with respect to $\varepsilon$ we have that
\begin{equation*}
\frac{\partial \psi_s}{\partial s}\frac{d s}{d\varepsilon} + \frac{\partial \psi_s}{\partial \varepsilon} = 0
\end{equation*}
Now, observe that $\partial \psi_s/\partial s\neq 0$ since the eigenvalues of $A_0=A_{11}-A_{12}A_{22}^{-1}A_{21}$ are distinct.\footnote{Recall that having distinct eigenvalues is a generic condition for a matrix an $n_1\times n_1$ matrix, though not explicitly required for the asymptotic results; its only a condition for the big-O approximation $\lambda_i=\lambda_i(A_0)+O(\vep)$ for $i=1,\ldots, \m_1$ and $\lambda_i=\vep^{-1}(\lambda_j(A_{22})+O(\vep))$ where $i=\m_1+j$ for $j=1, \ldots, \m_2$.} For each $i=1,\ldots, n$, this gives us a well-defined derivative $ds/d\vep$ (by the implicit mapping theorem) and hence, with $s(0)=\lambda_i(A_0)$, the $O(\vep)$ approximation of $s(\vep)$ follows directly. That is, 
\[\lambda_i=\lambda_i(A_0)+O(\vep), \ i=1,\ldots, \m_1\]
Similarly, taking the total derivative of $\psi_f(p,\vep)=0$ and again applying the implicit function theorem, we have
\[\lambda_{i+\m_1}=\vep^{-1}(\lambda_j(A_{22}+O(\vep)), \ i=1, \ldots, \m_2\]
where we have used the fact that $p=s\vep$.

\section{Proof of Theorem~\ref{prop:instability}}
\label{app_sec:proofinstabilityprop}

Let $x^\ast$ be a stable critical point of $1$-{\gda} which is not a differential Stackelberg equilibrium. Without loss of generality, suppose that $\schurtt_1(-J(x^\ast))$ has at least one eigenvalue with strictly positive real part.  

Since both $\schurtt_1(-J(x^\ast))$ and $D_2^2f(x^\ast)$ have no zero valued eigenvalues, by Lemma~\ref{lem:landtis}.b, there exists non-singular Hermitian matrices $P_1,P_2$ and positive definite Hermitian matrices $Q_1,Q_2$ such that $\schurtt_1(-J(x^\ast))P_1+P_1\schurtt_1(-J(x^\ast))=Q_1$ and $D_2^2f(x^\ast)P_2+P_2D_2^2f(x^\ast)=Q_2$. Further, $\schurtt_1(-J(x^\ast))$ and $P_1$ have the same inertia, meaning 
\[\upsilon_+(\schurtt_1(-J(x^\ast)))=\upsilon_+(P_1), \ \upsilon_-(\schurtt_1(-J(x^\ast)))=\upsilon_-(P_1), \ \zeta(\schurtt_1(-J(x^\ast)))=\zeta(P_1)\]
where for a given matrix $A$, $\upsilon_+(A)$, $\upsilon_-(A)$, and $\zeta(A)$ are the number of eigenvalues of the argument that have positive, negative and zero real parts, respectively.
Similarly, $D_2^2f(x^\ast)$ and $P_2$ have the same inertia:
\[\upsilon_+(D_2^2f(x^\ast))=\upsilon_+(P_2), \ \upsilon_-(D_2^2f(x^\ast))=\upsilon_-(P_2), \ \zeta(D_2^2f(x^\ast))=\zeta(P_2).\]
 Recall that we assumed $\schurtt_1(-J(x^\ast))$ has at least one eigenvalue with strictly positive real part. Hence, $\upsilon_+(P_1)=\upsilon_+(\schurtt_1(-J(x^\ast)))\geq 1$.

Define
\[P=\bmat{I & L_0^\top\\ 0 & I}\bmat{P_1 & 0 \\ 0& P_2}\bmat{I & 0\\ L_0 & I}\]
where $L_0=(D_2^2f(x^\ast))^{-1}D_{12}^\top f(x^\ast)$. Since $P$ is congruent to $\mathrm{blockdiag}(P_1,P_2)$, by Sylvester's law of inertia~\cite[Thm.~4.5.8]{horn1985matrix}, $P$ and $\mathrm{blockdiag}(P_1,P_2)$ have the same inertia, meaning that $\upsilon_+(P)=\upsilon_+(\mathrm{blockdiag}(P_1,P_2))$, $\upsilon_-(P)=\upsilon_-(\mathrm{blockdiag}(P_1,P_2))$, and $\zeta(P)=\zeta(\mathrm{blockdiag}(P_1,P_2))$. 
Consider now the Lyapunov equation
$-PJ_\tau(x^\ast)-J_\tau^\top(x^\ast)P=Q_\tau$
for $-J_\tau(x^\ast)$ where
\[Q_\tau=\bmat{I & L_0^\top\\ 0 & I}\underbrace{\bmat{Q_1 & P_1D_{12}f(x^\ast)+\schurtt_1(-J(x^\ast))L_0^\top P_2 \\ (P_1D_{12}f(x^\ast)+\schurtt_1(-J(x^\ast))L_0^\top P_2)^\top & P_2L_0D_{12}f(x^\ast)+(P_2L_0D_{12}f(x^\ast))^\top+\tau Q_2}}_{B_\tau}\bmat{I & 0\\ L_0 & I}\]
which can be verified by straightforward calculations.

Since $\upsilon_+(P_1)\geq 1$, we have that $\upsilon_+(P)\geq 1$. Now, we find the value of $\tau_0$ such that for all $\tau>\tau_0$, $Q_\tau>0$ so that, in turn, we can apply Lemma~\ref{lem:landtis}.a, to conclude that
$\spec(-J_\tau(x^\ast))\not\subset \mb{C}_-^\circ$. Indeed, observe that $Q_\tau>0$ is equivalent to $B_\tau>0$ and both matrices are symmetric so that $B_\tau>0$ if and only if $Q_1>0$ and $\schurtt_2(B_\tau)>0$ where
\begin{align*}
    \schurtt_2(B_\tau)&=P_2L_1D_{12}f(x^\ast)+(P_2L_1D_{12}f(x^\ast))^\top+\tau Q_2\\
    &\qquad-(P_1D_{12}f(x^\ast)+\schurtt_1(-J(x^\ast))L_0^\top P_2)^\top Q_1^{-1} (P_1D_{12}f(x^\ast)+\schurtt_1(-J(x^\ast))L_0^\top P_2).
\end{align*}
Now, $\schurtt_2(B_\tau)$ is also a real symmetric matrix, and hence, it is positive definite if and only if all its eigenvalues are positive. To determine the range of $\tau$ for which $Q_\tau>0$, we simply need to solve the eigenvalue problem
\begin{align*}
    0=&\det(\tau I-Q_2^{-1}((P_1D_{12}f(x^\ast)+\schurtt_1(-J(x^\ast))L_0^\top P_2)^\top Q_1^{-1} (P_1D_{12}f(x^\ast)+\schurtt_1(-J(x^\ast))L_0^\top P_2)\\
    &\qquad-P_2L_1D_{12}f(x^\ast)-(P_2L_1D_{12}f(x^\ast))^\top)).
\end{align*}
and extract the maximum eigenvalue, namely,
\begin{align*}\tau_0=&\lambda_{\max}(Q_2^{-1}((P_1D_{12}f(x^\ast)+\schurtt_1(-J(x^\ast))L_0^\top P_2)^\top Q_1^{-1} (P_1D_{12}f(x^\ast)\\
&\qquad+\schurtt_1(-J(x^\ast))L_0^\top P_2)-P_2L_1D_{12}f(x^\ast)-(P_2L_1D_{12}f(x^\ast))^\top)).
\end{align*}
Hence, as noted previously, by Lemma~\ref{lem:landtis}.a, 
we conclude that for all $\tau\in (\tau_0, \infty)$, $\spec(-J_\tau(x^\ast))\not\subset \mb{C}_-^\circ$. 

To provide some context for the proof approach, it follows the same idea as the proof of Theorem~\ref{thm:iffstack} in Appendix~\ref{app_sec:iffstacksuff}. Indeed, to determine the range of $\tau$ such that $\schurtt_2(B_\tau)$ is positive definite,  we can formulate an eigenvalue problem to determine the value of $\tau$ such that the matrix 
$\schurtt_2(B_\tau)$ becomes singular. We vary $\tau$ from zero to infinity in order to find the point such that for all larger $\tau$, $\schurtt_2(B_\tau)$ is positive definite. Intuitively, such an argument works since $\tau$ scales the positive definite matrix $Q_2$.

\section{Proof of Theorem~\ref{thm:ganconvergence}}
\label{app_sec:proof_regularization}
As in \citet{mescheder2018training}, we only apply the regularization to the discriminator. In the following proof, we use $\nabla_x(\cdot)$ to denote the partial gradient with respect to $x$ of the argument $(\cdot)$ when the argument is the discriminator $\Dis(\cdot;\omega)$ in order prevent any confusion between the notation $D(\cdot)$ which we use elsewhere for derivatives.

To prove the first part of this result, we following similar arguments to Theorem 4.1 of \cite{mescheder2018training}. To prove the second part, we leverage the concept of the quadratic numerical range.
For both components of the proof, we will use the following form of the Jacobian of the regularized game. Indeed, first observe that 
the structural form of $J_{(\tau,\mu)}(x^\ast)$ is
\begin{equation}
    J_{(\tau,\mu)}(x^\ast)=\bmat{0 & B\\
-\tau B^\top & \tau(C+\mu R)}
\label{eq:jacform}
\end{equation}
where $B=D_{12}f(x^\ast)$, $C=-D_2^2f(x^\ast)$ and $R=D_2^2R_i(x^\ast)$. This follows from Assumption~\ref{ass:ganassump}-a., which implies that $\Dis(x;\omega^\ast)=0$ in some neighborhood of $\supp(p_{\mc{D}})$ and hence, $\nabla_x\Dis(x;\omega^\ast)=0$ and $\nabla^2_x\Dis(x;\omega^\ast)=0$ for $x\in \supp(p_{\mc{D}})$. In turn, we have that $D_1^2f(x^\ast)=0$.

\paragraph{Proof that $x^\ast=(\theta^\ast,\omega^\ast)$ is a differential Stackelberg equilibrium.} For any fixed $\mu\in[0,\infty)$, then we first observe that $x^\ast$ is also a critical point of the unregularized dynamics. Indeed, by Assumption~\ref{ass:ganassump}-a., $\Dis(x;\omega^\ast)=0$ in some neighborhood of $\supp(p_{\mc{D}})$ and hence, $\nabla_x\Dis(x;\omega^\ast)=0$ and $\nabla^2_x\Dis(x;\omega^\ast)=0$ for $x\in \supp(p_{\mc{D}})$. Further, 
$D_2R_i(\theta, \omega)=\mu \mb{E}_{p_i(x)}[D_2(\nabla_x\Dis(x;\omega))\nabla_x\Dis(x,\omega)]$ for $i=1,2$ where $p_1(x)=p_{\mc{D}}(x)$ and $p_2(x)=p_\theta(x)$.
Thus, using the above observation that $\nabla_x\Dis(x;\omega^\ast)=0  $, we have that $D_2R_i(\theta^\ast,\omega^\ast)=0$ for $i=1,2$ meaning that the derivative of the regularizer with respect to $\omega$ is zero at $x^\ast=(\theta^\ast,\omega^\ast)$ which in turn implies that $D_1f(x^\ast)=0$ and $-D_2f(x^\ast)=0$. Hence, $x^\ast$ is a critical point of the unregularized dynamics as claimed. 
Further, $C+\mu R>0$ which follows from Lemma D.5 in \cite{mescheder2018training}. 
From Lemma D.6 in  \cite{mescheder2018training}, due to Assumption~\ref{ass:ganassump}-c., if $v\neq 0$ and $v\not\in T_{\theta^\ast}\mc{M}_\Gen$, then $Bv\neq 0$ which implies that $B$ can only be rank deficient on  $T_{\theta^\ast}\mc{M}_\Gen$. 
Using this fact along with  the structure of the Jacobian as in \eqref{eq:jacform}, we have that the Schur complement of $J_{(\tau,\mu)}(x^\ast)$ is equal to $B^\top (C+\mu R)^{-1}B>0$ since $C+\mu R>0$. Hence, $x^\ast=(\theta^\ast,\omega^\ast)$ is a differential Stackelberg equilibrium.

\paragraph{Proof of stability.} 
 Examining~\eqref{eq:jacform}, it is straightforward to see that the quadratic numerical range $\mc{W}^2(J_{(\tau,\mu)})$  has eigenvalues of the form
\[\lambda_{\tau,\mu}=\tfrac{1}{2}(\tau(c+\mu r))\pm \tfrac{1}{2}\sqrt{(-\tau(c+\mu r))^2-4\tau |b|^2}\]
where  $b=\la D_{12}f(x^\ast)v,w\ra$, $c=\la -D_{2}^2f(x^\ast)w,w\ra$ and $r=\la D_2^2R_i(x^\ast)w,w\ra$ for vectors $v\in W_1 \cap (T_{\theta^\ast}\mc{M}_\Gen)^{\bot}$ and $w\in W_2\cap (T_{\omega^\ast}\mc{M}_\Dis)^{\bot}$ where $U^{\bot}$ denotes the orthogonal complement of $U$. We claim that for any value of $\mu\in[0,\mu_1]$ and any $\tau\in(0,\infty)$, $\Re(\lambda_{\tau,\mu})>0$. Indeed,
 we argue this by considering the two possible cases: (1) $ (\tau (c+\mu r))^2\leq 4|b|^2\tau $ or (2) $(\tau (c+\mu r))^2> 4\tau |b|^2 $.
\begin{itemize}
\item \textbf{Case 1}: 
Suppose that $(\tau (c+\mu r))^2\leq 4|b|^2\tau$. Then,
$\text{Re}(\lambda_{\tau,\mu})=\tfrac{1}{2}(\tau (c+\mu r))>0$ trivially since $c+\mu r>0$. 

\item \textbf{Case 2}: Suppose that $(\tau(c+\mu r))^2> 4\tau |b|^2$. In this case,
we want to ensure
that 
\[\text{Re}(\lambda_\tau)>\tfrac{1}{2}(\tau(c+\mu r))- \tfrac{1}{2}\sqrt{(-\tau(c+\mu r))^2-4\tau |b|^2}>0.\]
which is true since
 \[(\tau(c+\mu r))^2>(-\tau(c+\mu r))^2-4\tau |b|^2\ \Longleftrightarrow\ 0>-4\tau |b|^2\]
\end{itemize}
This concludes the proof.

\section{Proof of Proposition~\ref{prop:dimension}}
\label{app_sec:proof_dimension}
This proposition follows immediately from observing the structure of the Jacobian: for any matrix of the form
\[-J=\bmat{0 & -B\\ B^\top & -C}\]
at least one eigenvalue will be purely imaginary if $\m_2<\m_1/2$ where $B\in \mb{R}^{\m_1\times \m_2}$ and $C\in \mb{R}^{\m_2\times \m_2}$. Indeed, by Lyapunov's stability theorem for linear systems~\cite[Theorem 8.2]{hespanha2018linear}, a matrix $A$ is Hurwitz stable if and only if for every symmetric positive definite $Q=Q^\top>0$, there exists a unique symmetric positive definite $P=P^\top>0$, such that $A^\top P+PA=-Q$. Hence, $-J$ is Hurwitz stable if and only if there exists a $P=P^\top>0$ such that
\begin{align*}0<Q&=\bmat{0 & -B\\ B^\top & C}\bmat{P_1 & P_2\\ P_2^\top & P_3}+\bmat{P_1 & P_2\\ P_2^\top & P_3}\bmat{0 & B\\ -B^\top & C}\\
&=\bmat{-BP_2^\top-P_2B^\top & -BP_3+P_1B+ P_2C\\B^\top P_1+CP_2^\top-P_3B^\top&B^\top P_2+CP_3+ P_2^\top B+P_3C}\end{align*}
Since this is a symmetric positive definite matrix, the block diagonal components must also be symmetric positive definite so that $-BP_2-P_2B^\top>0$.\footnote{If a block matrix $Q$ with block entries $Q_{ij}$ for $i,j\in\{1,2\}$ is positive definite symmetric, then $Q_{ii}>0$ for $i=1,2$.}  Recall that $B\in \mb{R}^{\m_1\times \m_2}$ and $P_2\in \mb{R}^{\m_2\times \m_1}$.
Hence, a necessary condition for this matrix to be positive definite is that $\m_2\geq \m_1/2$ for $-BP_2-P_2B^\top$ to have full rank; of course this is not sufficient, but it is necessary. 
It is easy to see this argument is independent of whether a learning rate ratio $\tau\neq0$ or regularization is incorporated. 

\section{Extensions in the Stochastic Setting}
\label{app_sec:extensionsstochastic}
Let $\tilde{x}(t)$ be the asymptotic pseudo-trajectories of the stochastic approximation process $\{x_k\}$. That is, $\tilde{x}(t)$ are linear interpolates between the sample points $x_k$ generated by the stochastic $\tau$-{\gda} process, and are defined by
\[\tilde{x}(t)=\tilde{x}(t_k)+\frac{(t-t_k)}{\gamma_k}(\tilde{x}(t_{k+1})-\tilde{x}(t_k))\]
where $t_k=t_k+\gamma_k$ and $t_0=0$.
\begin{assumption}
The stochastic process $\{w_{k}\}$ is a martingale difference sequence with respect to the increasing family of $\sigma$-fields defined by 
\[\mc{F}_k=\sigma(x_\ell, w_\ell, \ell\leq k), \ \forall k\geq 0,\]
so that $\mb{E}[w_{k+1}|\ \mc{F}_k]=0$ almost surely  for all $k\geq 0$. Furthermore, there exists $c_1,c_2\in C(\mb{R}^d,\mb{R}_{>0})$ such that
\[\Pr\{\|w_{k+1}\|>v|\ \mc{F}_k\}\leq c_1(x_k)\exp(-c_2(x_k)v), \ n\geq 0\]
for all $v\geq \tilde{v}$ where $\tilde{v}$ is some sufficiently large, fixed number.  
\label{ass:noiserelaxed}
\end{assumption}
\begin{proposition}
Suppose that Assumption \ref{ass:noiserelaxed} holds and that $x^\ast$ is a differential Stackelberg equilibrium. Let $\gamma_{k}=1/(k+1)^{\beta}$ where $\beta\in(0,1]$. There exists a $\tau^\ast\in(0,\infty)$ and an $\epsilon_0\in(0,\infty)$ such that for any fixed $\epsilon\in(0,\epsilon_0]$, there exists  functions $h_1(\epsilon)=O(\log(1/\epsilon))$ and $h_2(\epsilon)=O(1/\epsilon)$ so that when $T\geq h_1(\epsilon)$ and $k_0\geq K_\tau$ where $K_\tau$ is such that $1/\gamma_k\geq h_2(\epsilon)$ for all $k\geq K_\tau$, the stochastic iterates of $\tau$-{\gda} with stepsize sequence $\gamma_k$ and timescale separation $\tau\in(\tau^\ast,\infty)$ satisfy
\[\Pr\{\|\tilde{x}(t)-x^\ast\|\leq \epsilon\ \forall t\geq t_{k_0}+T+1|\ \tilde{x}(t_{k_0})\in B_\epsilon(x^\ast)\}=1-O(k_0^{1-\beta/2}\exp(-C_\tau k_0^{\beta/2}))\]
for some constant $C_\tau>0$.
\label{prop:concentration}
\end{proposition}
The proof largely follows from the proofs of Theorem 1.1 and 1.2 in \cite{thoppe2019concentration}, combined with the existence of a finite timescale separation parameter obtained via Theorem~\ref{thm:iffstack}. Indeed,  since $x^\ast$ is a differential Stackelberg equilibrium, by Theorem~\ref{thm:iffstack} there exists a range of $\tau$---namely, $(\tau^\ast,\infty)$---such that for any $\tau\in (\tau^\ast,\infty)$, $x^\ast$ is a locally asymptotically stable equillibrium for $\dot{x}=-\Lambda_\tau g(x)$. Hence, fixing any $\tau\in (\tau^\ast, \infty)$, a converse Lyapunov theorem can be applied to construct a local Lyapunov function. Let $V:\mb{R}^n\to \mb{R}$ be this Lyapunov function so that there exists $r,r_0,\epsilon_0>0$ such that $r>r_0$, and 
\[B_\epsilon(x^\ast)\subseteq V^{r_0}\subset \mc{N}_{\epsilon_0}(V^{r_0})\subseteq V^r\]
for any $\epsilon\in(0,\epsilon_0]$ where, for a given $q>0$, $V^{q}=\{x\in \mathrm{dom}(V):\ V(x)\leq q\}$ and $\mc{N}_{\epsilon_0}(V^{r_0})$ is an $\epsilon_0$--neighborhood of $V^{r_0}$---i.e., $\mc{N}_{\epsilon_0}(V^{r_0})=\{x\in \mb{R}^n|\ \exists y\in V^{r_0}, \ \|x-y\|\leq \epsilon_0\}$. From here, the result follows from an application of the results in the work by \citet{thoppe2019concentration}. 

The utility of this result is that it provides a guarantee in the stochastic setting for a more reasonable and practically useful stepsize sequence. However, constructing the constants such as  $K_\tau$, $C_\tau$ and $\epsilon_0$ is highly non-trivial as can be seen in the work of \citet{thoppe2019concentration} and similar works in the area of stochastic approximation  \cite{borkar2008stochastic}. One direction of future work is examining the Lyapunov approach for directly analyzing the nonlinear singularly perturbed system; it is known, however, that the stochastic singularly perturbed systems have much weaker guarantees in terms of stability~\cite[Chap.~4]{kokotovic1986singular}. 

\section{Further Details on Related Work}
\label{app_sec:related}
In this section, we provide further details on the discussion from Section~\ref{sec:related} regarding the results presented by~\citet{jin2019local} on the local stability of gradient descent-ascent with a finite timescale separation. The purpose of this discussion is to make clear that Proposition 27 from the work of~\citet{jin2019local} does not disagree with the results we provide in Theorem~\ref{thm:iffstack} and Theorem~\ref{prop:instability} and is instead complementary. In what follows, we recall Proposition 27 of~\citet{jin2019local} in separate pieces in the terminology of this paper and delineate its meaning from our results on the stability of gradient descent-ascent with a finite timescale separation.

To begin, we consider the component of Proposition 27 from~\citet{jin2019local} which says that given any \emph{fixed} and finite timescale separation $\tau>0$, a zero-sum game can be constructed with a differential Stackelberg equilibrium that is not stable with respect to the continuous time limiting system of $\tau$-{\gda} given by the dynamics $\dot{x}=-\Lambda_{\tau} g(x)$.
\begin{proposition}[Rephrasing of {\citealt[Proposition 27(a)]{jin2019local}}]
For any fixed $\tau>0$, there exists a zero-sum game $\mc{G} = (f, -f)$ such that $\spec(J_{\tau}(x^{\ast}))\not\subset \mb{C}_{+}^{\circ}$ for a differential Stackelberg equilibrium $x^{\ast}$. 
\label{prop:cnjp1}
\end{proposition}
We now explain the proof. Let us consider any $\epsilon>0$ and the game 
\begin{equation}
f(x, y) = -x^2 + 2\sqrt{\epsilon}xy - (\epsilon/2)y^2.
\label{eq:exgame}
\end{equation}
At the unique critical point $(x^{\ast}, y^{\ast})=(0, 0)$, the Jacobian of the dynamics is given by
\begin{equation*}
J_{\tau}(x^{\ast}, y^{\ast}) = 
\begin{bmatrix}
-2 & 2\sqrt{\epsilon} \\
-2\tau\sqrt{\epsilon}  & \tau \epsilon 
\end{bmatrix}.
\end{equation*}
Moreover, observe that $(x^{\ast}, y^{\ast})$ is a differential Stackelberg equilibrium and not a differential Nash equilibrium since $D_1^2f(x^{\ast}, y^{\ast})=-2\ngtr 0$, $-D_2^2f(x^{\ast}, y^{\ast})=\epsilon>0$ and ${\tt S}_1(J(x^{\ast}, y^{\ast})) = 2>0$. Finally, the spectrum of the Jacobian is
\[\spec(J_{\tau}(x^{\ast}, y^{\ast})) = \Big\{\frac{-2 + \tau \epsilon \pm \sqrt{\tau^2\epsilon^2 - 12\tau\epsilon + 4}}{2}\Big\}.\]

Let us now fix $\tau$ as any arbitrary positive value. Then, consider the game construction from~\eqref{eq:exgame} with $\epsilon = 1/\tau$. For the fixed choice of $\tau$ and subsequent game construction, we get that
\[\spec(J_{\tau}(x^{\ast}, y^{\ast})) = \{\big(-1 \pm i\sqrt{7}\big)/2\} \not\subset \mb{C}_{+}^{\circ}.\]
This in turn means the differential Stackelberg equilibrium is not stable with respect to the dynamics $\dot{x}=-\Lambda_{\tau} g(x)$ for the given choice of $\tau$. Since the choice of $\tau$ was arbitrary, this is a valid procedure to generate a game with a differential Stackelberg equilibrium that is not stable with respect to $\dot{x}=-\Lambda_{\tau} g(x)$ given a choice of $\tau$ beforehand.

This result contrasts with that of Theorem~\ref{thm:iffstack} in the following fundamental way. In the proof of Proposition~\ref{prop:cnjp1}, $\tau$ is fixed and then the game is constructed, whereas in Theorem~\ref{thm:iffstack} the game is fixed and then the conditions on $\tau$ given. To illustrate this point, consider the game construction from~\eqref{eq:exgame} with $\epsilon$ fixed to be an arbitrary positive value. It can be verified that $\spec(J_{\tau}(x^{\ast}, y^{\ast})) \subset \mb{C}_{+}^{\circ}$ for all $\tau > 2/\epsilon$. This means that given the differential Stackelberg equilibria in this game construction, there is indeed a finite $\tau^{\ast}$ such that the equilibrium is stable with respect to $\dot{x}=-\Lambda_{\tau} g(x)$ for all $\tau \in (\tau^{\ast}, \infty)$. Put concisely, Proposition~\ref{prop:cnjp1} is showing that there is exists a continuum of games for which a differential Stackelberg equilibrium is unstable with an improper choice of finite learning rate ratio $\tau$. On the other hand, Theorem~\ref{thm:iffstack} is proving that given a game with a differential Stackelberg equilibrium, there exists a range of suitable finite learning rate ratios such that the differential Stackelberg equilibrium is guaranteed to be stable. 

We now move on to examining the portion of Proposition 27 from~\citet{jin2019local} which says that given any \emph{fixed} and finite timescale separation $\tau>0$, a zero-sum game can be constructed with a critical point that is not a differential Stackelberg equilibrium which is stable with respect to the continuous time limiting system of $\tau$-{\gda} given by $\dot{x}=-\Lambda_{\tau} g(x)$.

\begin{proposition}[Rephrasing of {\citealt[Proposition 27(b)]{jin2019local}}]
For any fixed $\tau$, there exists a zero-sum game $\mc{G} = (f, -f)$ such that $\spec(J_{\tau}(x^{\ast})) \subset \mb{C}_{+}^{\circ}$ for a critical point $x^{\ast}$ satisfying $g(x^{\ast})=0$ that is not a differential Stackelberg equilibrium. 
\label{prop:cnjp2}
\end{proposition}
In a similar manner as following Proposition~\ref{prop:cnjp1}, we now explain the proof of Proposition~\ref{prop:cnjp2} and then contrast the result with Theorem~\ref{prop:instability}. Again, consider any $\epsilon>0$, along with the game construction 
\begin{equation}
f(x, y) = x_1^2 + 2\sqrt{\epsilon}x_1y_1 + (\epsilon/2)y_1^2 - x_2^2/2 + 2\sqrt{\epsilon}x_2y_2 - \epsilon y_2^2.
\label{eq:exgame2}
\end{equation}
At the unique critical point $(x^{\ast}, y^{\ast})=(0, 0)$, the Jacobian of the dynamics is given by
\begin{equation*}
J_{\tau}(x^{\ast}, y^{\ast}) = 
\begin{bmatrix}
2 & 0 & 2\sqrt{\epsilon} & 0 \\
0 & -1 & 0 & 2\sqrt{\epsilon} \\
-2\tau\sqrt{\epsilon} & 0 & -\tau\epsilon & 0 \\
0 & -2\tau \sqrt{\epsilon} & 0 & 2\tau \epsilon
\end{bmatrix}
\end{equation*}
Observe that $(x^{\ast}, y^{\ast})$ is neither a differential Nash equilibrium nor a differential Stackelberg equilibrium since $D_1^2f(x^{\ast}, y^{\ast})=\text{diag}(2, -1)$ and $-D_2^2f(x^{\ast}, y^{\ast})=\text{diag}(\epsilon, 2\epsilon)$ are both indefinite. The spectrum of the Jacobian is
\[\spec(J_{\tau}(x^{\ast}, y^{\ast})) = \Big\{\frac{2 - \tau \epsilon \pm \sqrt{\tau^2\epsilon^2 - 12\tau\epsilon + 4}}{2}, \frac{-1+2\tau\epsilon  \pm \sqrt{4\tau^2\epsilon^2 - 12\tau\epsilon + 1}}{2}\Big\}.\]
Now, fix $\tau$ as any arbitrary positive value, then consider the game construction from~\eqref{eq:exgame2} with $\epsilon = 1/\tau$. For the fixed choice of $\tau$ and resulting game construction given the choice of $\epsilon$, we have that
\[\spec(J_{\tau}(x^{\ast}, y^{\ast})) = \{1\pm i\sqrt{7}, 1\pm i\sqrt{7}\} \subset \mb{C}_{+}^{\circ}.\]
This indicates that the non-equilibrium critical point is stable with respect to the dynamics $\dot{z}=-\Lambda_{\tau} g(z)$ where $z=(x,y)$ for the given choice of $\tau$. Similar to the proof of Proposition~\ref{prop:cnjp1}, since the choice of $\tau$ was arbitrary, the procedure to generate a game with a non-equilibrium critical point that is stable with respect to $\dot{z}=-\Lambda_{\tau} g(z)$ is valid given a choice of $\tau$ beforehand.

The key distinction between Proposition~\ref{prop:cnjp2} and Theorem~\ref{prop:instability} is analogous to that between Proposition~\ref{prop:cnjp1} and Theorem~\ref{thm:iffstack}.
Indeed, the proof and result of Proposition~\ref{prop:cnjp2} rely on $\tau$ being fixed followed by the game being constructed. On the other hand, in Theorem~\ref{prop:instability} the game is fixed and then the conditions on $\tau$ given. To make this clear, consider the game construction from~\eqref{eq:exgame2} with $\epsilon$ fixed to be an arbitrary positive value. It turns out that $\spec(J_{\tau}(x^{\ast}, y^{\ast})) \not\subset \mb{C}_{+}^{\circ}$ for all $\tau > 2/\epsilon$ since \[\text{Re}\Big(\frac{2 - \tau \epsilon \pm \sqrt{\tau^2\epsilon^2 - 12\tau\epsilon + 4}}{2}\Big)<0.\] As a result, given the unique critical point of the game there is a finite $\tau_{0}$ such that the non-equilibrium critical point is not stable with respect to $\dot{x}=-\Lambda_{\tau} g(x)$ for all $\tau \in (\tau_{0}, \infty)$. In summary, Proposition~\ref{prop:cnjp2} is showing that there is exists a continuum of games for which a non-equilibrium critical point is stable given an unsuitable choice of finite learning rate ratio $\tau$. In contrast, Theorem~\ref{prop:instability} is showing that given a game with a non-equilibrium critical point, there exists a range of finite learning rate ratios such that it is not stable.

To recap, the discussion in this section is meant to explicitly contrast Proposition 27 from the work of~\citet{jin2019local} with Theorem~\ref{thm:iffstack} and Theorem~\ref{prop:instability} since they may potentially appear contradictory to each other without close inspection. The result of~\citet{jin2019local} shows that (i) given a fixed finite learning ratio, there exists a game for with a differential Stackelberg equilibria that is not stable and (ii) given a fixed finite learning ratio, there exists a game with a non-equilibrium critical point that is stable. From a different perspective, we show that (i) given a fixed game and differential Stackelberg equilibrium, there exists a range of finite learning rate ratios for which the equilibrium is stable (Theorem~\ref{thm:iffstack}) and (ii) given a fixed game and a non-equilibrium critical point, there exists a range of finite learning rate ratios for which the critical point is not stable (Theorem~\ref{prop:instability}).

\section{Experiments Supplement}
\label{app_sec:experiments}
In this section we present several experiments not included in the body of the paper along with supplemental simulation results and details for the experiments presented in Section~\ref{sec:experiments}. We study a torus game in Section~\ref{app_sec:torus} and examine the connection between timescale separation and the region of attraction. Then, in Section~\ref{app_sec:diracgan}, we return to the Dirac-GAN game and consider the non-saturating objective function. In Section~\ref{app_sec:covariance},  we explore a generative adversarial network formulation using the Wasserstein cost function with a linear generator and quadratic discriminator for the problem of learning a covariance matrix.
We finish in Section~\ref{app_sec:genad} by presenting further results and details on our experiments training generative adversarial networks on image datasets.

\subsection{Location Game on the Torus}
\label{app_sec:torus}
We use the example in this section to further study the role of timescale separation on the regions of attraction around critical points.  
Consider the zero-sum game defined by the cost 
\begin{equation}
f(x_1,x_2)=-0.15 \cos(x_1)+\cos(x_1-x_2)+0.15 \cos(x_2).
\label{eq:torusgame}
\end{equation}
This game can be interpreted as a location game on the torus. Specifically, the first player seeks to be far from the second player but near zero, while the second player seeks to be near the first player. This is a non-convex game on a non-convex strategy space. 
The  
critical points are given by the set\footnote{Note that because the joint strategy space is a torus, $(\pm \pi,\pm\pi)=(\mp\pi,\pm\pi)$, $(\pi,0)=(-\pi,0)$, and $(0,-\pi)=(0,\pi)$.}:
\[\{x:\ g(x)=0\}=\{(0,0),(\pi,\pi),(\pi,0),(0,\pi),(-1.646,-1.496),(1.646,1.496)\}.\]
The critical points $(0,0)$ and $(\pi,\pi)$ are the only differential Stackelberg equilibrium and neither is a differential Nash equilibrium. The differential Stackelberg equilibrium at $(0, 0)$ is stable for all $\tau\in (\tau^{\ast}, \infty)$ where $\tau^{\ast}=0.74$ and the differential Stackelberg equilibrium $(\pi, \pi)$ is stable for all $\tau\in (\tau^{\ast}, \infty)$ where $\tau=1.35$. The rest of the critical points are unstable for any choice of $\tau$. We remark that we computed $\tau^{\ast}$ for each differential Stackelberg equilibrium using the construction from Theorem~\ref{thm:iffstack} in Section~\ref{sec:mainresults} and it again gave the exact value of $\tau^{\ast}$ such that the system is stable for all $\tau > \tau^{\ast}$.

In Figure~\ref{fig:torus_vfield}, we show the trajectories of $\tau$-{\gda} with $\gamma_1=0.001$ and $\tau \in \{1, 2, 5, 10\}$ given the initializations $(x_{1}^0, x_{2}^0) = (2, -1)$ and $(x_{1}^0, x_{2}^0) = (1.9, -2.1)$ overlayed on the vector field generated by the respective timescale separation parameters. We observe that as the timescale separation $\tau$ grows, the rotational dynamics in the vector field dissipate and the directions of movement become sharp. As we mentioned in previous examples, $\tau$-{\gda} moves directly to the zero line of $-D_2f(x_1, x_2)$ and then along that line to an equilibrium given sufficient timescale separation. The warping of the vector field that occurs as a result of timescale separation impacts the equilibrium that the dynamics converge to from a fixed initial condition and the neighborhood on which $\tau$-{\gda} converges to an equilibrium. In other words, the \emph{region of attraction} around critical points depends heavily on the timescale separation $\tau$.

To illustrate this fact, in Figure~\ref{fig:torus_roa} we show the regions of attraction for each choice of timescale separation.  The vector fields are again shown for each $\tau \in \{1, 2, 5, 10\}$, but now with colors overlayed indicating the equilibria that the dynamics converge to given an initialization at that position. This experiment was generated by running $\tau$-{\gda} with a dense set of initial conditions chosen uniformly over the strategy space. Positions in the strategy space without color did not converge to an equilibrium in the fixed horizon of 20000 iterations with $\gamma_1=0.04$. This happens when $\tau$-{\gda} is not initialized in the local neighborhood of attraction around a stable equilibrium. For the choice of $\tau=1$, $(0, 0)$ is the only stable equilibrium. However, as demonstrated in Figure~\ref{fig:torus_vfield}, $\tau$-{\gda} fails to converge to the equilibrium from the initial conditions  $(x_{1}^0, x_{2}^0) = (2, -1)$ and $(x_{1}^0, x_{2}^0) = (1.9, -2.1)$. This behavior is further demonstrated over the strategy space in Figure~\ref{fig:torus_roa} and highlights the local nature of the guarantees since convergence is only assured given an initialization in a suitable local neighborhood around a stable critical point. On the other hand, $\tau$-{\gda} converges to an equilibrium from any initial condition for $\tau \in \{2, 5, 10\}$ as can be seen by Figure~\ref{fig:torus_roa}. Notably, the equilibrium to which the learning dynamics converge depends on the timescale separation and initial condition. To give a concrete example, consider the initial conditions shown in Figure~\ref{fig:torus_vfield} of $(x_{1}^0, x_{2}^0) = (2, -1)$ and $(x_{1}^0, x_{2}^0) = (1.9, -2.1)$. For the initial condition $(x_{1}^0, x_{2}^0) = (2, -1)$, $\tau$-{\gda} converges to the equilibrium at $(0, 0)$ for each $\tau \in \{2, 5, 10\}$. Yet, for the initial condition $(x_{1}^0, x_{2}^0) = (1.9, -2.1)$, $\tau$-{\gda} converges to the equilibrium at $\{(0, 0), (\pi, \pi), (\pi, \pi)\}$ for the respective choices of $\tau \in \{2, 5, 10\}$. In other words, the region of attraction around the critical points changes so that from a fixed initial condition $\tau$-{\gda} may converge to distinct equilibrium depending on the initial condition. From Figure~\ref{fig:torus_roa}, we see that the region of attraction around $(x_{1}^0, x_{2}^0) = (1.9, -2.1)$ grows from $\tau=1$ to $\tau=2$ and $\tau=4$, but then shrinks at $\tau=10$. This example highlights that timescale separation has a fundamental impact on the region of attraction around critical points and as $\tau$ grows it is possible for the region of attraction around an equilibrium to shrink. Collectively, this motivates explicit methods for trying to shape the region of attraction around desirable equilibria.

\begin{figure}[t!]
    \centering
    \subfloat[][]{\includegraphics[width=\textwidth]{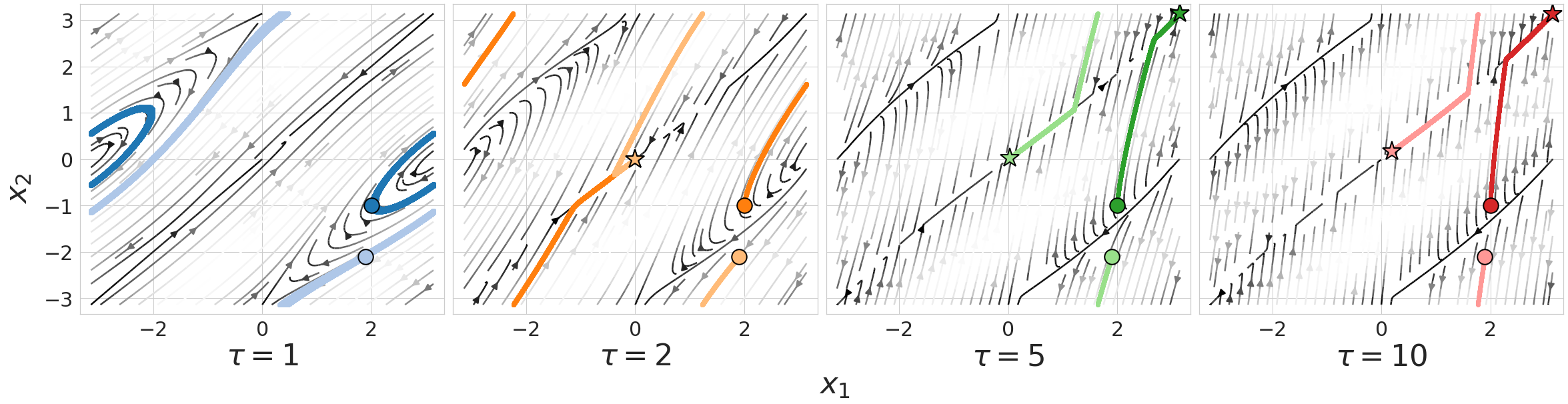}\label{fig:torus_vfield}}
    
    \subfloat{\includegraphics[width=\textwidth]{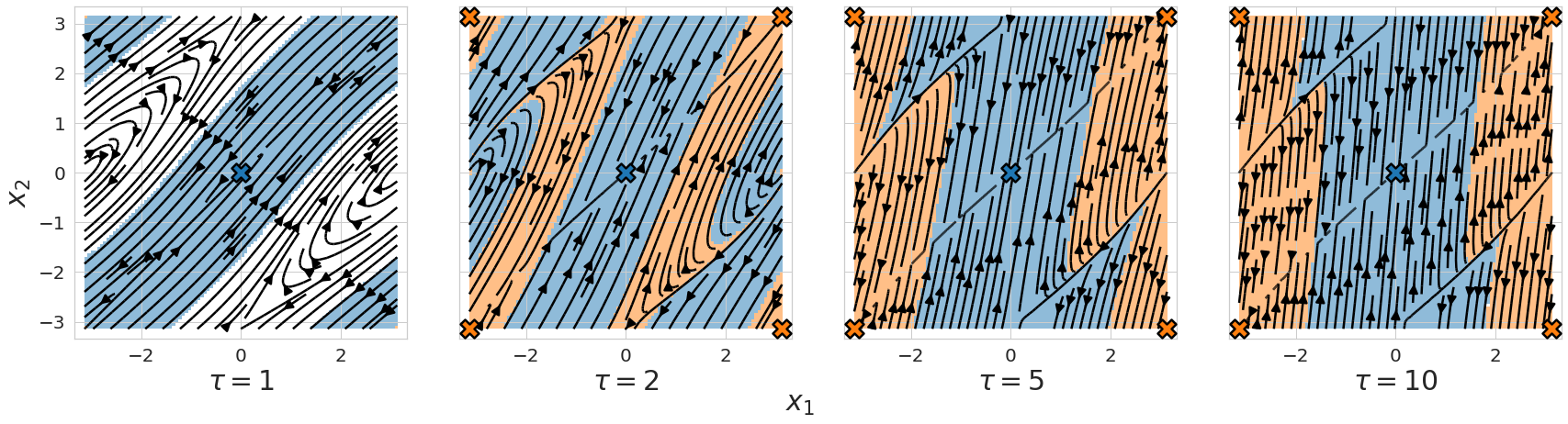}}
    
    \setcounter{subfigure}{1}
    \subfloat[][]{\includegraphics[width=.4\textwidth]{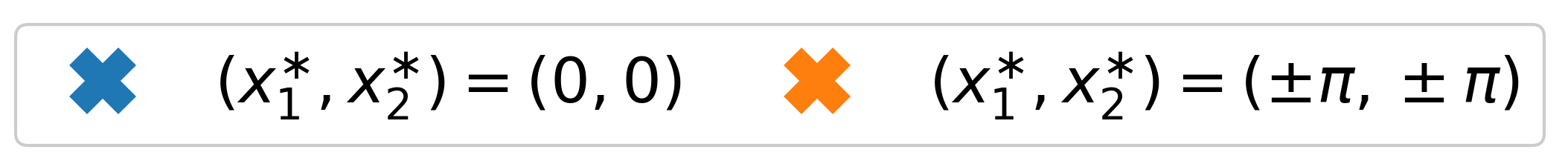}\label{fig:torus_roa}}
    
    \caption{Experimental results for the torus game defined in~\eqref{eq:torusgame} of Appendix~\ref{app_sec:torus}. In Figure~\ref{fig:torus_vfield}, we overlay multiple trajectories produced by $\tau$-{\gda} onto the vector field generated by the choice of timescale separation selection $\tau$. The shading of the vector field is dictated by its magnitude so that lighter shading corresponds to a higher magnitude and darker shading corresponds to a lower magnitude. Figure~\ref{fig:torus_roa} demonstrates the effect of timescale separation on the regions of attraction around critical points by coloring points in the strategy space according to the equilibrium $\tau$-{\gda} converges. We remark that areas without coloring indicate where $\tau$-{\gda} did not converge in the time horizon.}
    \label{fig:torus}
\end{figure}

\begin{figure*}[t!]
  \centering
  \subfloat[][]{\includegraphics[width=.4\textwidth]{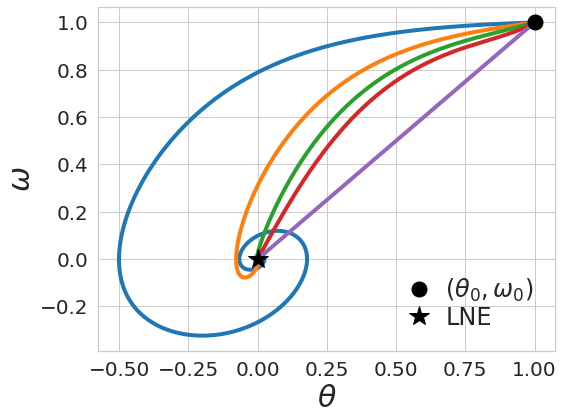}\label{fig:dirac2_a}} \hspace{7mm}
  \subfloat[][]{\includegraphics[width=.4\textwidth]{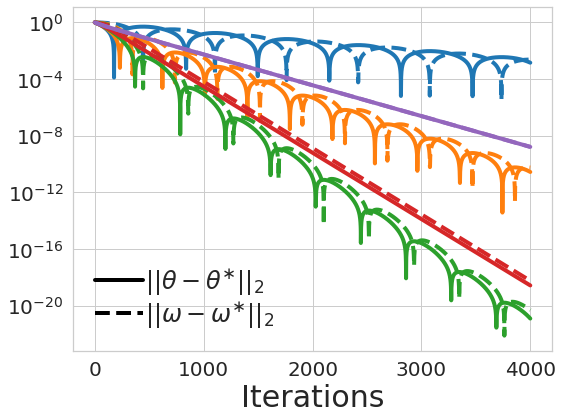}\label{fig:dirac2_b}}

  \subfloat{\includegraphics[width=.9\textwidth]{figs/DiracGan/legend}\label{fig:dirac2_leg}}
  
  \subfloat[][]{\includegraphics[width=\textwidth]{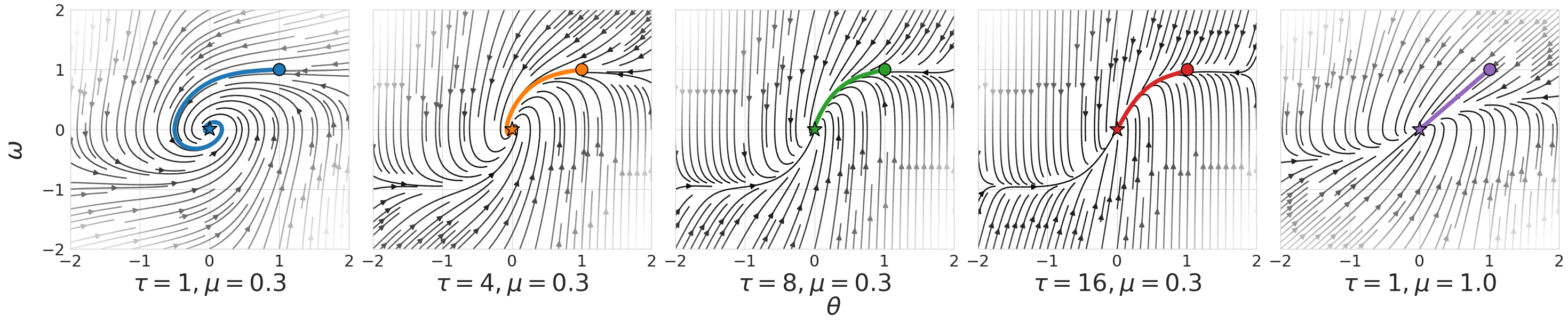}\label{fig:dirac2_c}}
  \caption{Experimental results for the Dirac-GAN game defined in~\eqref{eq:diracreg2} of Appendix~\ref{app_sec:diracgan}. 
  Figure~\ref{fig:dirac2_a} shows trajectories of $\tau$-{\gda} for $\tau \in \{1, 4, 8, 16\}$ with regularization $\mu=0.3$ and $\tau=1$ with regularization $\mu=1$. Figure~\ref{fig:dirac2_b} shows the distance from the equilibrium along the learning paths.  Figure~\ref{fig:dirac2_c} shows the trajectories of $\tau$-{\gda} overlayed on the vector field generated by the respective timescale separation and regularization parameters. The shading of the vector field is dictated by its magnitude so that lighter shading corresponds to a higher magnitude and darker shading corresponds to a lower magnitude.}
  \label{fig:dirac2}
\end{figure*}

\subsection{Dirac-GAN and Regularization: Non-Saturating Formulation}
\label{app_sec:diracgan}
In Section~\ref{sec:exp_dirac}, we presented experiments for the Dirac-GAN game studied by~\citet{mescheder2017numerics} using the original generative adversarial network formulation of~\citet{goodfellow2014generative}. In this section, we revisit the Dirac-GAN game using the non-saturating generative adversarial network formulation also proposed by~\citet{goodfellow2014generative}. While we refer the reader back to Section~\ref{sec:exp_dirac} for complete details on the Dirac-GAN, we do recall some key components of the formulation. Recall that the zero-sum game which arises from the original objective with regularization $\mu>0$ is defined by the cost
\begin{equation*}
f(\theta,\omega)=\ell(\theta\omega)+\ell(0) - \frac{\mu}{2}\omega^2.
\end{equation*}
As discussed in Section~\ref{sec:exp_dirac}, the unique critical point of the game is $(\theta^\ast,\omega^\ast)=(0,0)$ and it corresponds to the local Nash equilibrium of the unregularized game and a differential Stackelberg equilibrium of the regularized game. Moreover, the equilibrium is stable with respect to the continous time dynamics for all $\tau>0$ and $\mu>0$ so that the discrete time update $\tau$-{\gda} converges with a suitable learning rate $\gamma_1$.

The non-saturating generative adversarial network formulation proposed by~\citet{goodfellow2014generative} in the context of the Dirac-GAN game corresponds to player 1 maximizing $\ell(-\theta\omega)$ instead of minimizing $\ell(\theta\omega)$. This results in the general-sum game defined by the costs
\begin{equation}
(f_1(\theta,\omega), f_2(\theta,\omega))=(-\ell(-\theta\omega)+\ell(0) - \frac{\mu}{2}\omega^2,-\ell(\theta\omega)-\ell(0) + \frac{\mu}{2}\omega^2).
\label{eq:diracreg2}
\end{equation}
As shown by~\citet{mescheder2018training}, the unique critical point of the game remains at $(\theta^\ast,\omega^\ast)=(0,0)$. Moreover, it can be observed that $J_{\tau}(\theta^{\ast}, \omega^{\ast})$ in this formulation is identical to that from~\eqref{eq:diracjaocb} so this game is locally equivalent to the zero-sum game that arises from the original objective proposed by~\citet{goodfellow2014generative}. This is despite the fact that the non-saturating objective was motivated by global concerns (vanishing gradients early in the training process) rather than local considerations. In Figure~\ref{fig:dirac2} we present experiments with $\tau$-{\gda} for the regularized Dirac-GAN game with the non-saturating objective and $\ell(t)=-\ell(1+\exp(-t))$. We observe similar behavior as the experiments with the standard objective and refer back to Section~\ref{sec:exp_dirac} for the insights we draw from the simulation. This experiment is primarily included for completeness and to motivate our use of the non-saturating objective in the generative adversarial networks experiments we perform on image datasets in Section~\ref{sec:ganimage}.

\subsection{Generative Adversarial Network: Learning a Covariance Matrix}
\label{app_sec:covariance}
We now consider a generative adversarial network formulation presented by~\citet{daskalakis2017training} for learning a covariance matrix. This is a simple example with degeneracies much like the Dirac-GAN game, but it can be generalized to arbitrary dimensional strategy spaces and has served as a benchmark for comparing convergence rates in a number of recent papers on learning in games. Often, the example is used to show that gradient descent-ascent cycles and converges slowly. However, by and large, timescale separation is not considered. We show that gradient descent-ascent converges fast in this game with suitable timescale separation and further explore the interplay between timescale separation, regularization, and rate of convergence. We primarily follow the notation of~\citet{daskalakis2017training} when describing the problem.

The objective of this problem is to learn a covariance matrix using the Wasserstein GAN formulation. The real data $x$ is drawn from a mean-zero multivariate normal distribution with an unknown covariance matrix $\Sigma$. The generator is restricted to be a linear function of the random input noise $z\sim \mc{N}(0, I)$ and is of the form $G_V(z)=Vz$.
The discriminator is restricted to the set of all quadratic functions, which we represent by $D_{W}(x) = x^{\top}Wx$.
The parameters of the generator and the discriminator are given by $W \in \mb{R}^{d\times d}$ and $V\in \mb{R}^{d\times d}$, respectively.
For the given generator and discriminator classes the Wasserstein GAN game is defined by the cost
\begin{equation*}
f(V, W) = \mb{E}_{x\sim \mc{N}(0, \Sigma)}[x^{\top}Wx]- \mb{E}_{z\sim \mc{N}(0, I)}[z^{\top}V^{\top}WVz].
\end{equation*}
As shown by~\citet{daskalakis2017training}, the cost function can be simplified to be expressed as 
\begin{equation*}
f(V, W) = \sum_{i=1}^d\sum_{j=1}^dW_{ij}\Big(\Sigma_{ij} - \sum_{k=1}^d V_{ik}V_{jk}\Big).
\end{equation*}
With this cost, the individual gradients for gradient descent-ascent are given by \[g(V, W) = (-(W+W^{\top})V, -(\Sigma-VV^{\top})).\]
From the individual gradients, it is clear that the critical points of the game are given by $(V, W)$ such that $VV^{\top}=\Sigma$ and $W+W^{\top}=0$. Moreover, given the form of $g(V, W)$, the game Jacobian at any critical point $(V^{\ast}, W^{\ast})$ is of the form
\begin{equation*}
J_{\tau}(V^{\ast}, W^{\ast}) = \begin{bmatrix} 0 & D_{12}f(V^{\ast}, W^{\ast}) \\ -\tau D_{12}^{\top}f(V^{\ast}, W^{\ast}) & 0 \end{bmatrix}.
\end{equation*}
Consequently, the eigenvalues of the game Jacobian are purely imaginary and the critical points are not stable. To fix this problem, \citet{daskalakis2017training} regularized both the generator and discriminator. We only regularize the discriminator in this example. The cost function of the zero-sum game with regularization $\mu>0$ is given by 
\begin{equation}
f(V, W) = \sum_{i=1}^d\sum_{j=1}^dW_{ij}\Big(\Sigma_{ij} - \sum_{k=1}^d V_{ik}V_{jk}\Big) - \frac{\mu}{2}\text{Tr}(W^{\top}W).
\label{eq:cov_loss}
\end{equation}
The individual gradients for gradient descent-ascent in this regularized game are then
\[g(V, W) = (-(W+W^{\top})V, -(\Sigma-VV^{\top})+ \frac{\mu}{2} W).\]
\begin{figure*}[t!]
  \centering
  \subfloat[][$d=1, \mu=0.5$]{\includegraphics[width=.33\textwidth]{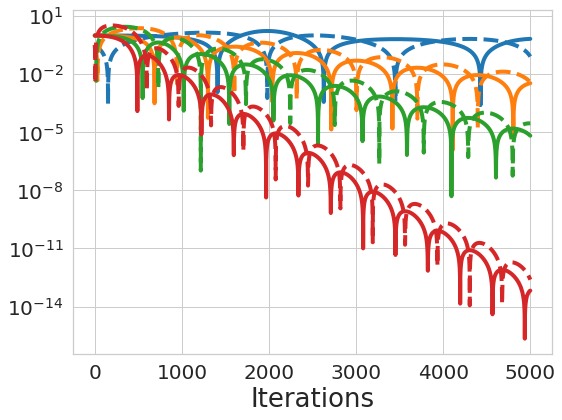}\label{fig:cov_a}}\hfill
  \subfloat[][$d=1, \mu=0.75$]{\includegraphics[width=.33\textwidth]{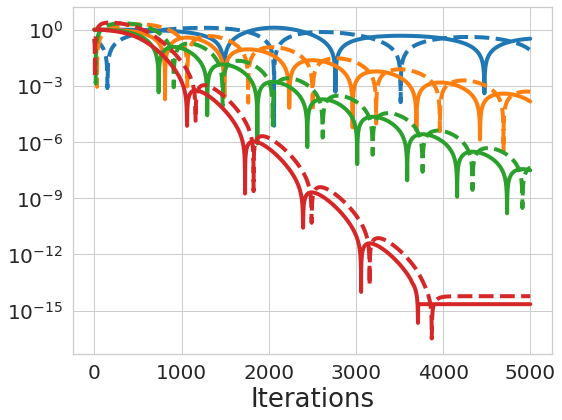}\label{fig:cov_b}}\hfill 
  \subfloat[][$d=1, \mu=1$]{\includegraphics[width=.33\textwidth]{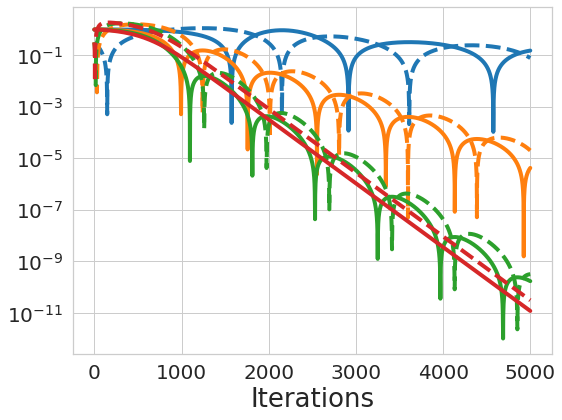}\label{fig:cov_c}}

    \subfloat[][]{\includegraphics[width=.33\textwidth]{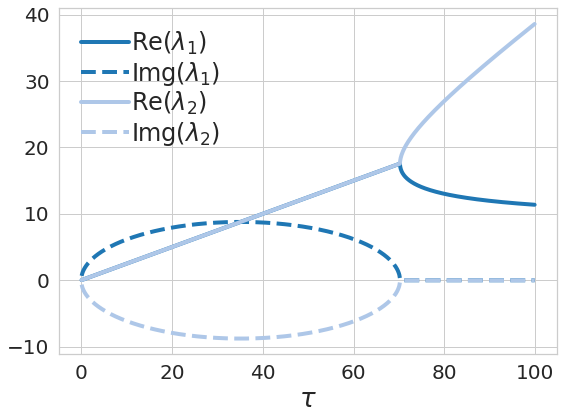}\label{fig:cov_d}}\hfill
  \subfloat[][]{\includegraphics[width=.33\textwidth]{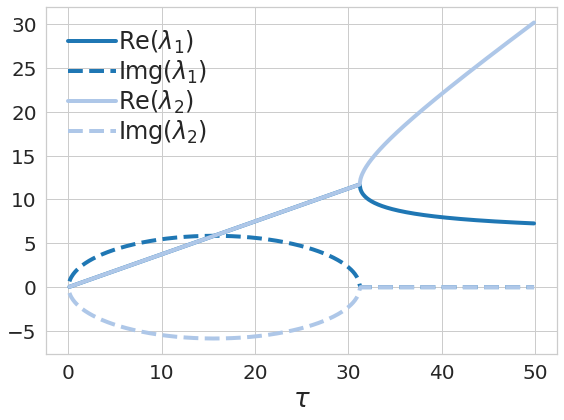}\label{fig:cov_e}}\hfill 
  \subfloat[][]{\includegraphics[width=.33\textwidth]{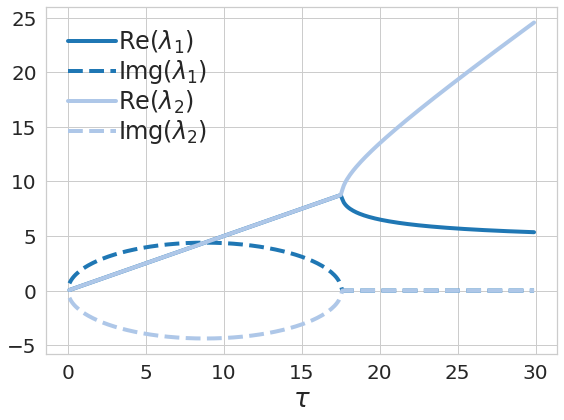}\label{fig:cov_f}} 
  
     \subfloat[][$d=5, \mu=1$]{\includegraphics[width=.33\textwidth]{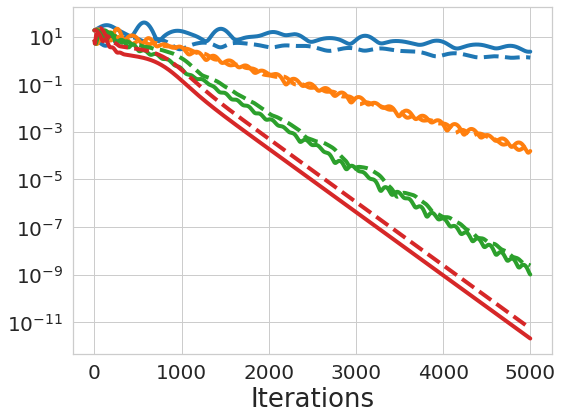}\label{fig:cov_a2}}\hfill
  \subfloat[][$d=10, \mu=1$]{\includegraphics[width=.33\textwidth]{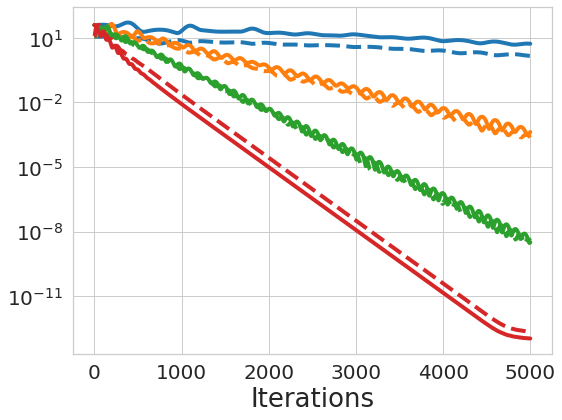}\label{fig:cov_b2}}\hfill 
  \subfloat[][$d=20, \mu=1$]{\includegraphics[width=.33\textwidth]{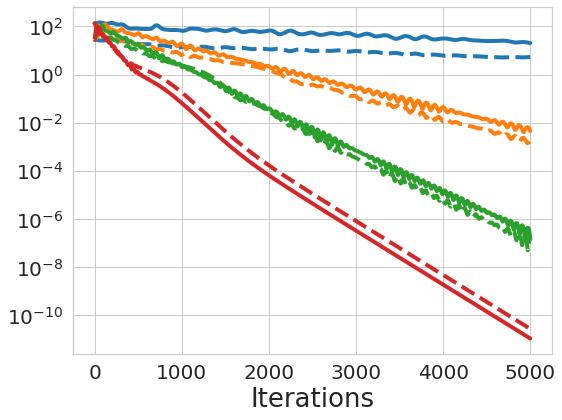}\label{fig:cov_c2}}
  
       \subfloat[][Legend for Figures~\ref{fig:cov_a},~\ref{fig:cov_b},~\ref{fig:cov_c},~\ref{fig:cov_a2},~\ref{fig:cov_b2}, and~\ref{fig:cov_c2}. ]{\includegraphics[width=.8\textwidth]{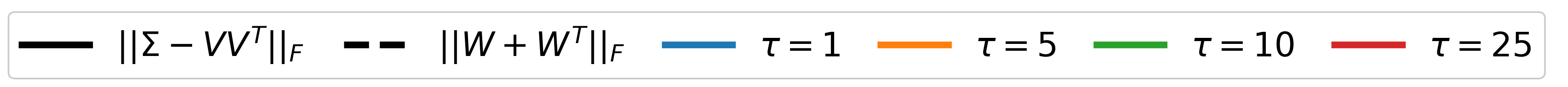}\label{fig:cov_leg}}

  \caption{Experimental results for the generative adversarial network formulation for learning a covariance matrix defined by the cost from~\eqref{eq:cov_loss} of Section~\ref{app_sec:covariance}. 
  Figures~\ref{fig:cov_a},~\ref{fig:cov_b}, and~\ref{fig:cov_c} show the distance from the equilibrium along the learning paths of $\tau$-{\gda} with $d=1$. Figures~\ref{fig:cov_d},~\ref{fig:cov_e}, and~\ref{fig:cov_f} show the trajectories of the eigenvalues of $J_{\tau}(x^{\ast})$ as a function of the $\tau$, respectively. Figures~\ref{fig:cov_a2},~\ref{fig:cov_b2}, and~\ref{fig:cov_c2} show the distance from the equilibrium along the learning paths of $\tau$-{\gda} with $d=5, 10, 20$.}
  \label{fig:cov}
\end{figure*}

We begin by considering the simplest form of this problem, which is that $d=1$. The critical points with this restriction are $(V^\ast,W^\ast)=(\sigma,0)$ and $(V^\ast,W^\ast)=(-\sigma,0)$ and the game Jacobian evaluated at them is 
\begin{equation*}
J_{\tau}(V^{\ast}, W^{\ast}) = \begin{bmatrix} 0 & -2\sigma \\ 2\tau \sigma  & \tau \mu  \end{bmatrix}.
\end{equation*}
Each critical point is a local Nash equilibrium of the unregularized game and a differential Stackelberg equilibrium of the regularized game since $-D_2^2f(V^\ast,W^\ast) =  \mu >0$ and $\schurtt_1(J(V^\ast, W^{\ast}))= 4\sigma^2/\mu >0$.
Furthermore, $\spec(J_{\tau}(V^{\ast}, W^{\ast})) = \{(\tau\mu \pm \sqrt{\tau^2\mu^2-16\tau \sigma^2})/2\}$ so that each critical point is stable for all $\tau \in (0, \infty)$ and $\mu \in (0, \infty)$ since $\spec(J_{\tau}(\theta^{\ast}, \omega^{\ast})) \subset \mb{C}_{+}^{\circ}$. Thus, given a suitably chosen learning rate $\gamma_1$, the discrete time update $\tau$-{\gda} locally converges to an equilibrium. For this reason, we focus on studying the rate of convergence for the problem as a function of timescale separation and regularization. Figures~\ref{fig:cov_a},~\ref{fig:cov_b}, and~\ref{fig:cov_c} show the distance from an equilibrium along the learning path of $\tau$-{\gda} with $\tau \in \{1, 5, 10, 25\}$ given a fixed initial condition with learning rate $\gamma_1 = 0.001$ and regularization $\mu\in \{0.5, 0.75, 1\}$, respectively. Moreover, Figures~\ref{fig:cov_d},~\ref{fig:cov_e}, and~\ref{fig:cov_f} show the trajectories of the eigenvalues for $J_{\tau}(V^{\ast}, W^{\ast})$ as a function of $\tau$ for the regularization parameters $\mu\in \{0.5, 0.75, 1\}$. Finally, Figures~\ref{fig:cov_a3},~\ref{fig:cov_b3}, and~\ref{fig:cov_c3} show the trajectories of $\tau$-{\gda} overlayed on the vector field generated by the respective timescale separation and regularization parameters. 

From the eigenvalue trajectories, we see that as $\mu$ grows, the eigenvalues become purely real at a smaller value of $\tau$. Moreover, as $\mu$ increases, the magnitude of the real and imaginary parts of the eigenvalues decreases. We observe the effect of this on the convergence, where the dynamics do not cycles as much for larger $\mu$. Again, we see the trade-off between timescale separation, regularization, and convergence. For example, despite the eigenvalues being purely real with $\mu=1$ and $\tau=25$ so that there is no rotational dynamics, the convergence is slower than for $\mu=0.75$ where there is some non-zero imaginary piece of the eigenvalues.
 
 Figures~\ref{fig:cov_a2},~\ref{fig:cov_b2}, and~\ref{fig:cov_c2} show the distance from a critical point along the learning path of $\tau$-{\gda} with $\tau \in \{1, 5, 10, 25\}$ given a fixed initial condition with learning rate $\gamma_1 = 0.001$, regularization $\mu=1$, and the dimension of the problem $d$ among the set $\{5, 10, 20\}$, respectively. The primary purpose of showing this set of results is simply to be clear that the behavior for $d=1$, which is easier to explain and visualize, transfers over to higher dimensional formulations of this problem. This is to be expected since the problem dimension is not necessarily fundamental to the convergence rate, but rather it depends on the conditioning of $\Sigma$ and each $\Sigma$ was chosen so that the behavior was comparable for each choice of dimension.

\begin{figure*}[t!]
  \centering
    \subfloat[][$d=1, \mu=0.5$]{\includegraphics[width=.8\textwidth]{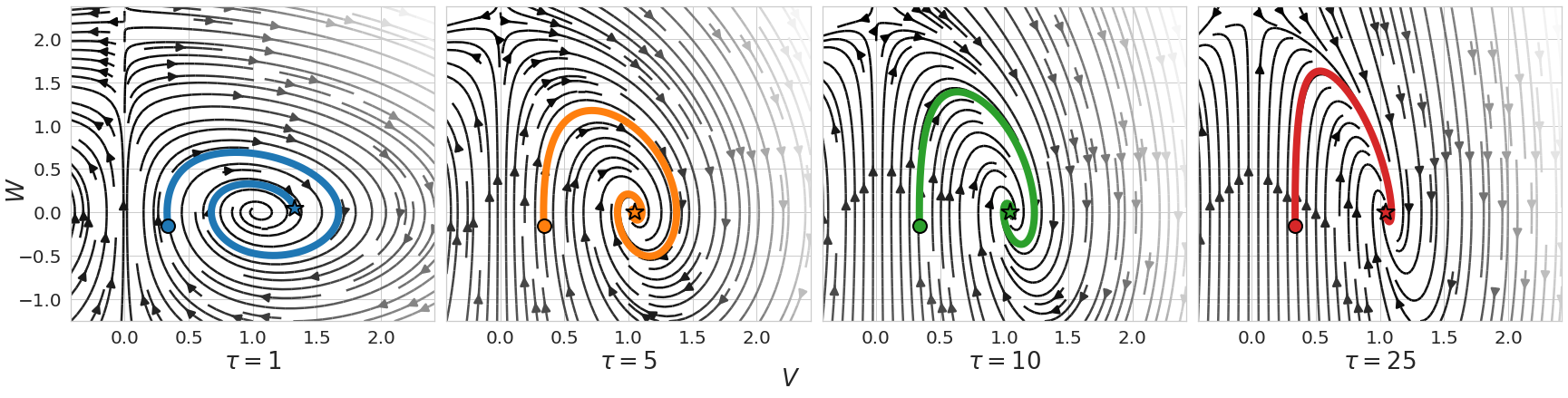}\label{fig:cov_a3}}
    
  \subfloat[][$d=1, \mu=0.75$]{\includegraphics[width=.8\textwidth]{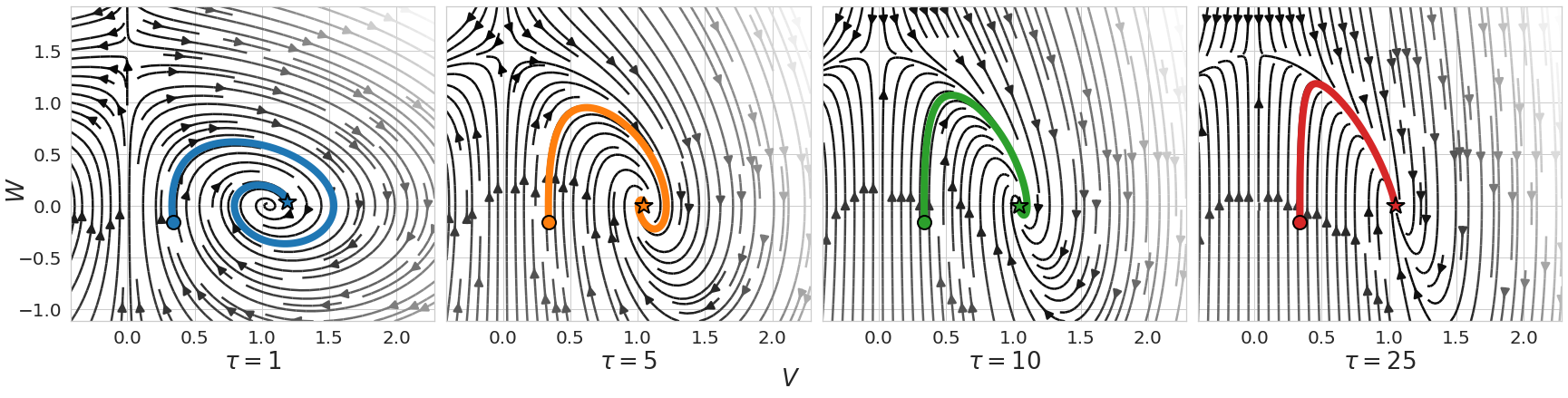}\label{fig:cov_b3}}
  
  \subfloat[][$d=1, \mu=1$]{\includegraphics[width=.8\textwidth]{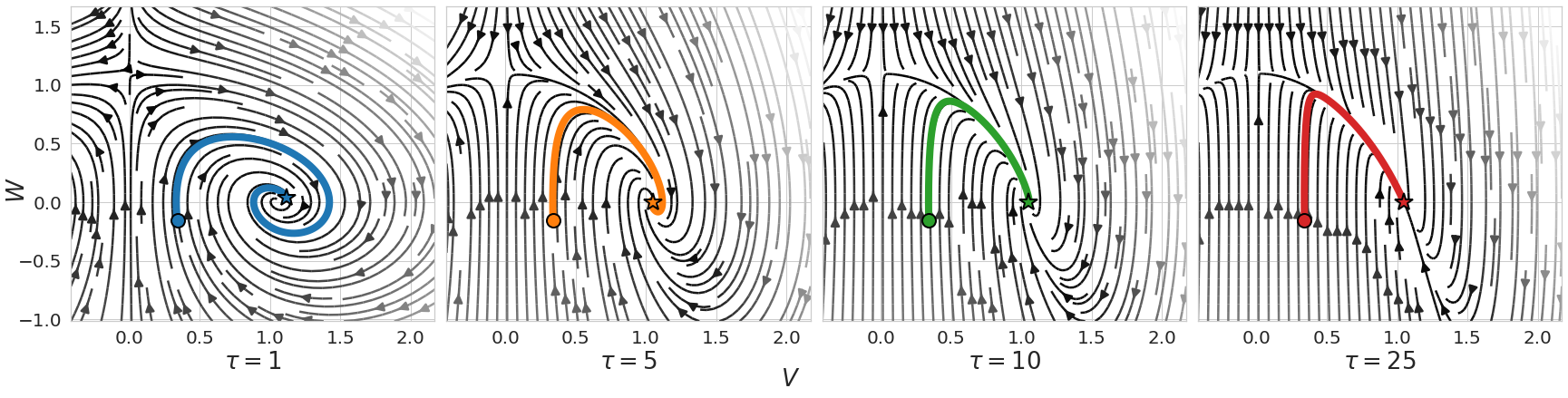}\label{fig:cov_c3}} 
   \caption{Experimental results for learning a covariance matrix defined by the cost from~\eqref{eq:cov_loss} of Section~\ref{app_sec:covariance}.  We overlay the trajectories produced by $\tau$-{\gda} onto the vector field generated by the choices of $\tau$ and $\mu$. The shading of the vector field is dictated by its magnitude so that lighter shading corresponds to a higher magnitude and darker shading corresponds to a lower magnitude.}
  \label{fig:cov2}
\end{figure*}

\subsection{Generative Adversarial Networks: Image Data}
\label{app_sec:genad}
In this section we provide additional results and details from the experiments we ran training generative adversarial networks on the CIFAR-10 and CelebA datasets. In Figure~\ref{fig:big_samples} we show more generated samples on each of the datasets. 
 We ran our simulations based on the work of~\citet{mescheder2018training} and used the publicly available code from the link~\url{https://github.com/LMescheder/GAN_stability}. We refer the readers to~\cite{mescheder2018training} for details on the implementation and architectures, as we primarily only changed the learning rates used to run the experiments. For the networks, we ran the experiments using the architecture provided in the gan\_training/models/resnet.py file of the repository.
 In Figure~\ref{tab:hyper} we include the hyperparameters we used for the experiments. To be clear, we used the same exact setup for both training CIFAR-10 and CelebA datasets.
 We computed the Frechet Inception Distance using 10k samples from the real and generated data. For both experiments and across the set of hyperparameters we did the evaluation using a fixed random noise vector to make for an equal comparison and a fixed set of real images which were randomly selected. The evaluation was done using the training data. We used the FID score implementation in pytorch available at~\url{https://github.com/mseitzer/pytorch-fid}.
\begin{figure*}[h!]
  \centering
  \subfloat[][CIFAR-10 generated sample images]{\includegraphics[width=.4\textwidth]{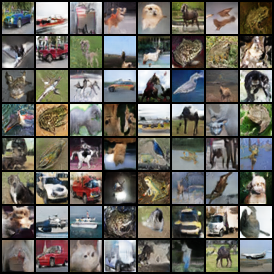}\label{fig:cifar_a}}\hfill
   \subfloat[][CelebA generated sample images]{\includegraphics[width=.4\textwidth]{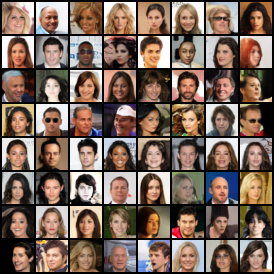}\label{fig:celeb_a}}\hfill
  \caption{Generated sample images with $\tau=4$ and $\beta =0.9999$}
  \label{fig:big_samples}
\end{figure*}

\begin{figure}[h!]
    \footnotesize
    \centering
    \begin{tabular}{|c|c|} \hline 
        Hyperparameter & Value(s)   \\\hline
         Objective & NSGAN \\ \hline
         Batch Size & 64 \\ \hline
         Latent Distribution & $z \in \mb{R}^{256}$ \\ \hline
         Generator Learning Rate & 0.0001 \\ \hline
         Timescale Separation $\tau$ & \{1, 2, 4, 8\} \\ \hline
         Learning Rate Decay & $(1+ x)^{-0.005}$ \\ \hline
         Optimizer & RMSprop \\ \hline
         RMSprop Smoothing Constant $\alpha$ & 0.99 \\ \hline
          RMSprop $\epsilon$ &  $10^{-8}$ \\ \hline
         Regularization $\mu$ & \{1, 10\} \\ \hline
         EMA Parameter $\beta$ & \{0.99, 0.999, 0.9999\} \\ \hline
    \end{tabular}
    \caption{Hyperparameters for GAN experiments on CIFAR-10 and CelebA}
    \label{tab:hyper}
\end{figure}

\end{appendices}

\iflyap
\section{Lyapunov Analysis}
\label{app_sec:lyap}
\input{app_secs/lydap}
\fi
\end{document}